\documentclass[10pt]{article}
\usepackage{bofflab}

\usepackage{microtype}
\usepackage{parskip}
\usepackage[scaled=0.92]{lato}
\usepackage[letterpaper,
  top=0.9in, bottom=0.9in,
  left=1in,  right=1in,
  footskip=0.5in,
  includefoot=false]{geometry}
\usepackage{booktabs}
\usepackage[font=small,labelfont=bf]{caption}
\usepackage{subcaption}
\usepackage{multirow}
\usepackage{ragged2e}
\usepackage{float}
\usepackage{colortbl}
\usepackage{wrapfig}
\usepackage{needspace}
\usepackage{placeins}
\usepackage[T1]{fontenc}

\makeatletter
\let\c@table\c@figure
\makeatother

\usepackage{tikz}
\usetikzlibrary{arrows.meta,calc,positioning,decorations.markings,hobby}
\usepackage{tikz-cd}

\usepackage{amsmath}
\usepackage{amssymb}
\usepackage{mathtools}
\usepackage{amsthm}
\usepackage{stmaryrd}

\usepackage{algorithm}
\usepackage{algorithmic}

\usepackage{appendix}

\usepackage{thmtools}
\usepackage{thm-restate}

\usepackage[capitalise,noabbrev]{cleveref}
\crefformat{section}{Section~#2#1#3}
\crefformat{subsection}{Section~#2#1#3}
\crefformat{subsubsection}{Section~#2#1#3}
\crefformat{equation}{(#2#1#3)}
\crefrangeformat{equation}{(#3#1#4) to (#5#2#6)}
\crefmultiformat{equation}{(#2#1#3)}{ and (#2#1#3)}{, (#2#1#3)}{ and (#2#1#3)}
\crefname{appendix}{Appendix}{Appendices}
\Crefname{appendix}{Appendix}{Appendices}
\crefformat{appendix}{Appendix~#2#1#3}
\AddToHook{cmd/appendix/before}{%
  \crefalias{section}{appendix}%
  \crefalias{subsection}{appendix}%
  \crefalias{subsubsection}{appendix}%
}

\makeatletter
\def\thm@space@setup{%
  \thm@preskip=0.8em plus 0.2em minus 0.2em
  \thm@postskip=0.8em plus 0.2em minus 0.2em
}
\makeatother

\theoremstyle{plain}
\newtheorem{theorem}{Theorem}
\newtheorem{proposition}[theorem]{Proposition}
\newtheorem{lemma}[theorem]{Lemma}

\theoremstyle{definition}

\theoremstyle{remark}

\crefname{theorem}{Theorem}{Theorems}
\Crefname{theorem}{Theorem}{Theorems}
\crefname{lemma}{Lemma}{Lemmas}
\Crefname{lemma}{Lemma}{Lemmas}
\crefname{proposition}{Proposition}{Propositions}
\Crefname{proposition}{Proposition}{Propositions}
\crefname{corollary}{Corollary}{Corollaries}
\Crefname{corollary}{Corollary}{Corollaries}
\crefname{definition}{Definition}{Definitions}
\Crefname{definition}{Definition}{Definitions}

\newcommand{\figonescale}{1}

\newcommand{\image}[4]{\begin{figure}[!tb]
    \centering
    \includegraphics[width=#1\linewidth]{images/#2}
    \caption{#3\label{#4}}
  \end{figure}}

\makeatletter
\renewcommand\fps@table{!tb}
\makeatother

\DeclareMathOperator*{\argmax}{\arg\max}

\newcommand{\E}{\mathbb{E}}

\newcommand{\Cov}{\operatorname{Cov}}
\renewcommand{\P}{\mathbb{P}}

\newcommand{\given}{\,\vert\,}
\newcommand{\givenlarge}{\;\middle\vert\;}
\newcommand{\compare}{\,\|\,}

\newcommand{\indep}{\perp\!\!\!\perp}

\newcommand{\Normal}{\mathcal{N}}

\newcommand{\KL}{\operatorname{KL}}
\newcommand{\indicator}{\mathbf{1}}

\newcommand{\convp}{\overset{p}{\longrightarrow}}

\newcommand{\R}{\mathbb{R}}

\newcommand{\C}{\mathbb{C}}

\newcommand{\tr}{\operatorname{tr}}

\newcommand{\grad}{\nabla}

\newcommand{\mc}{\mathcal}
\newcommand{\erfi}{\operatorname{erfi}}

\DeclarePairedDelimiter{\norm}{\lVert}{\rVert}
\DeclarePairedDelimiter{\abs}{\lvert}{\rvert}

\DeclarePairedDelimiter{\set}{\{}{\}}
\DeclarePairedDelimiter{\floor}{\lfloor}{\rfloor}

\title{Are we really tilting?\\{\Large The mechanics of reward guidance in flow and diffusion models}\\[0.2em]}

\author{Sanjit Dandapanthula}
\author{Nicholas M. Boffi}

\affiliation{Carnegie Mellon University}

\correspondence{Sanjit Dandapanthula at \texttt{\href{mailto:sanjitd@cmu.edu}{sanjitd@cmu.edu}}.}

\begin{document}
\begin{abstract}
    Reward guidance algorithms steer a learned generative process toward the reward-tilted measure at inference time.
While empirically powerful, these methods are prone to \emph{reward hacking}: the guided model over-optimizes the reward at the cost of fidelity to the learned distribution.
Prior work has attributed this to the complexity of neural reward functions or implicit biases in diffusion training, but its fundamental origins remain poorly understood.
We show that reward hacking arises from an approximation made in most practical implementations of reward-guided diffusion---finite-particle \textbf{plug-in estimation} of the Doob $h$-function---even in the simplest non-trivial settings of Gaussian and Gaussian mixture targets with quadratic rewards.
In closed form, we isolate two distinct failure modes of the plug-in estimator: it leads to \textbf{reward hacking within each mode} and it \textbf{cannot select high-reward modes}.
We propose a closed-form reward damping schedule that corrects the within-mode bias with no additional compute, and clarify the role of best-of-$n$ sampling in compensating for the mode selection failure.
Experiments on Gaussian mixture targets, a 2D checkerboard, and FLUX.1 text-to-image generation confirm that our theoretical insights carry over to practical settings.

\end{abstract}
\maketitle

\begin{figure}[H]
    \centering
    \includegraphics[width=\linewidth]{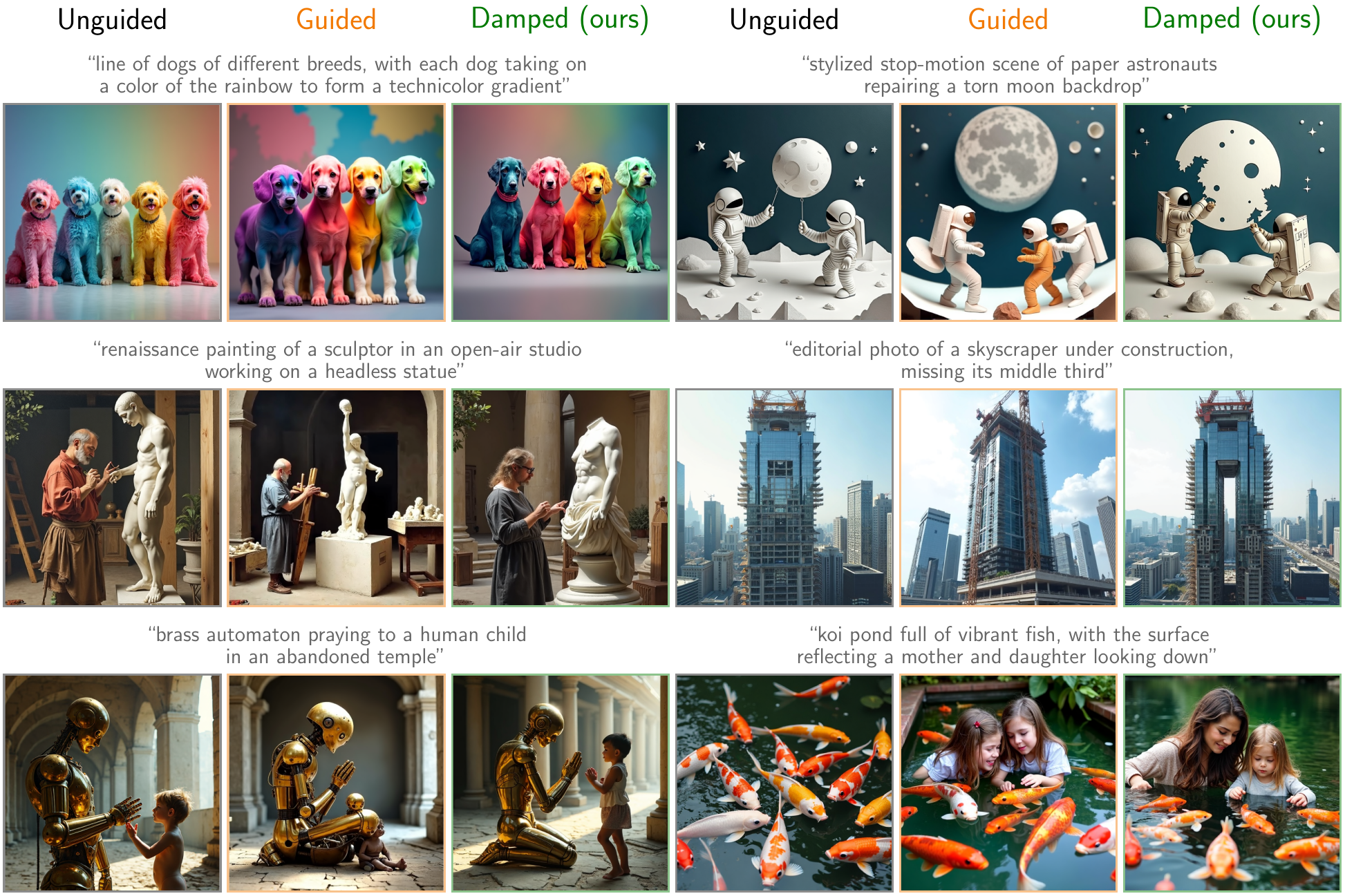}
    \caption{\textbf{Reward damping.} We introduce \emph{reward damping}, a simple and principled guidance schedule to mitigate reward hacking. Base FLUX.1 \citep{esser2024scaling} samples guided with ImageReward \citep{xu2023imagereward}; further experimental details in \Cref{sec:exp-flux-details}.}
    \label{fig:cover}
\end{figure}

\section{Introduction}

Flow and diffusion-based generative models have become the dominant paradigm for high-quality sample generation across diverse domains, powering state-of-the-art systems for text-to-image synthesis \citep{ho2020denoising,rombach2022high,saharia2022photorealistic,esser2024scaling}, molecular design \citep{hoogeboom2022equivariant,xu2022geodiff,guan20233d,schneuing2024structure}, and protein structure prediction \citep{watson2023de,yim2023se,ingraham2023illuminating,wu2024protein}. In many applications, however, sampling from the learned distribution $\rho_1(x)$ is not the end goal---we often want samples from the reward-tilted measure $\tilde{\rho}_1(x) \propto \rho_1(x)\, e^{\lambda r(x)}$ given a reward function $r$ and inverse temperature $\lambda$. For instance, practitioners may desire images that better match a text prompt \citep{xu2023imagereward, kirstain2023pick, black2024ddpo, wallace2024diffusion, huang2025rewarddance}, molecules with higher binding affinity to a target protein pocket \citep{guan2023targetdiff, guan2023decompdiff, jian2024general}, or proteins with improved stability \citep{watson2023de, gruver2023protein, ingraham2023illuminating}. This is the problem of \emph{reward-guided generation}. Widely used methods such as classifier guidance \citep{dhariwal2021diffusion} and classifier-free guidance \citep{ho2022classifier} can be understood as instances of reward guidance for particular choices of reward function, and more general value function approaches have been developed in \citet{uehara2024fine, uehara2025inference}.

Despite widespread empirical use, these methods are often prone to \emph{reward hacking}: as the guidance scale increases, generated samples over-optimize the reward at the expense of fidelity to the learned distribution, producing unrealistic or degenerate outputs \citep{gao2023scaling}. While prior experimental work has attributed reward hacking to the complexity of neural reward functions \citep{gao2023scaling, eisenstein2024helping, rafailov2024scaling} or implicit biases in generative model training \citep{zhang2024confronting}, the fundamental origins of reward hacking still remain poorly understood. Similarly, it has been empirically observed that taking the highest-reward option out of $n$ guided samples (\emph{best-of-$n$}) significantly improves the performance of guidance in generative models in many settings, but the precise mechanism underlying the improvement from best-of-$n$ is not well understood. This motivates the primary question that we study in this work:

\begin{center}
    \emph{Why does guidance fail to sample from the reward tilt, and how can we mitigate this bias?}
\end{center}

\begin{figure}
    \centering
    \scalebox{\figonescale}{\begin{tikzpicture}[
                panel/.style={draw=gray!35, rounded corners=4pt, line width=0.45pt},
                title/.style={font=\sffamily\fontsize{10}{11}\selectfont, align=center, text width=\W cm},
                note/.style={font=\sffamily\fontsize{7.95}{8.95}\selectfont, align=center},
                legendtext/.style={font=\sffamily\fontsize{7.85}{8.75}\selectfont, align=left},
                target/.style={red!88!black, line width=1.6pt}
            ]

            \def\W{3.79}
            \def\H{3.83}
            \def\G{0.25}
            \def\Y{2.52}
            \def\NoteH{1.22}
            \def\NoteTop{2.38}

            \pgfmathsetmacro{\Xzero}{0}
            \pgfmathsetmacro{\Xone}{\W+\G}
            \pgfmathsetmacro{\Xtwo}{2*(\W+\G)}
            \pgfmathsetmacro{\Xthree}{3*(\W+\G)}
            \pgfmathsetmacro{\TotalW}{4*\W+3*\G}

            \newcommand{\rewardmode}[5]{%
                \begin{scope}[rotate around={#5:(#1,#2)}]
                    \fill[gray!10, opacity=0.48] (#1,#2) ellipse ({#3} and {#4});
                    \fill[gray!18, opacity=0.42] (#1,#2) ellipse ({0.73*#3} and {0.73*#4});
                    \fill[gray!25, opacity=0.35] (#1,#2) ellipse ({0.50*#3} and {0.50*#4});
                    \fill[gray!35, opacity=0.32] (#1,#2) ellipse ({0.30*#3} and {0.30*#4});
                    \fill[gray!48, opacity=0.20] (#1,#2) ellipse ({0.17*#3} and {0.17*#4});
                    \draw[gray!36, opacity=0.62, line width=0.18pt] (#1,#2) ellipse ({#3} and {#4});
                    \draw[gray!34, opacity=0.55, line width=0.18pt] (#1,#2) ellipse ({0.73*#3} and {0.73*#4});
                    \draw[gray!32, opacity=0.48, line width=0.18pt] (#1,#2) ellipse ({0.50*#3} and {0.50*#4});
                \end{scope}
            }

            \newcommand{\landscape}[2]{%
                \begin{scope}[shift={(#1,#2)}]
                    \rewardmode{2.32}{2.62}{0.84}{0.56}{-12}
                    \rewardmode{1.00}{1.34}{0.64}{0.82}{15}
                    \rewardmode{2.52}{1.08}{0.64}{0.82}{-28}
                    \draw[target] (3.38,3.35) -- (3.55,3.52);
                    \draw[target] (3.38,3.52) -- (3.55,3.35);
                \end{scope}
            }

            \newcommand{\cloud}[9]{%
                \begin{scope}[shift={(#1,#2)}]
                    \pgfmathsetseed{#9}
                    \foreach \i in {1,...,#8}{
                            \pgfmathsetmacro{\ang}{360*rnd}
                            \pgfmathsetmacro{\rad}{sqrt(rnd)}
                            \pgfmathsetmacro{\jitx}{0.16*(rnd-0.5)}
                            \pgfmathsetmacro{\jity}{0.16*(rnd-0.5)}
                            \fill[#3, opacity=0.78]
                            ({#4 + #6*(\rad*cos(\ang)+\jitx)},
                            {#5 + #7*(\rad*sin(\ang)+\jity)}) circle (0.028);
                        }
                \end{scope}
            }

            \node[title] at ({\Xzero+0.5*\W},{\Y+\H+0.30}) {Base target};
            \node[title, text=blue!78!black] at ({\Xone+0.5*\W},{\Y+\H+0.30}) {Analytic reward tilting};
            \node[title, text=orange!95!black] at ({\Xtwo+0.5*\W},{\Y+\H+0.30}) {Practical guidance};
            \node[title, text=green!48!black] at ({\Xthree+0.5*\W},{\Y+\H+0.30}) {Damping + best-of-\textit{n}};

            \foreach \X in {\Xzero,\Xone,\Xtwo,\Xthree}{
                    \path[fill=white, rounded corners=4pt] (\X,\Y) rectangle ({\X+\W},{\Y+\H});
                    \begin{scope}
                        \clip[rounded corners=4pt] (\X,\Y) rectangle ({\X+\W},{\Y+\H});
                        \landscape{\X}{\Y}
                    \end{scope}
                    \draw[panel] (\X,\Y) rectangle ({\X+\W},{\Y+\H});
                }

            \cloud{\Xzero}{\Y}{gray!52!black}{2.32}{2.62}{0.70}{0.47}{42}{11}
            \cloud{\Xzero}{\Y}{gray!52!black}{1.00}{1.34}{0.55}{0.72}{42}{12}
            \cloud{\Xzero}{\Y}{gray!52!black}{2.52}{1.08}{0.55}{0.72}{42}{13}

            \cloud{\Xone}{\Y}{blue!76!black}{3.04}{3.10}{0.52}{0.31}{92}{21}
            \cloud{\Xone}{\Y}{blue!76!black}{2.88}{2.94}{0.78}{0.45}{28}{22}

            \cloud{\Xtwo}{\Y}{orange!95!black}{1.88}{2.26}{0.22}{0.19}{72}{31}
            \cloud{\Xtwo}{\Y}{orange!95!black}{3.08}{2.38}{0.22}{0.19}{72}{32}
            \cloud{\Xtwo}{\Y}{orange!95!black}{3.59}{3.58}{0.18}{0.16}{72}{33}

            \cloud{\Xthree}{\Y}{green!48!black}{3.01}{3.06}{0.54}{0.35}{108}{41}
            \cloud{\Xthree}{\Y}{green!48!black}{2.86}{2.90}{0.78}{0.46}{26}{42}

            \foreach \X/\C in {
                    \Xzero/gray!7,
                    \Xone/blue!5,
                    \Xtwo/orange!8,
                    \Xthree/green!7}{
                    \fill[\C, rounded corners=5pt] (\X,{\NoteTop-\NoteH}) rectangle ({\X+\W},\NoteTop);
                }

            \node[note, text width=3.36cm] at ({\Xzero+0.5*\W},{\NoteTop-0.5*\NoteH})
            {Samples from a mixture\\of several modes};

            \node[note, text width=3.40cm] at ({\Xone+0.5*\W},{\NoteTop-0.5*\NoteH})
            {Selects correct mode\\and optimizes reward};

            \node[note, text width=3.68cm] at ({\Xtwo+0.5*\W},{\NoteTop-0.5*\NoteH})
            {\textcolor{orange!95!black}{\bfseries Overconcentrates} and\\\textcolor{orange!95!black}{\bfseries overshoots} the mean};

            \node[note, text width=3.40cm] at ({\Xthree+0.5*\W},{\NoteTop-0.5*\NoteH})
            {Recovers similar behavior\\to analytic tilting};

            \draw[panel] (-0.18,0.15) rectangle ({\TotalW+0.18},0.91);
            \node[legendtext, anchor=center] at ({0.5*\TotalW},0.53)
            {\tikz[baseline=-0.55ex]{\fill[gray!12] (0,0) circle (0.17); \fill[gray!20] (0,0) circle (0.12); \fill[gray!32] (0,0) circle (0.075); \fill[gray!46] (0,0) circle (0.035);}\hspace{0.46em}base target
                \hspace{0.50cm}
                \tikz[baseline=-0.55ex]\fill[gray!55!black, opacity=0.72] (0,0) circle (0.065);\hspace{0.46em}base samples
                \hspace{0.50cm}
                \tikz[baseline=-0.55ex]\fill[blue!76!black, opacity=0.78] (0,0) circle (0.065);\hspace{0.42em}tilted samples
                \hspace{0.50cm}
                \tikz[baseline=-0.55ex]\fill[orange!95!black, opacity=0.85] (0,0) circle (0.065);\hspace{0.42em}guided samples
                \hspace{0.50cm}
                \tikz[baseline=-0.55ex]\fill[green!48!black, opacity=0.85] (0,0) circle (0.065);\hspace{0.42em}our proposal
                \hspace{0.50cm}
                \tikz[baseline=-0.55ex]\draw[target] (-0.055,-0.055) -- (0.055,0.055) (-0.055,0.055) -- (0.055,-0.055);\hspace{0.46em}reward maximizer};

        \end{tikzpicture}}

    \caption{\textbf{Overview.}
        Compared to analytic reward tilting, practical guidance algorithms over-concentrate within each mode
        and fail to select high-reward modes. We propose a damped reward scale to mitigate within-mode reward
        hacking and clarify the role of best-of-$n$ in mode selection; combining these two methods often enables
        us to approximately recover the reward tilt.
    }
    \label{fig:reward-tilting-overview}
\end{figure}
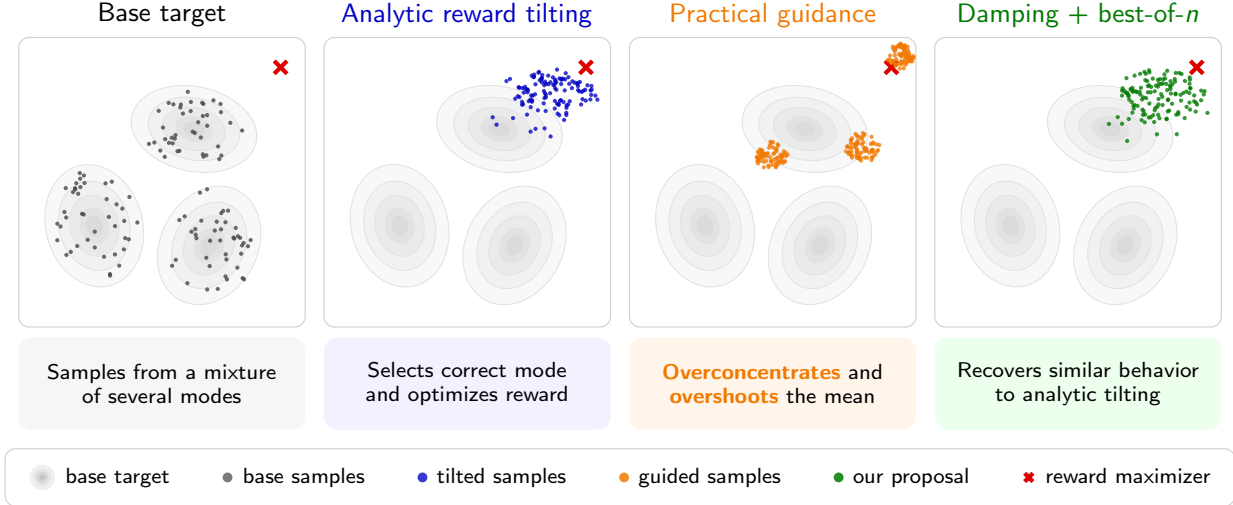

Altogether, our \textbf{main contributions} are (visually summarized in \Cref{fig:reward-tilting-overview}):
{\setlength{\leftmargini}{1.4em}
\begin{enumerate}
    \item We prove in Gaussian settings that significant within-mode reward hacking arises from the \emph{finite-particle plug-in estimation} in most implementations of reward-guided diffusion.
    \item We show in Gaussian mixture settings that guidance using plug-in estimation fails to select between modes and has no mechanism for accurately weighting distant high-reward modes.
    \item We propose a simple closed form damped reward schedule $\lambda_t$ to mitigate within-mode reward hacking and clarify the role of best-of-$n$ sampling in performing mode selection.
    \item We demonstrate the effectiveness of our damped reward schedule and showcase the role of best-of-$n$ sampling through a diverse set of experiments on Gaussian mixture targets, a 2D checkerboard, and FLUX.1 text-to-image generation.
\end{enumerate}}

\section{Background on reward guidance} \label{sec:background}

We begin by reviewing the general theoretical framework for reward guidance in generative models through stochastic interpolants and the Doob $h$-transform.

\subsection{Stochastic interpolants and flow matching}

We assume access to a pre-trained flow model $b : [0, 1] \times \R^d \to \R^d$, and samples from the data distribution $\rho_1 \in \mc{P}(\R^d)$ are drawn by numerically integrating the probability flow ODE
\begin{align*}
    \dot{x}_t = b_t(x_t)
\end{align*}
starting from noise $x_0 \sim \Normal(0, I_d)$ until $t = 1$. Such a model $b_t$ can be obtained by minimizing the flow matching objective \citep{lipman2023flow,albergo2025stochastic}. To facilitate our analysis, we consider the case $b_t := \E[\dot{I}_t \given I_t = x]$, where
\begin{equation} \label{eq:linear-interpolant}
    I_t := (1 - t) I_0 + t I_1
\end{equation}
is the \emph{linear interpolant} for $I_0 \sim \Normal(0, I_d)$ and $I_1 \sim \rho_1$ independent. Furthermore, if $\sigma_t \in \R^{d \times \ell}$ is any fixed noise schedule and $\rho_t$ denotes the density of $x_t$, the solution of the SDE
\begin{equation} \label{eq:unguided-diffusion}
    dX_t = \left( b_t(X_t) + \frac{1}{2} (\sigma_t \sigma_t^\top) \grad \log \rho_t(X_t) \right) dt + \sigma_t\, dB_t
\end{equation}
also shares the same time-marginal distribution as $x_t$ (\Cref{app:sde-matches-flow}).

\subsection{Reward guidance and plug-in estimation}

The problem of reward-guided generation is often formulated as steering a trained generative model at inference-time (without additional fine-tuning) to obtain samples from the \emph{reward-tilted measure}
\begin{equation} \label{eq:reward-tilted-measure}
    \tilde{\rho}_1(x) \propto \rho_1(x)\, e^{\lambda r(x)}
\end{equation}
for a reward function $r : \R^d \to \R$ and inverse temperature $\lambda > 0$. The reward-tilted measure \eqref{eq:reward-tilted-measure} commonly arises in reinforcement learning \citep{uehara2025inference} as the solution to a KL-regularized variational optimization problem (\Cref{app:variational-tilt}).

\paragraph{Doob $h$-transform,} The Doob $h$-transform (e.g., \cite{rogers2000diffusions}) provides a principled framework to solve the reward guidance problem by modifying the drift of the unguided SDE \eqref{eq:unguided-diffusion} to steer the terminal distribution toward the reward-tilted measure \eqref{eq:reward-tilted-measure}. Defining the \emph{Doob $h$-function} $h_t(x) \coloneq \E[e^{\lambda r(X_1)} \given X_t = x]$, the \emph{Doob $h$-transform} follows the guided ODE
\begin{equation} \label{eq:guided-ode}
    \dot{\tilde{x}}_t = b_t(\tilde{x}_t) + \frac{1}{2} (\sigma_t \sigma_t^\top) \grad \log h_t(\tilde{x}_t)
\end{equation}
from $\tilde{x}_0 = x_0$ yields a sample $\tilde{x}_1 \sim \tilde{\rho}_1$ from the reward tilt \emph{as long as $X_0 \indep X_1$}. To satisfy this constraint and ensure validity of the Doob $h$-transform, \citet{domingo2025adjoint} demonstrates that choosing the \emph{memoryless noise schedule} $\sigma_t = \sqrt{2 (1 - t) / t}\, I_d$ ensures that $X_0 \indep X_1$.

\paragraph{Plug-in estimation.} The Doob $h$-function $h_t(x) = \E[e^{\lambda r(X_1)} \given X_t = x]$ in \Cref{thm:doob-h-transform} is intractable in general, so practitioners usually approximate it with a \emph{$k$-particle plug-in estimator}
\begin{equation} \label{eq:plugin-estimator}
    \hat{h}^{(k)}_t(x) = \frac{1}{k} \sum_{i=1}^k e^{\lambda r(X_1^{(i)})},
\end{equation}
where $X_1^{(1)}, \ldots, X_1^{(k)} \sim  p_{1 \vert t}(\cdot \given x)$ are independent samples ($p_{1 \vert t}$ denotes the law of $(X_1 \given X_t = x)$). In practice, $k$ is often chosen to be small due to computational constraints \citep{holderrieth2026glass, potaptchik2026meta, holderrieth2026diamond}, introducing finite-sample bias.
\section{Related work}

In this section, we briefly review related work on reward guidance for generative models.

\paragraph{Applications of reward guidance.}
Reward guidance has become a central ingredient in aligning generative models with downstream objectives across several domains. In text-to-image generation, learned human-preference reward models such as ImageReward \citep{xu2023imagereward}, PickScore \citep{kirstain2023pick}, and HPSv2 \citep{wu2023hpsv2} are routinely used to fine-tune and steer state-of-the-art diffusion and flow-matching models. In structure-based drug design, reward guidance is used to bias diffusion-generated ligands toward target pockets and higher predicted binding affinity \citep{guan2023targetdiff, guan2023decompdiff, jian2024general}. In protein design, classifier-guided and fine-tuned diffusion models are used to generate sequences and structures optimized for stability, binding, and other functional properties \citep{watson2023de, gruver2023protein, ingraham2023illuminating, uehara2024fine}.

\paragraph{Practical approaches to reward guidance.}
The \emph{stochastic interpolant} framework \citep{albergo2025stochastic, boffi2025build} provides a flexible formulation of diffusion and flow-matching generative models, and underpins many modern reward-alignment algorithms. Within this framework, \emph{GLASS flows} \citep{holderrieth2026glass} implement the Doob $h$-transform by approximating the intractable value function $h_t(x) = \E[e^{\lambda r(X_1)} \given X_t = x]$ with a $k$-particle plug-in estimator, where each particle is obtained by simulating an inner ODE starting from fresh Gaussian noise. This inner ODE simulation is expensive, motivating the development of \emph{stochastic flow map} frameworks that sample the posterior $(X_1 \given X_t)$ in a single step \citep{potaptchik2026meta,holderrieth2026diamond}. A different deterministic approach is \emph{flow map reward guidance} \citep{huang2026guide}, which backpropagates the reward gradient through the unguided probability flow ODE from $X_t$ to $X_1$. A complementary line of work \citep{uehara2024fine, uehara2025inference,domingo2025adjoint} casts reward-guided generation as a reinforcement learning or stochastic optimal control problem and fine-tunes the generative model via policy-gradient or value-function learning; we discuss the connection to stochastic optimal control in \Cref{app:stochastic-optimal-control}.

\paragraph{Theoretical results for guidance.}
Several recent works have studied guidance in simplified settings. \citet{chidambaram2024does} analyzed classifier-free guidance for Gaussian mixtures, showing it can amplify the signal-to-noise ratio but distort mode weights. \citet{wu2024theoretical} provided theoretical insights under classifier guidance, while \citet{pavasovic2025classifier} and \citet{ventura2026emergence} study classifier-free guidance from a high-dimensional perspective. \citet{xi2024analysis} finds that monotonic schedulers perform well for classifier-free guidance in practice, and our reward damping schedule in \Cref{sec:reward-damping} is a principled choice for such a monotonic schedule. These works focus on specific guidance schemes; in contrast, we study the bias introduced by practical approximations to the general reward-guided Doob $h$-transform framework. Closest to our setting, \citet{potaptchik2026meta} prove that the $k$-particle plug-in estimator of the Doob $h$-transform converges to exact guidance in the limit as $k \to \infty$; we complement this asymptotic guarantee by characterizing the bias induced by a finite $k$ and showing that it systematically produces overly aggressive guidance. A recent line of work also proves computational hardness results for reward-guided diffusion models \citep{moitra2026steering, moitra2026tractability}, showing that exact reward tilting is NP-hard even with certain quadratic rewards.
\section{Theoretical results}
\subsection{Within-mode reward hacking} \label{sec:reward-hacking}

We now show that the plug-in approximation \eqref{eq:plugin-estimator} leads to reward hacking within a single mode; we consider a Gaussian target $\Normal(\mu, \Sigma)$ and a quadratic reward $r(x) = -\norm{x - a}_2^2$ for a target point $a \in \R^d$. We will compare the plug-in guidance to the analytic guidance framework described in \Cref{app:analytic-guidance}, which provides the benchmark for exact reward guidance. Throughout this section, we use the notation of \Cref{sec:background} and assume noise schedule $\sigma_t = \eta_t I_d$ for some $\eta_t > 0$. We let $\mu_t = t \mu$ and $\Sigma_t = (1-t)^2 I_d + t^2 \Sigma$ denote the mean and covariance of $X_t$, and let
\begin{align*}
    \mu_{1 \vert t}(x) = \mu + \Sigma^{1/2} \Sigma_t^{-1/2} \Psi_t^{1/2} (x - \mu_t),
    \qquad
    \Sigma_{1 \vert t} = \Sigma (I_d - \Psi_t),
\end{align*}
denote the mean and covariance of the Gaussian distribution $(X_1 \given X_t = x)$, where $\Psi_t \coloneq \exp\!\left( -\int_t^1 \eta_v^2\, \Sigma_v^{-1}\, dv \right)$ (see \Cref{app:ivf-bias-gaussian-proof} for the derivation).

\subsubsection{Plug-in estimation bias} \label{sec:plugin-bias}

In this section, we analyze the bias due to the $k$-particle plug-in estimator \eqref{eq:plugin-estimator} in the limit as the step size in the numerical ODE solver goes to zero; we call this the \emph{plug-in flow}. Defining $u_t^{(k)}(x) = \E[\grad \log \hat{h}^{(k)}_t(X_t) \given X_t = x]$ (where the expectation is taken over the $k$ particle draws), we obtain the following result.

\begin{proposition}[Convergence to plug-in flow, informal]
    Under appropriate regularity conditions, the Euler trajectory $\bar{x}_t$ obtained by running the guided ODE with the plug-in estimator converges uniformly in probability as the step size in the solver goes to zero:
    \begin{equation*}
        \sup_{t \in [0, 1]}\; \norm{\bar{x}_t - \tilde{x}_t}_2 \convp 0.
    \end{equation*}
    where the limiting trajectory starts at $\tilde{x}_0 = X_0$ and follows the plug-in flow defined by
    \begin{lavenderbox}
        \begin{equation} \label{eq:plugin-flow}
            \dot{\tilde{x}}_t = b_t(\tilde{x}_t) + \frac{1}{2} (\sigma_t \sigma_t^\top)\, u_t^{(k)}(\tilde{x}_t).
        \end{equation}
    \end{lavenderbox}
\end{proposition}

Intuitively, this result means that the Monte Carlo error of the particles in the plug-in estimator vanishes as the step size goes to zero, so the plug-in estimator behaves like its expectation $u_t^{(k)}(x)$ in the limit; note that each $k$ produces a \emph{different} biased plug-in flow as the step size vanishes ($k \to \infty$ recovers the analytic guidance). We defer the formal statement to \Cref{prop:plugin-flow} in \Cref{app:plugin-flow-convergence}, whose proof relies on Gr\"onwall's inequality along with Doob's $L^2$ maximal inequality, and study the bias induced by the plug-in flow directly. We first characterize the plug-in flow with $k = 1$ particle for a Gaussian target under the memoryless schedule; the corresponding result for an arbitrary noise schedule is deferred to \Cref{thm:plugin-flow-gaussian-nonmemoryless} in \Cref{app:plugin-flow-gaussian-nonmemoryless}.

\begin{restatable}[Plug-in flow for Gaussian target]{theorem}{thmPluginGaussian} \label{thm:plugin-flow-gaussian}
    Under the memoryless schedule, the guidance term for the $k = 1$ plug-in flow is given by
    \begin{align*}
        u_t^{(1)}(x) = -2 \lambda t\, \Sigma \Sigma_t^{-1} (\mu_{1 \vert t}(x) - a),
    \end{align*}
    and the terminal distribution of the plug-in flow is $\tilde{x}_1 \sim \Normal(\mu^{(1)}, \Sigma^{(1)})$, where
    \begin{align*}
        \mu^{(1)}     \coloneq \mu - T_{\mathrm{pull}}^{(1)} (\mu - a), & \qquad
        T_{\mathrm{pull}}^{(1)} \coloneq \sqrt{\pi}\, (\lambda \Sigma)^{1/2} \exp(-\lambda \Sigma)\, \erfi\!\left((\lambda \Sigma)^{1/2}\right), \\
        \Sigma^{(1)}                                                    & \coloneq \Sigma \exp(-2 \lambda \Sigma).
    \end{align*}
\end{restatable}

The proof is in \Cref{app:plugin-flow-gaussian-memoryless-proof} and follows by the reparameterization trick and solving ODEs for the mean and covariance of the plug-in flow in closed form. Since $T_{\mathrm{pull}}^{(1)}$ and $\Sigma^{(1)}$ are matrix functions of $\Sigma$, all operators simultaneously diagonalize. Along an eigendirection $v$ with $\Sigma$ having eigenvalue $\sigma > 0$, let $x \coloneq \lambda \sigma$. The eigenvalues of $T_{\mathrm{pull}}^{(1)}$ and $T_\mathrm{pull}$ along $v$ are
\begin{align*}
    \lambda_v(T_\mathrm{pull}^{(1)}) = \sqrt{\pi x}\, e^{-x}\, \erfi\!\left(\sqrt{x}\right), \qquad
    \lambda_v(T_\mathrm{pull}) = \frac{2x}{1 + 2x}.
\end{align*}
Note that the eigenvalues of $T_\mathrm{pull}^{(1)}$ overshoot 1 around $x \approx 0.854$ before decaying back down to 1, while the eigenvalue of the true $T_{\mathrm{pull}}$ increases to 1 at a rational rate. Therefore, the mean of the plug-in target will overshoot in the direction of the reward maximizer $a$ compared to the analytic tilt. Similarly, the eigenvalues of $\Sigma^{(1)}$ and $\tilde{\Sigma}$ along $v$ are
\begin{align*}
    \lambda_v(\Sigma^{(1)}) = \sigma e^{-2x}, \qquad \lambda_v(\tilde{\Sigma}) = \frac{\sigma}{1 + 2x}.
\end{align*}
This implies that the plug-in estimator leads to exponentially fast contraction in the covariance of the guided samples. This mean overshoot and covariance contraction demonstrates that \emph{significant} reward hacking arises from the $k = 1$ plug-in estimator, even in a simple Gaussian setting (panel (B) of \Cref{fig:gmm-quadratic-main}). Although this result holds for $k = 1$ particle, we prove next that exponentially many particles are required to resolve reward hacking within a single mode. Recall that $W_\infty(\mu, \nu) \geq W_p(\mu, \nu)$ for all $p \geq 1$.

\begin{restatable}[$\infty$-Wasserstein bound]{theorem}{thmWasserstein} \label{thm:wasserstein-bound}
    Assuming that $\rho_1 = \Normal(\mu, \Sigma)$ and $r(x) = -\norm{x - a}_2^2$, let $\tilde{\rho}_1^{(k)}$ denote the terminal distribution of the plug-in flow~\eqref{eq:plugin-flow} with $k$ particles. Then, the density $\tilde{\rho}_1^{(k)}$ is close to the density $\tilde{\rho}_1^{(1)}$ in the $\infty$-Wasserstein distance:
    \begin{align*}
        W_\infty(\tilde{\rho}_1^{(k)}, \tilde{\rho}_1^{(1)})
        \lesssim \sqrt{\log k}.
    \end{align*}
\end{restatable}

The constant in this corollary does not depend on dimension. The proof, given in \Cref{app:wasserstein-proof}, follows from the reparameterization trick, Gaussian concentration inequality, and Gr\"onwall's inequality applied to the coupled ODE trajectories. Thus, increasing $k$ can only logarithmically mitigate the reward hacking from the $k = 1$ plug-in flow (panel (C) of \Cref{fig:gmm-quadratic-main}). In \Cref{app:fmrg-comparison}, we compare flow map reward guidance (FMRG) \citep{huang2026guide} to the plug-in flow and show that their behavior is qualitatively similar.

\subsubsection{Reward scale damping for bias correction} \label{sec:reward-damping}

In this section, we propose a simple practical method to correct the plug-in bias.

\begin{proposition}[Reward damping schedule]
    Assuming an isotropic Gaussian target $\rho_1 = \Normal(\mu, \sigma^2 I_d)$, we can recover the correct guidance by replacing the reward scale $\lambda$ with the \emph{reward damping} schedule
    \begin{lavenderbox}
        \begin{equation} \label{eq:reward-damping}
            \lambda_t \coloneq \frac{\lambda}{1 + 2 \lambda \sigma_{1 \vert t}^2}, \qquad \sigma_{1 \vert t}^2 = \frac{\sigma^2 (1-t)^2}{(1-t)^2 + t^2 \sigma^2}.
        \end{equation}
    \end{lavenderbox}
\end{proposition}

In practice, although the target may be non-isotropic with unknown covariance, we propose treating $\sigma > 0$ as a tunable hyperparameter ($\sigma = 0$ recovers constant $\lambda$ and increasing $\sigma$ corresponds to damping the guidance). This method is computationally free compared to a $k$-fold increase in the cost from using $k > 1$ particles. We visualize the effect of damping in \Cref{fig:gmm-quadratic-main} for a symmetric isotropic 2-component Gaussian mixture with quadratic reward $r(x) = -\norm{x - a}_2^2$; experimental details and additional ablations are provided in \Cref{app:gmm-additional}. We verify our theory in \Cref{sec:exp-damping-gaussian}, and in \Cref{sec:experiments-within-mode}, we further demonstrate that these considerations extend to high-dimensional text-to-image settings.

\image{0.95}{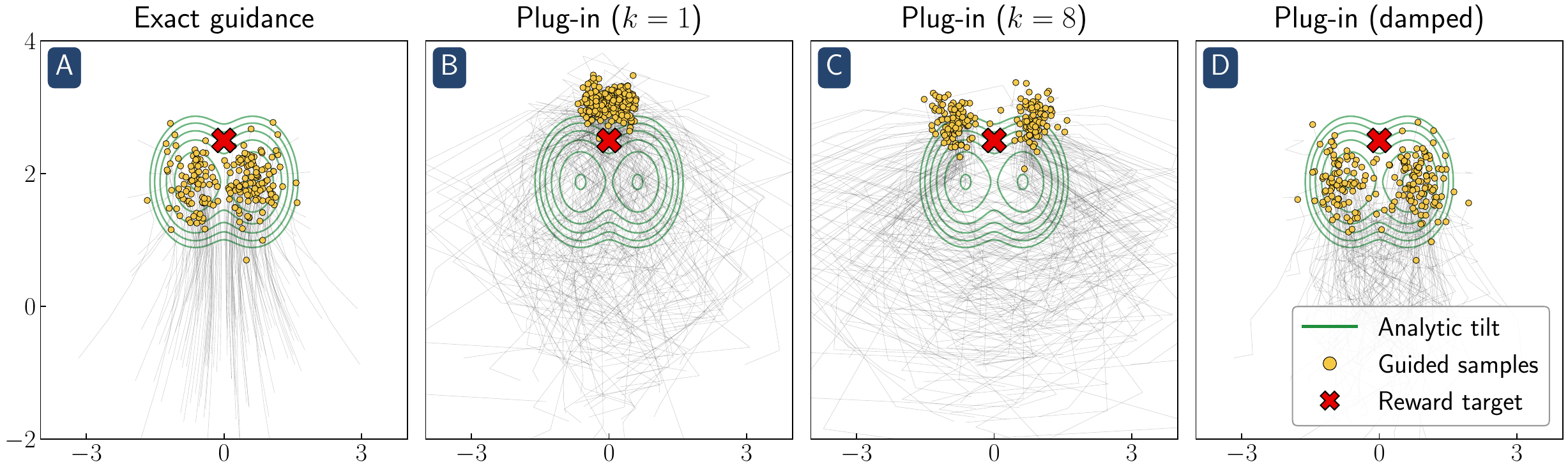}{\textbf{Reward hacking for Gaussian mixtures.} Exact guidance (A) faithfully samples the tilted distribution $\tilde{\rho}_1$. The $k = 1$ plug-in estimator (B) overshoots the mean and shrinks covariance; reward damping (D) corrects the collapse better than $k = 8$ (C) with $8 \times$ less computational cost.}{fig:gmm-quadratic-main}

Two natural alternatives to the reward damping schedule fail to sample from the reward-tilted measure, even for isotropic Gaussians. We show in \Cref{sec:ivf-bias} that non-memoryless schedules do under-correct the mean and inflate the variance, but no practical choice of non-memoryless schedule can match the analytic tilt. It also follows from \Cref{thm:plugin-flow-gaussian} that using a smaller reward scale $\lambda^\prime < \lambda$ similarly fails to recover the analytic tilt.

\subsection{Mode selection: the role of the initial seed} \label{sec:best-of-n}

Throughout this section, we consider a Gaussian mixture target $I_1 \sim \sum_{i=1}^m \pi_i\, \Normal(\cdot \given \mu_i, \Sigma_i)$ under the memoryless noise schedule $\sigma_t = \sqrt{2 (1 - t) / t}\, I_d$. Let $\rho_{it} = \Normal(\cdot \given t \mu_i,\, (1 - t)^2 I_d + t^2 \Sigma_i)$ denote the unguided marginal density of the $i$th component, $\rho_t = \sum_{i=1}^m \pi_i\, \rho_{it}$ the unguided density, and $h_{it}$ the Doob $h$-function for the $i$th component (given by \Cref{prop:analytic-guidance-gaussian} with $\mu_i$, $\Sigma_i$ in place of $\mu$, $\Sigma$).

\subsubsection{Plug-in guidance cannot select modes} \label{sec:plugin-mode-selection}

We now extend the analysis of \Cref{sec:reward-hacking} from Gaussian targets to Gaussian mixtures. The exact Doob $h$-transform performs long-range mode selection, while the $k = 1$ plug-in flow does not. In \Cref{app:ivf-bias-gmm-proof}, we show that the overall analytic guidance term for a Gaussian mixture is given by a weighted sum of the component-wise guidance terms $\sum_{i=1}^m w_{it}\, \grad \log h_{it}$, where $w_{it} \propto \pi_i\, \rho_{it}(x)\, h_{it}(x)$. Evaluating the Gaussian integral for $h_{it}$, one finds
\begin{align*}
    h_{it}(x) \propto \exp\!\left( -\lambda\, (\mu_{i, 1 \vert t}(x) - a)^\top (I_d + 2 \lambda \Sigma_{i, 1 \vert t})^{-1} (\mu_{i, 1 \vert t}(x) - a) \right),
\end{align*}
so a mode $i$ far from the current location can still exert exponentially large guidance if its reward is high. This long-range mode attraction is a special feature of the exact $h$-transform; we now show that the plug-in estimator does not share this property.

\begin{restatable}[Plug-in flow for Gaussian mixture target]{theorem}{thmPluginGMM} \label{thm:plugin-flow-gmm}
    Fixing $t \in [0, 1]$, let $\bar{X}_1(x)$ denote a sample from $p_{1 \vert t}(\cdot \given x)$. Then, the guidance term for the plug-in flow with $k = 1$ particle is given by
    \begin{equation} \label{eq:plugin-gmm-guidance}
        u_t^{(1)}(x) = -2 \lambda\, \E\!\left[\grad \bar{X}_1(x)^\top (\bar{X}_1(x) - a)\right].
    \end{equation}
\end{restatable}

The proof (\Cref{app:plugin-flow-gmm-proof}) follows from applying the chain rule. In \Cref{app:plugin-flow-gmm-proof}, we additionally provide a variational equation using GLASS flows (\Cref{thm:glass-flow}) which can be solved to obtain the Jacobian term $\grad \bar{X}_1(x)$. In \eqref{eq:plugin-gmm-guidance}, the reward enters only through the term $(\bar{X}_1(x) - a)$. Since the sample from $p_{1 \vert t}(\cdot \given x)$ typically remains in the same mode as $x$, a far-away mode contributes negligibly to $\E[\grad \bar{X}_1(x)^\top (\bar{X}_1(x) - a)]$ regardless of how large its reward is. This is in contrast to the exact $h$-transform, which attracts trajectories toward high-reward modes exponentially through the weights $h_{it}(x)$. In this sense, the $k = 1$ plug-in flow cannot select high-reward modes from a distance.

\subsubsection{Best-of-\texorpdfstring{$n$}{n} can select modes} \label{sec:best-of-n-modes}

We now show through a simple representative example that plug-in guidance with any finite number of particles fundamentally cannot select between well-separated modes with poor reward gradients, while best-of-$n$ sampling compensates for this failure. In the following theorem, we let $\rho_1(x) = \frac{1}{2} \Normal(x \given \mu, \sigma^2) + \frac{1}{2} \Normal(x \given -\mu, \sigma^2)$ for $\mu > 0$, use the reward $r(x) = -R\, \indicator_{x < 0}$ for $R > 0$, and assume the memoryless noise schedule (\Cref{thm:memoryless-schedule}). Define the \emph{correct mode probability} as the probability that the final sample is nonnegative.

\begin{restatable}[Best-of-$n$ for mode selection]{theorem}{thmModeSelection} \label{thm:mode-selection}
    In the above setting, we have:
    \begin{enumerate}
        \item[(i)] \emph{(Analytic guidance)} Under the tilted measure, the correct mode probability is $\tilde{p} = (1 + e^{-\lambda R})^{-1} \to 1$ exponentially as $R \to \infty$.
        \item[(ii)] \emph{(Plug-in guidance)} Under the $k$-particle plug-in flow with any finite $k \geq 1$, the correct mode probability is $\tilde{p}_1^{(k)} = 1/2$.
        \item[(iii)] \emph{(Best-of-$n$)} Running $n$ independent finite-$k$ plug-in trajectories and selecting the highest-reward output yields $\tilde{p}_n^{(k)} = 1 - 2^{-n}$.
    \end{enumerate}
\end{restatable}

\begin{figure}
    \centering
    \includegraphics[width=0.9\linewidth]{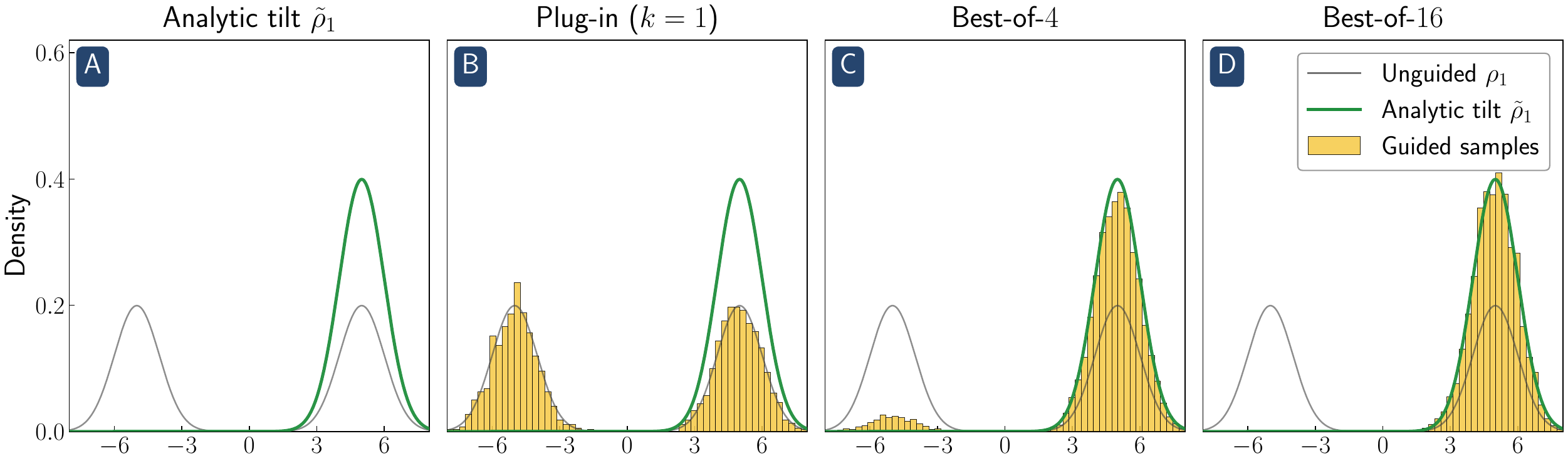}
    \caption{\textbf{Mode selection for Gaussian mixture.} Best-of-$n$ increases the correct mode probability compared to $k = 1$ guidance, better matching the analytic tilt.\label{fig:ms-step-overview}}
\end{figure}

The proof is given in \Cref{app:mode-selection-proof}. Part (ii) holds because the indicator reward has zero gradient Lebesgue-a.e., so the plug-in guidance vanishes and the guided flow coincides with the unguided flow, which is symmetric. \Cref{fig:ms-step-overview} illustrates this empirically on the setting above; we defer a similar demonstration with a Gaussian reward (deviating from the theory) as well as further experiments to \Cref{app:mode-selection-gmm}. In \Cref{sec:experiments-mode-selection}, we demonstrate that this phenomenon arises in a 2D checkerboard setting as well as in a realistic text-to-image setting with a VLM reward.

\section{Experiments} \label{sec:experiments}

All code to reproduce these experiments is available in the following GitHub repository:
\begin{center}
    \url{https://github.com/sanjitdp/reward-guidance}.
\end{center}
All experiments run on a single NVIDIA RTX A6000 or L40S GPU, and each image takes less than 1.5 minutes to generate; see \Cref{sec:exp-compute} for further descriptions of the compute used.

\subsection{Within-mode reward hacking (FLUX.1: text-to-image)} \label{sec:experiments-within-mode}

\begin{figure}[!b]
    \centering
    \includegraphics[width=0.95\linewidth]{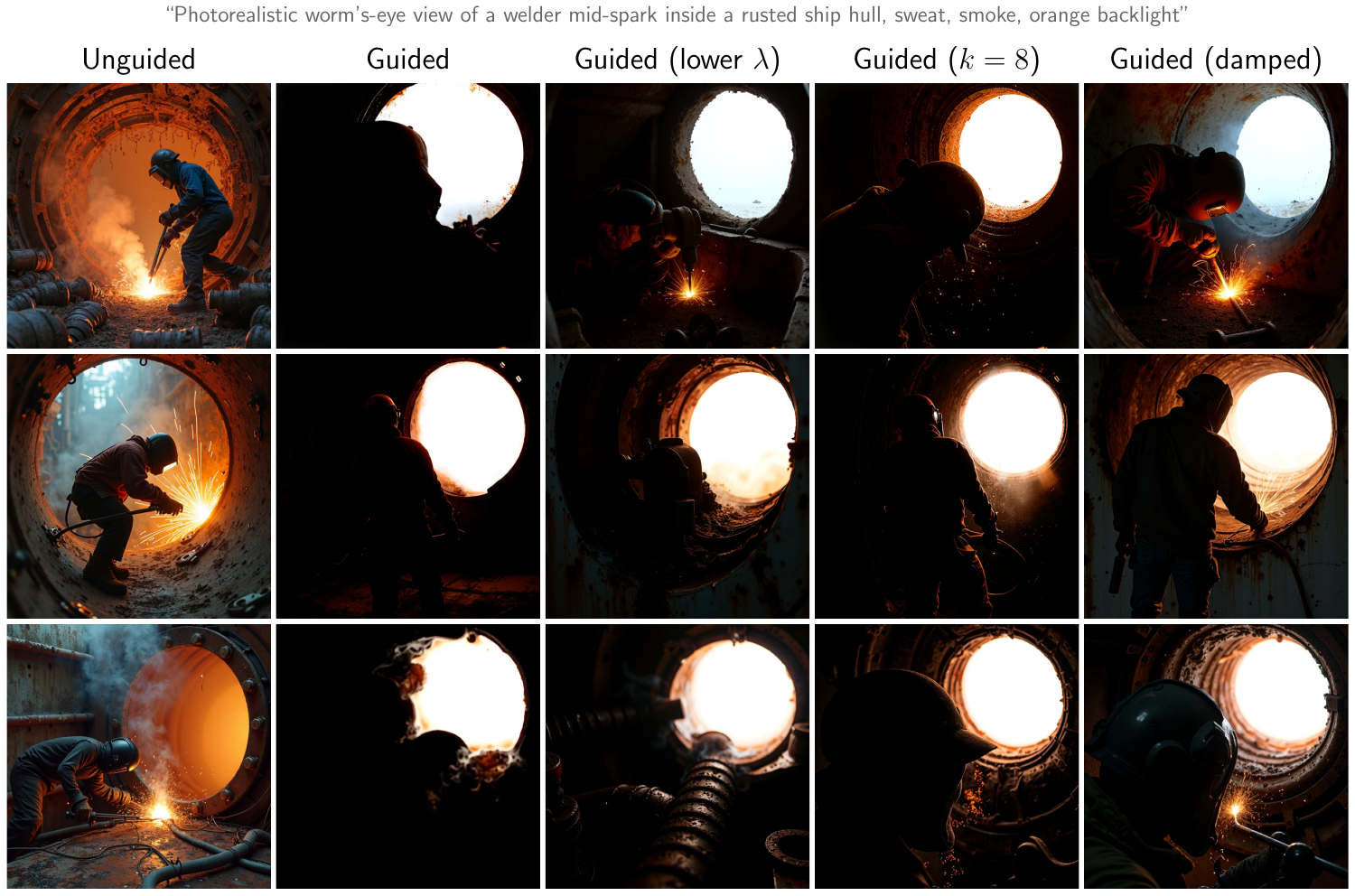}
    \caption{\textbf{Masked intensity reward.} The reward is the mean pixel intensity inside a top-right circular mask minus the mean intensity outside. Naive guidance maximizes the masked-region brightness by removing the welder from the frame; reward damping obtains a high-reward sample including the welder.\label{fig:flux-brightness}}
\end{figure}

\begin{figure}[!b]
    \centering
    \includegraphics[width=0.95\linewidth]{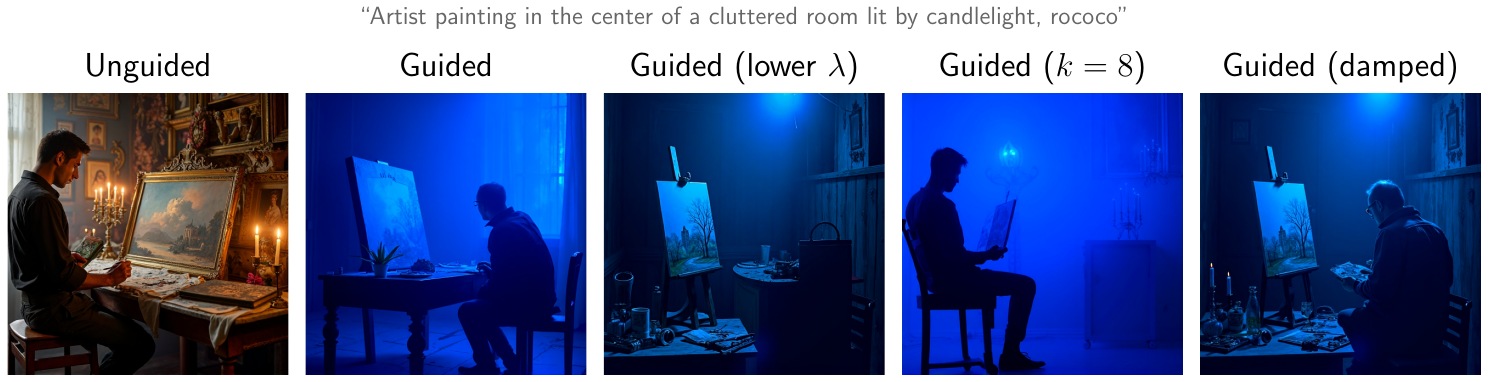}
    \caption{\textbf{Blueness reward (artist).} The reward is the mean blue channel minus the mean of the red and green channels. Naive guidance produces blue images that forget the artist or candles, while reward damping produces a very blue image that retains the artist and shows the warm candlelight.\label{fig:flux-color}}
\end{figure}

\begin{figure}[!b]
    \centering
    \includegraphics[width=0.95\linewidth]{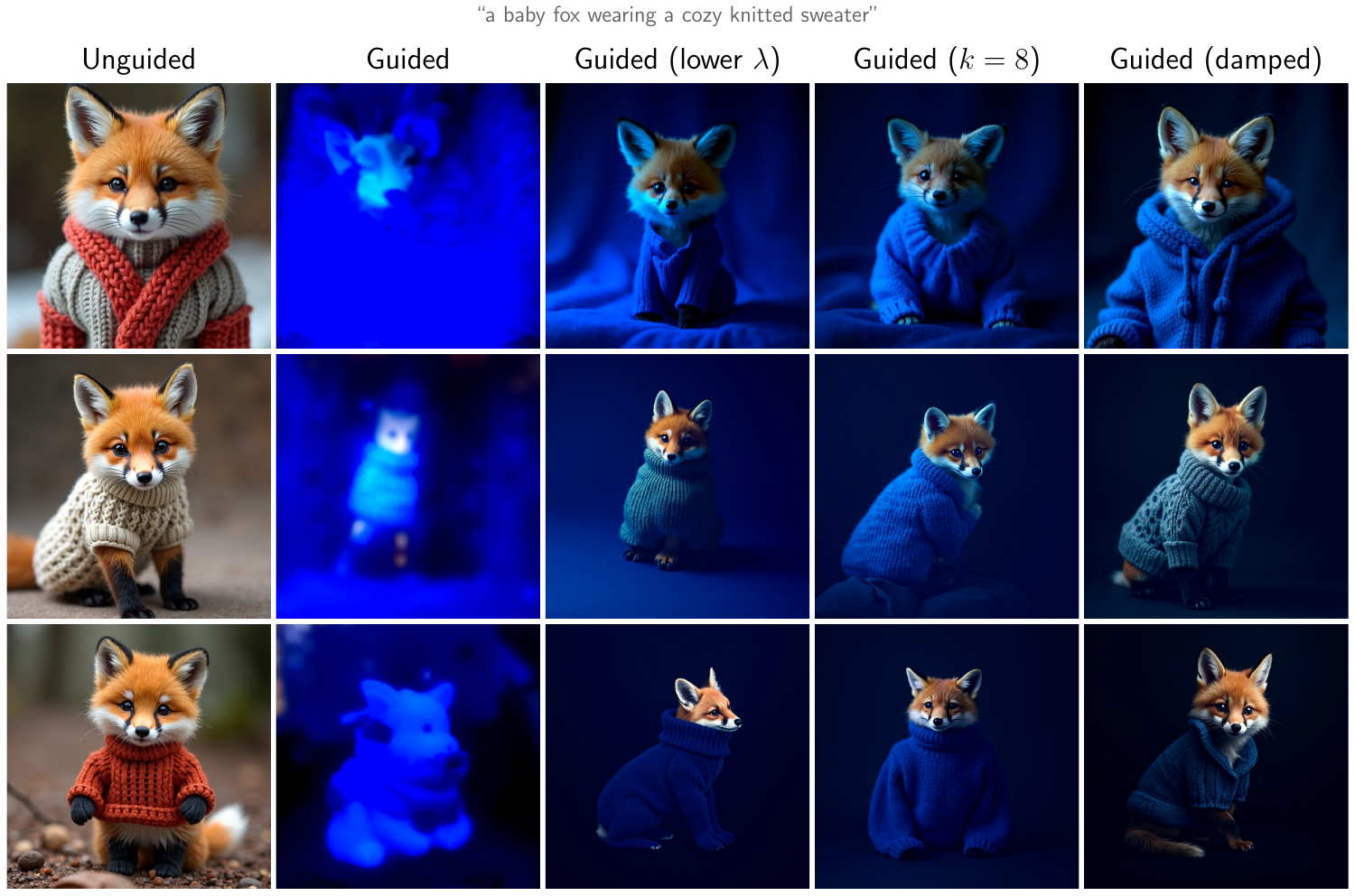}
    \caption{\textbf{Blueness reward (fox).} Naive guidance creates overwhelmingly blue outputs that wash out the fox's orange fur; reward damping recognizes that the fox should remain orange while the sweater and background turn blue.\label{fig:flux-fox}}
\end{figure}

\image{0.95}{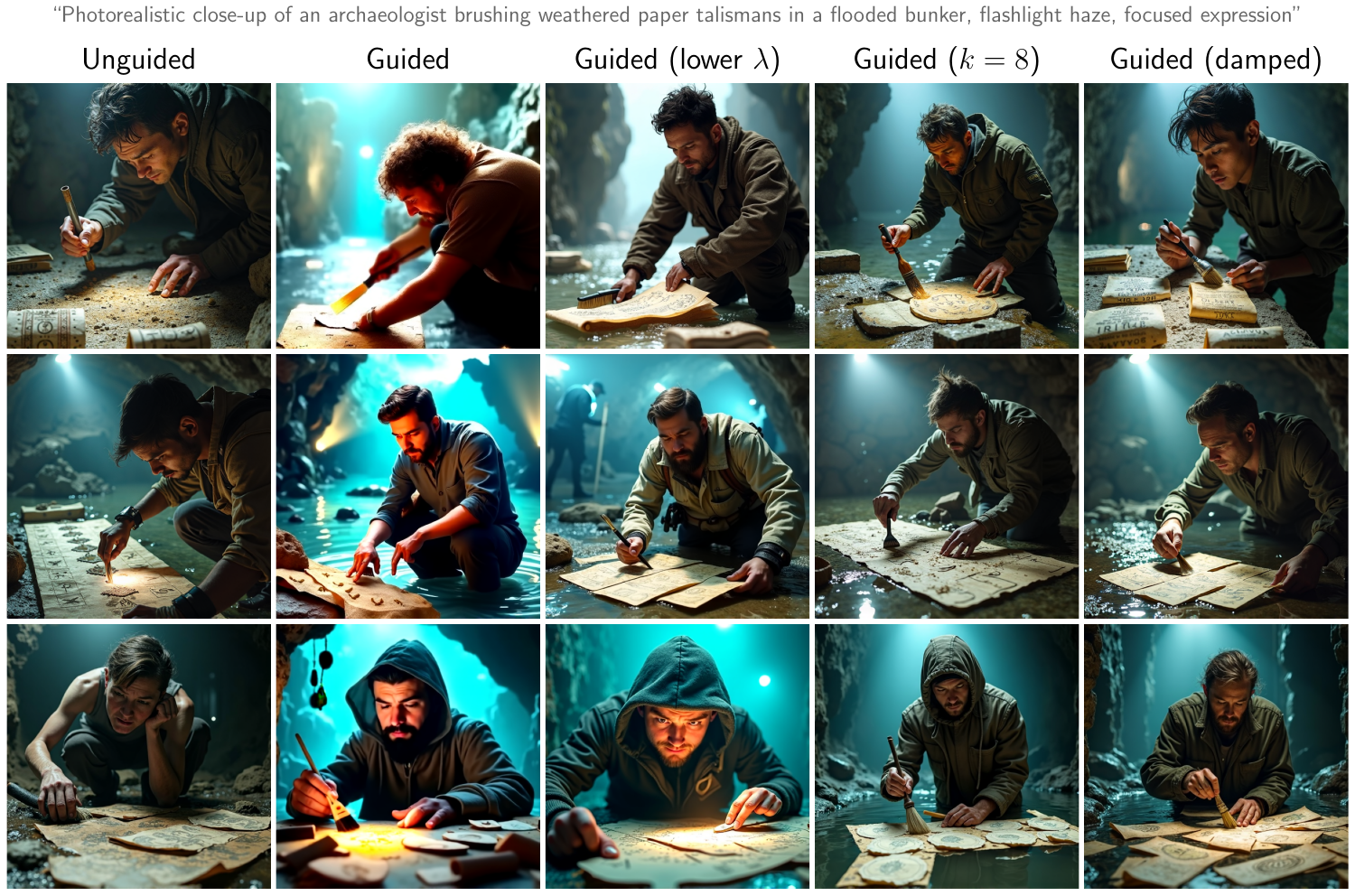}{\textbf{ImageReward (archaeologist).} We guide using ImageReward \citep{xu2023imagereward}, a learned human-preference reward. Both the unguided and naively guided images produce unnatural images or fail to show a brush, while the damped guidance produces a more natural image with a visible brush.}{fig:flux-archaeo}

\image{0.95}{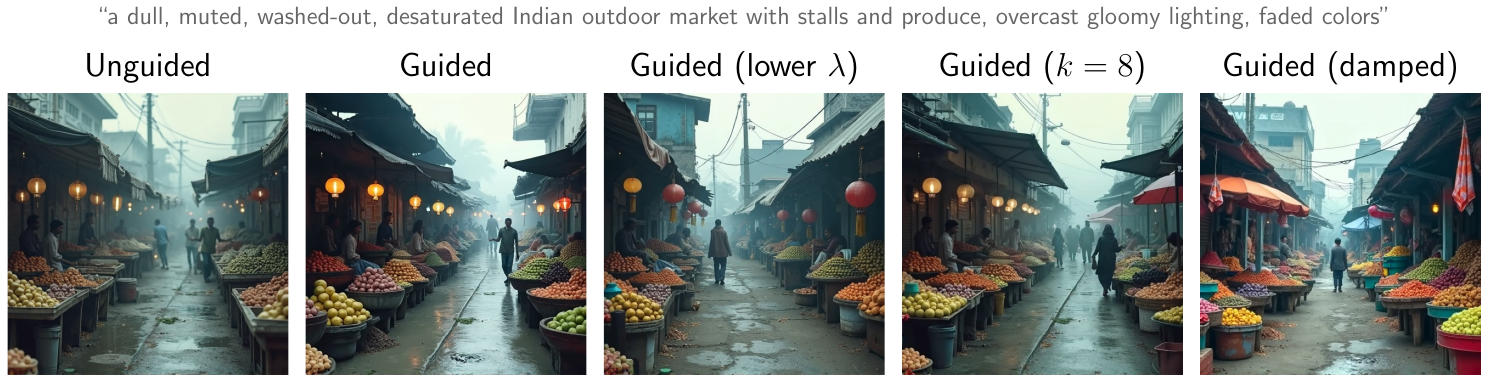}{\textbf{ImageReward (market).} The image is scored against the prompt ``a vibrant Indian outdoor market with colorful stalls and produce.'' Naive guidance brightens the lamps while leaving the rest of the image washed out; reward damping produces a more balanced improvement across brightness and color.}{fig:flux-market}

\image{0.95}{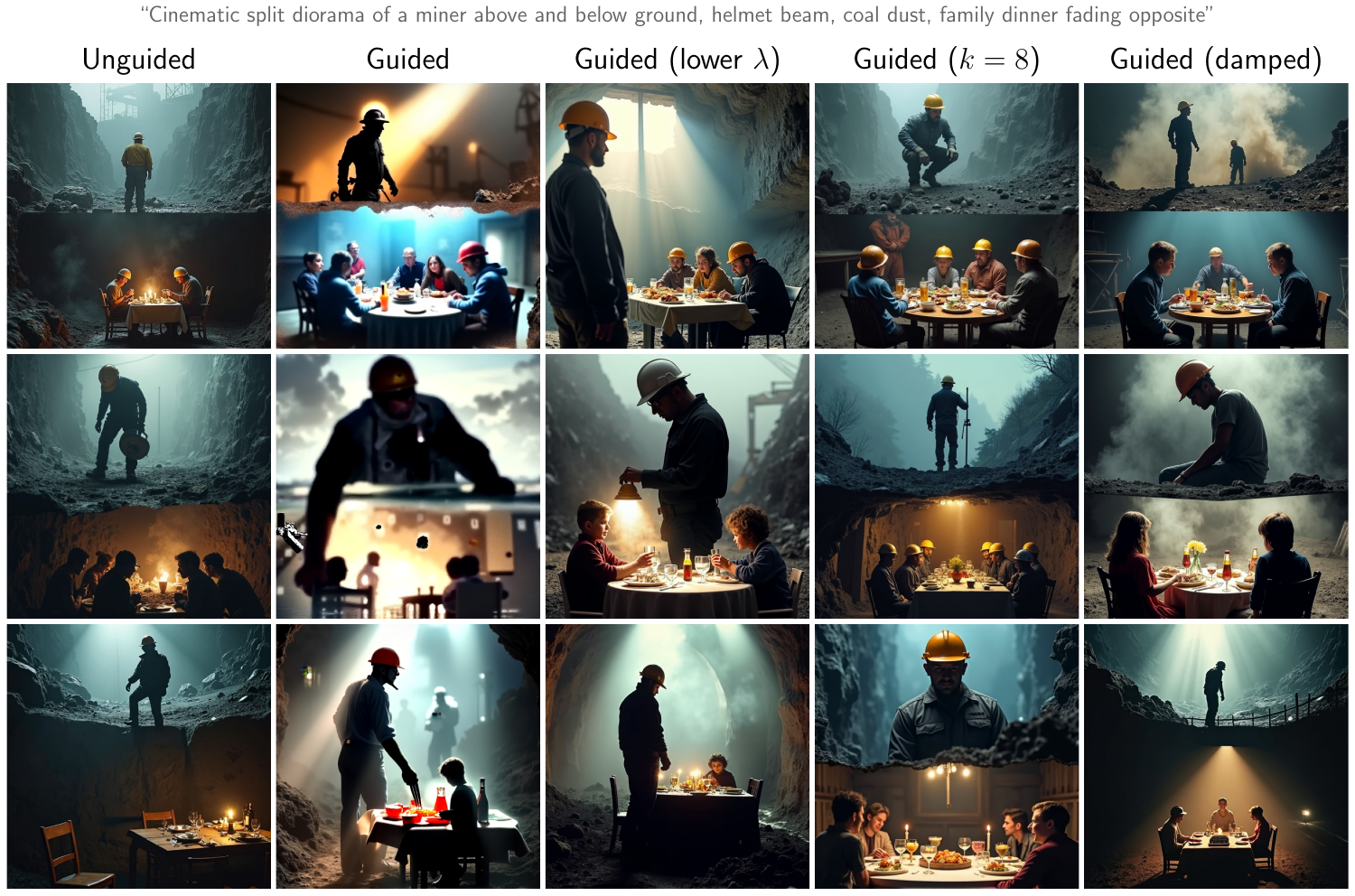}{\textbf{ImageReward (miner).} Naive guidance hacks ImageReward by ignoring the split diorama constraint; reward damping produces dramatic images that still respect the constraint.}{fig:flux-miner}

We apply guidance to FLUX.1-dev \citep{esser2024scaling}, a state-of-the-art 12-billion-parameter text-to-image diffusion transformer using Diamond maps \citep{holderrieth2026diamond} for lookahead; further details are provided in \Cref{sec:exp-flux-details}. Across our experiments, damped guidance (\Cref{sec:reward-hacking}) matches or exceeds the reward hacking mitigation of plug-in with $k = 8$ particles using $8\times$ less compute per step. \Cref{fig:flux-brightness} contains an experiment with a masked-intensity reward, \Cref{fig:flux-color,fig:flux-fox} contain experiments with a blueness reward, and \Cref{fig:flux-archaeo,fig:flux-miner,fig:flux-market} contain experiments with ImageReward \citep{xu2023imagereward}, a learned human-preference reward.

\paragraph{Base model.} All FLUX experiments use the pretrained \texttt{black-forest-labs/FLUX.1-dev} checkpoint at $512 \times 512$ resolution with $28$ inference steps, the standard FLUX guidance scalar $3.5$, and the FLUX flow-matching schedule. The posterior sampling from the law of $(X_1 \given X_t = x)$ is computed with a Diamond map \citep{holderrieth2026diamond} with $5$ inner steps.

\paragraph{Reward functions.} We use three rewards: the blueness reward $r(x) = \bar{x}_{\mathrm{blue}} - \bar{x}_{\mathrm{red}} - \bar{x}_{\mathrm{green}}$ (mean blue-channel excess), the masked-brightness reward (mean intensity inside a top-right circular mask, minus the mean intensity outside), ImageReward \citep{xu2023imagereward} (a learned human-preference scorer applied with the BLIP visual encoder, BLIP text encoder, and reward MLP exposed in-graph for backpropagation).

\paragraph{Results.} For the masked-intensity reward (\Cref{fig:flux-brightness}), naive guidance removes the welder to brighten the masked region. Damping reaches a high reward while keeping the welder in the scene. For the blueness reward, naive guidance produces a flat blue image that drops the artist and the candles (\Cref{fig:flux-color}), and washes the orange out of the fox's fur (\Cref{fig:flux-fox}). Damping produces very blue images that keep the artist visible with the warm candlelight, and that leave the fox's fur orange while turning the sweater and background blue. For ImageReward, naive guidance hacks the reward in three different ways: it produces an unnatural archaeologist scene with no brush (\Cref{fig:flux-archaeo}), brightens the market lamps and leaves the rest of the scene washed out (\Cref{fig:flux-market}), and produces a single dramatic miner scene that ignores the split-diorama constraint (\Cref{fig:flux-miner}). With damping, the archaeologist holds a brush, the market improves in brightness and color evenly, and the miner images stay dramatic while respecting the split-diorama constraint.

\subsection{Mode selection} \label{sec:experiments-mode-selection}

\subsubsection{Checkerboard.} \label{sec:experiments-mode-selection-checkerboard}

\begin{table}[!b]
    \centering
    \caption{\textbf{Checkerboard statistics.} Mean reward and covariance trace for each method ($\lambda = 10$). Uncertainties are $\pm 2$ standard errors of the mean; covariance-trace uncertainties are bootstrapped.\label{tab:cb-stats}}
    \small
    \setlength{\tabcolsep}{4pt}
    \begin{tabular}{lcc@{\hspace{0.9em}}|@{\hspace{0.9em}}lcc}
        \toprule
        Method          & Mean reward           & Cov.\ trace           & Method                                      & Mean reward           & Cov.\ trace           \\
        \midrule
        Analytic tilt   & $0.914_{\,\pm 0.003}$ & $0.457_{\,\pm 0.021}$ & Best-of-$2$                                 & $0.914_{\,\pm 0.003}$ & $0.520_{\,\pm 0.026}$ \\
        Plug-in ($k=1$) & $0.764_{\,\pm 0.004}$ & $1.930_{\,\pm 0.048}$ & Best-of-$4$                                 & $0.978_{\,\pm 0.002}$ & $0.107_{\,\pm 0.010}$ \\
        Plug-in ($k=8$) & $0.804_{\,\pm 0.007}$ & $1.397_{\,\pm 0.068}$ & Best-of-$4$, $\sigma_{\mathrm{damp}} = 0.2$ & $0.920_{\,\pm 0.004}$ & $0.429_{\,\pm 0.022}$ \\
        \bottomrule
    \end{tabular}
\end{table}

\begin{figure}[!tb]
    \centering
    \includegraphics[width=0.95\linewidth]{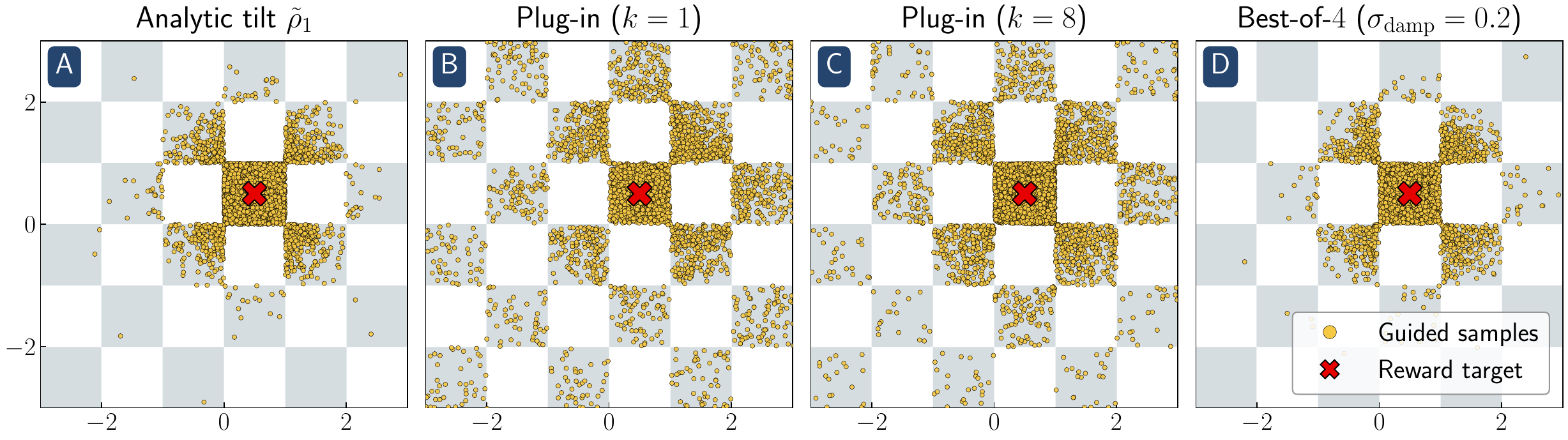}
    \caption{\textbf{Checkerboard guidance.} Unlike plug-in guidance, best-of-$n$ can select modes. With reward damping, best-of-$n$ significantly improves fidelity to the analytic tilt.\label{fig:cb-guidance-failure}}
\end{figure}

We first demonstrate the failure of plug-in guidance for mode selection using a base model sampling uniformly from a 2D checkerboard, where the reward is a Gaussian bump $r(x) = \exp(-\norm{x - c}_2^2 / (2\sigma_r^2))$. \Cref{fig:cb-guidance-failure} and \Cref{tab:cb-stats} show that plug-in estimators fail to concentrate samples near $c$, while a combination of best-of-$n$ and damping improves fidelity to $\tilde{\rho}_1$. We provide further details on the experimental setup and additional results in \Cref{sec:exp-checkerboard-details}.

\subsubsection{FLUX.1: text-to-image.} \label{sec:experiments-mode-selection-flux}

\begin{figure}[!b]
    \centering
    \includegraphics[width=0.98\linewidth]{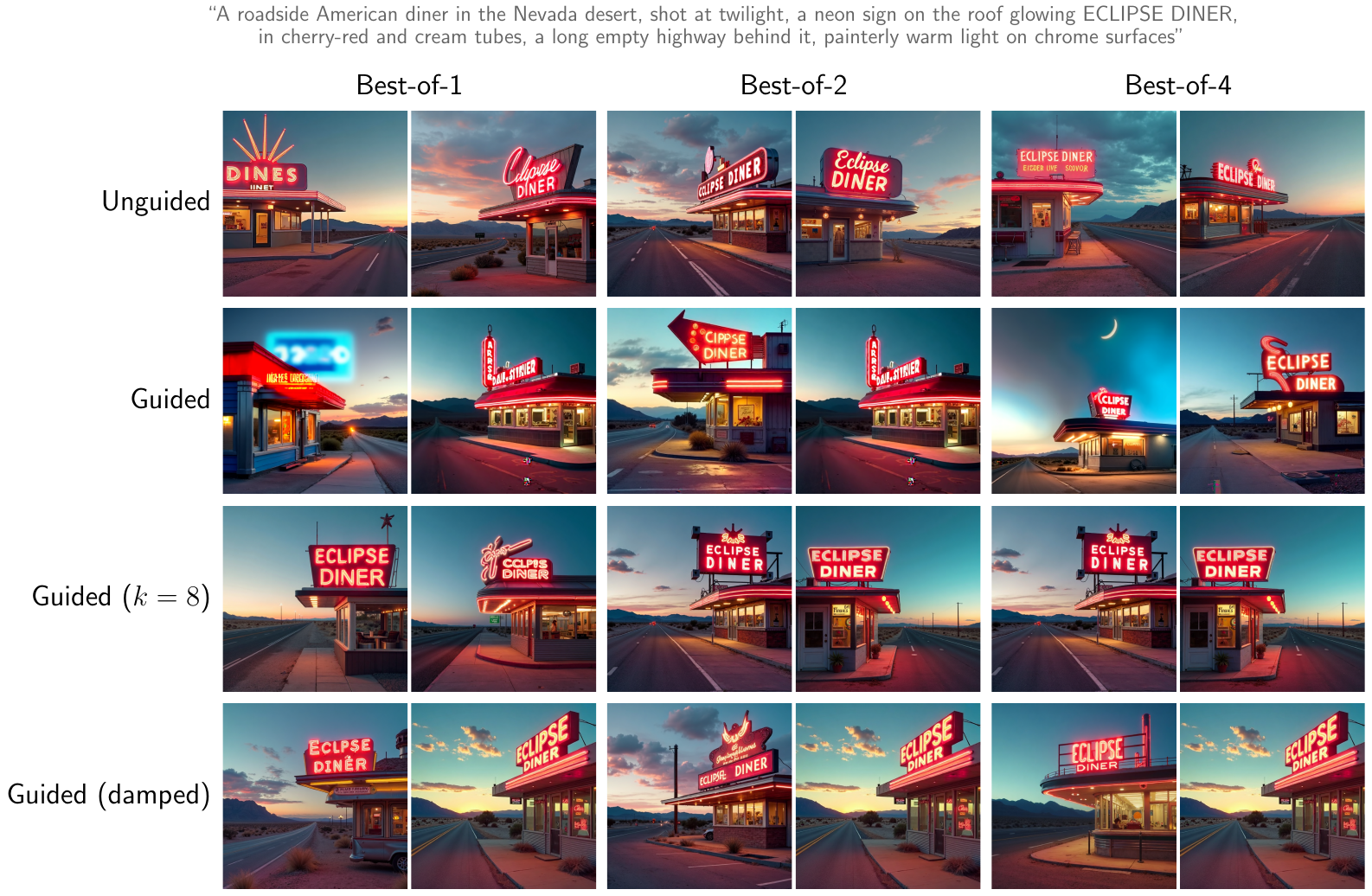}
    \caption{\textbf{FLUX mode selection (ECLIPSE DINER).} The VLM reward is derived from the question ``Does this image clearly show a neon sign with the word `ECLIPSE' as the main readable text?'' Finite-particle guidance hacks the complex reward function, damping slightly improves the reward, and best-of-$n$ substantially improves the reward, confirming the importance of the initial seed for mode selection.\label{fig:ms-flux-diner}}
\end{figure}

\image{0.98}{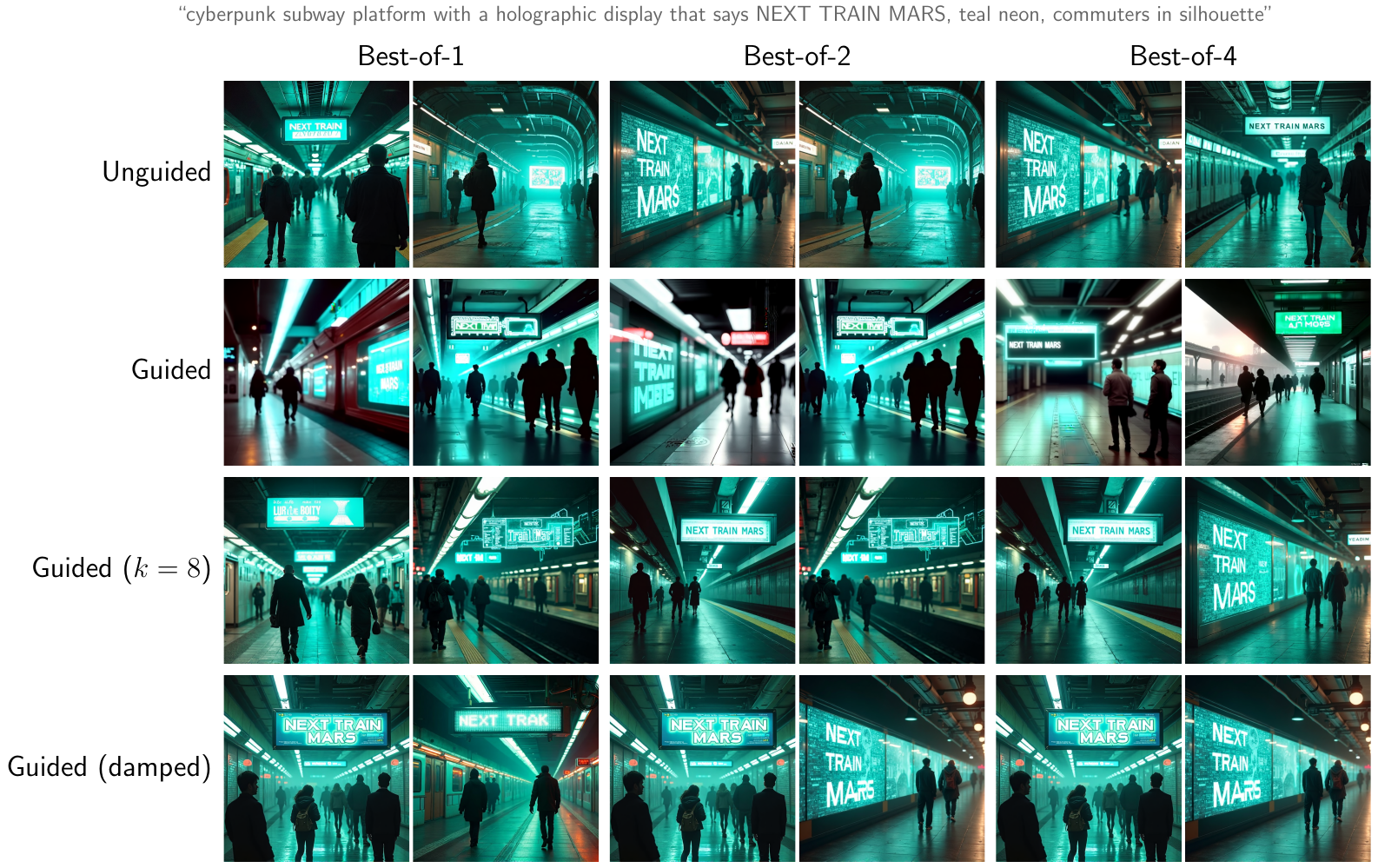}{\textbf{FLUX mode selection (NEXT TRAIN MARS).} The VLM reward is derived from the question ``Does this image clearly show a display with the text `NEXT TRAIN MARS' as the main readable text?'' We see the qualitative pattern of \Cref{fig:ms-flux-diner}: plug-in guidance hacks the reward, and best-of-$n$ substantially improves it.}{fig:ms-flux-subway}

In \Cref{fig:ms-flux-diner,fig:ms-flux-subway}, we use FLUX.1-dev \citep{esser2024scaling} with a VLM reward using Qwen2.5-VL-3B \citep{bai2025qwen25vl} to demonstrate that plug-in guidance cannot perform mode selection in a realistic setting. Here, the reward is $r(x) = \log(p(\mathrm{Yes})) - \log(p(\mathrm{No}))$, where $p(\cdot)$ denotes the VLM next-token probability. We provide numerics in \Cref{app:flux-mode-selection-additional}.

\section{Conclusion}

We demonstrate that reward hacking arises from the plug-in estimator commonly used in reward-guided generative modeling even in simple Gaussian settings, and showed that plug-in guidance cannot select between distant high-reward modes for Gaussian mixtures. We proposed reward damping, a principled time-dependent reward scale that corrects the within-mode bias at no additional cost, and clarified the role of best-of-$n$ in compensating for the mode selection failure. Our experiments in Gaussian mixture, checkerboard, and text-to-image settings demonstrate the limitations of guidance and showcase the usefulness of our methods for overcoming these limitations.

\paragraph{Limitations and future work.}
Our closed-form theoretical results are derived for Gaussian and Gaussian-mixture targets with quadratic rewards under the memoryless noise schedule, and the reward damping schedule itself is obtained from the Gaussian analysis; our empirical study is limited to a 2D checkerboard and FLUX.1 text-to-image generation. Several directions remain open: an analytic study of plug-in bias for non-Gaussian targets and non-quadratic rewards, an extension of our analysis to broader reward-tilted-sampling problems in reinforcement learning beyond diffusion, and practical applications of reward damping to settings such as preference fine-tuning, controlled molecular design, and constrained protein generation.

\section*{Acknowledgements}

We thank Aaditya Ramdas, Sivaraman Balakrishnan, and Arun Kuchibhotla for providing helpful feedback on earlier versions of this work. We also thank Jerry Huang and Peter Holderrieth for useful discussions about text-to-image guidance experiments.

\bibliographystyle{unsrtnat}
\bibliography{bibliography}

\newpage
\appendix
\crefalias{section}{appendix}
\section{Further background} \label{app:further-background}

\subsection{Forward SDE matches the probability flow time-marginals} \label{app:sde-matches-flow}

\begin{proposition}[Time-marginals of the forward SDE]
    The solution of the forward SDE \eqref{eq:unguided-diffusion} has the same time-marginal distribution as the probability flow ODE $\dot{x}_t = b_t(x_t)$.
\end{proposition}
\begin{proof}
    The Fokker--Planck equation associated with the SDE \eqref{eq:unguided-diffusion} is
    \begin{align*}
        \partial_t \rho_t = -\grad \cdot \left( \rho_t b_t + \frac{1}{2} \grad \cdot (\rho_t\, \sigma_t \sigma_t^\top) \right) + \frac{1}{2} \grad \cdot (\grad \cdot (\rho_t\, \sigma_t \sigma_t^\top)) = -\grad \cdot (\rho_t b_t),
    \end{align*}
    which matches the continuity equation for the probability flow ODE.
\end{proof}

\subsection{Reward-tilted measure as a KL-regularized variational problem} \label{app:variational-tilt}

\begin{proposition}[Variational characterization of the reward tilt]
    The reward-tilted measure \eqref{eq:reward-tilted-measure} solves the KL-regularized variational problem
    \begin{align*}
        \tilde{\rho}_1 = \argmax_\rho\, \set*{\E_{X \sim \rho}[r(X)] - \frac{1}{\lambda}\, \KL(\rho \compare \rho_1)}.
    \end{align*}
\end{proposition}
\begin{proof}[Proof outline]
    Writing the objective as a single integral against $\rho_1$,
    \begin{align*}
        \E_{X \sim \rho}[r(X)] - \frac{1}{\lambda}\, \KL(\rho \compare \rho_1)
        = \frac{1}{\lambda}\, \E_{X \sim \rho_1}\!\left[ \frac{\rho(X)}{\rho_1(X)} \left( \lambda r(X) - \log \frac{\rho(X)}{\rho_1(X)} \right) \right].
    \end{align*}
    Pointwise maximization in the density ratio $f(x) = \rho(x) / \rho_1(x)$ subject to $\E_{X \sim \rho_1}[f(X)] = 1$ yields $f(x) \propto e^{\lambda r(x)}$ via a Lagrange multiplier, so the optimal $\rho$ is exactly $\tilde{\rho}_1(x) \propto \rho_1(x)\, e^{\lambda r(x)}$.
\end{proof}

\subsection{Connections to stochastic optimal control} \label{app:stochastic-optimal-control}

The Doob $h$-transform approach to reward guidance can also be interpreted through the lens of stochastic optimal control. We start with an uncontrolled diffusion
\begin{align*}
    dX_t = \mu_t(X_t)\, dt + \sigma_t\, dB_t
\end{align*}
where $\mu_t(x) \coloneq b_t(x) + \frac{1}{2} (\sigma_t \sigma_t^\top) \grad \log \rho_t(x)$ and define a controlled process under a new measure $Q_u$ by
\begin{align*}
    dX_t = (\mu_t(X_t) + \sigma_t u_t(X_t))\, dt + \sigma_t\, d\tilde{B}_t.
\end{align*}
Now, the goal is to maximize the expected reward while minimizing the control cost, over all admissible controls $u$:
\begin{align*}
    \max_u\, \E_{Q_u}\!\left[ r(X_1) - \frac{1}{2 \lambda} \int_0^1 \norm{u_t(X_t)}_2^2\, dt \right].
\end{align*}
We then define the stochastic value function
\begin{align*}
    V_t(x) = \sup_u\, \E_{Q_u}\!\left[ r(X_1) - \frac{1}{2 \lambda} \int_t^1 \norm{u_s(X_s)}_2^2\, ds \givenlarge X_t = x \right],
\end{align*}
noting that it must satisfy the Hamilton--Jacobi--Bellman (HJB) equation
\begin{align*}
    \partial_t V_t(x) + \sup_u\, \set*{ \left( \mu_t(x) + \sigma_t u \right)^\top \grad V_t(x) + \frac{1}{2} \tr((\sigma_t \sigma_t^\top)\, \grad^2 V_t(x)) - \frac{1}{2 \lambda} \norm{u}_2^2 } = 0
\end{align*}
with the boundary condition $V_1(x) = r(x)$. Then, stationarity implies that the optimal control is given by $u_t^*(x) = \lambda\, \sigma_t^\top \grad V_t(x)$, and substituting back into the HJB equation yields the nonlinear PDE
\begin{align*}
    \partial_t V_t(x) + \mu_t(x)^\top \grad V_t(x) + \frac{1}{2} \tr((\sigma_t \sigma_t^\top)\, \grad^2 V_t(x)) + \frac{\lambda}{2} \norm{\sigma_t^\top \grad V_t(x)}_2^2 = 0.
\end{align*}
The only source of nonlinearity in this equation comes from the control cost term $\frac{\lambda}{2} \norm{\sigma_t^\top \grad V_t(x)}_2^2$, so we apply the Cole--Hopf transformation $h_t(x) = e^{\lambda V_t(x)}$ to linearize the PDE:
\begin{align*}
    \frac{\partial_t h_t(x)}{\lambda h_t(x)} + \mu_t^\top \left( \frac{\grad h_t(x)}{\lambda h_t(x)} \right) + \frac{\lambda}{2} \norm*{\sigma_t^\top \left( \frac{\grad h_t(x)}{\lambda h_t(x)} \right)}_2^2 + \frac{1}{2} \tr\!\left( \sigma_t \sigma_t^\top \left( \frac{\grad^2 h_t(x)}{\lambda h_t(x)} - \frac{\grad h_t(x) (\grad h_t(x))^\top}{\lambda h_t(x)^2} \right) \right) = 0.
\end{align*}
Here, note that
\begin{align*}
    \frac{1}{2 \lambda h_t(x)^2} \tr(\sigma_t \sigma_t^\top \grad h_t(x) (\grad h_t(x))^\top)
    = \frac{1}{2 \lambda h_t(x)^2} \norm*{\sigma_t^\top \grad h_t(x)}_2^2
\end{align*}
by the cyclic property of trace, so the nonlinear terms cancel and we are left with the linear PDE
\begin{align*}
    \partial_t h_t(x) + \mu_t(x)^\top \grad h_t(x) + \frac{1}{2} \tr((\sigma_t \sigma_t^\top)\, \grad^2 h_t(x)) = 0.
\end{align*}
Recalling the formula for the generator $\mc{L}_t$ of the unguided diffusion from~\eqref{eq:unguided-generator}, it follows that $h_t$ must satisfy the Kolmogorov backward equation $\partial_t h_t + \mc{L}_t h_t = 0$. Noting that $h_t(x) = \E[e^{\lambda r(X_1)} \given X_t = x]$ also satisfies this PDE with boundary condition $h_1(x) = e^{\lambda r(x)}$, we conclude that the stochastic value function is given by $V_t(x) = \frac{1}{\lambda} \log h_t(x)$ and the optimal control is given by $u_t^*(x) = \sigma_t^\top \grad \log h_t(x)$, which matches the Doob $h$-transform guidance term.

\subsection{GLASS flows for transition sampling} \label{app:glass-flows}

In this section, we describe the GLASS flow framework introduced by \citet{holderrieth2026glass} for sampling from the law of $(I_1 \given I_t = x)$. Instead of simulating the forward SDE, which is computationally expensive and hard to differentiate through, GLASS flows simulate a deterministic ODE starting from a Gaussian initial condition. This way, transition sampling is much cheaper to simulate, more numerically stable, and easier to differentiate through. The GLASS framework operates under the assumption that the initial distribution $\rho_0$ is Gaussian and that we have access to the optimal denoiser $D_t(x) \coloneq \E[I_1 \given I_t = x]$. Although the GLASS flow is defined more generally, we specialize it here to the setting of linear stochastic interpolants.

\begin{theorem}[GLASS flow for transition sampling] \label{thm:glass-flow}
    Consider the linear interpolant $I_t = (1 - t) X_0 + t X_1$ for $t \in [0, 1]$ with $X_0 \sim \Normal(0, I_d)$. For $s, t \in [0, 1)$, define the combined signal-to-noise ratio
    \begin{align*}
        \eta^*_s \coloneq \sqrt{\left( \frac{t}{1 - t} \right)^2 + \left( \frac{s}{1 - s} \right)^2}.
    \end{align*}
    To sample from the law of $(I_1 \given I_t = x)$, we can simulate the ODE
    \begin{align*}
        d\bar{X}_s = \frac{1}{1 - s}\, (D_{\tau_s^*}(X^*_s(x)) - \bar{X}_s)\, ds,
    \end{align*}
    from $s = 0$ to $s = 1$ starting from $\bar{X}_0 \sim \Normal(0, I_d)$, where
    \begin{align*}
        \tau_s^* \coloneq \frac{\eta^*_s}{1 + \eta^*_s}, \qquad X^*_s(x) \coloneq \frac{1}{\eta^*_s (1 + \eta^*_s)} \left( \frac{t}{(1 - t)^2}\, x + \frac{s}{(1 - s)^2}\, \bar{X}_s \right).
    \end{align*}
\end{theorem}

Let $X_t$ denote the solution to the forward SDE~\eqref{eq:unguided-diffusion} with initial condition $X_0 \sim \Normal(0, I_d)$. Then, note that the GLASS flow will work for transition sampling as long as the memoryless noise schedule is used, since the law of $(I_1 \given I_t = x)$ then matches the law of $(X_1 \given X_t = x)$ by \Cref{thm:memoryless-schedule}.

\begin{proof}
    Define another linear interpolant $J_s = (1 - s)\, \epsilon + s I_1$, where $\epsilon \sim \Normal(0, I_d)$ is independent of everything else; note that $(J_0 \given I_t = x) \sim \Normal(0, I_d)$ and $(J_1 \given I_t = x)$ follows the law of $(I_1 \given I_t = x)$. Then, the conditional probability flow ODE for $(J_s \given I_t = x)$ is given by
    \begin{align*}
        d\bar{X}_s
        = \frac{1}{1 - s}\, (\E[I_1 \given J_s = \bar{X}_s,\, I_t = x] - \bar{X}_s)\, ds.
    \end{align*}
    Therefore, it remains to show that conditioning on $J_s = \bar{X}_s$ and $I_t = x$ is equivalent to conditioning on $I_{\tau_s^*} = X^*_s(x)$ for an appropriate choice of effective time $\tau_s^*$ and effective observation $X^*_s(x)$. Since $(I_t \given I_1) \sim \Normal(t I_1, (1 - t)^2 I_d)$ and $(J_s \given I_1) \sim \Normal(s I_1, (1 - s)^2 I_d)$ are independent, their joint density is given by
    \begin{align*}
        p(I_t = x, J_s = \bar{X}_s \given I_1)
         & \propto \exp\!\left( -\frac{1}{2 (1 - t)^2} \norm{x - t I_1}_2^2 - \frac{1}{2 (1 - s)^2} \norm{\bar{X}_s - s I_1}_2^2 \right)                                                                        \\
         & \propto \exp\!\left( -\frac{1}{2} I_1^\top \left( \frac{t^2}{(1 - t)^2} + \frac{s^2}{(1 - s)^2} \right) I_1 + I_1^\top \left( \frac{t}{(1 - t)^2} x + \frac{s}{(1 - s)^2} \bar{X}_s \right) \right).
    \end{align*}
    On the other hand, the density of $(I_{\tau_s^*} \given I_1)$ is given by
    \begin{align*}
        p(I_{\tau_s^*} = X^*_s(x) \given I_1)
         & \propto \exp\!\left( -\frac{1}{2 (1 - \tau_s^*)^2} \norm{X^*_s(x) - \tau_s^* I_1}_2^2 \right)                                                                                     \\
         & \propto \exp\!\left( -\frac{1}{2} I_1^\top \left( \frac{(\tau_s^*)^2}{(1 - \tau_s^*)^2} \right) I_1 + I_1^\top \left( \frac{\tau_s^*}{(1 - \tau_s^*)^2} X^*_s(x) \right) \right).
    \end{align*}
    It suffices now to show that the two densities are proportional to each other, since Bayes' rule will then imply that $(I_1 \given I_t = x,\, J_s = \bar{X}_s) \overset{d}{=} (I_1 \given I_{\tau_s^*} = X^*_s(x))$. Equating the quadratic coefficients gives
    \begin{align*}
        \frac{\tau_s^*}{1 - \tau_s^*} = \eta^*_s
        \implies \tau_s^* = \frac{\eta^*_s}{1 + \eta^*_s},
    \end{align*}
    and because $\tau_s^* / (1 - \tau_s^*)^2 = \eta^*_s (1 + \eta^*_s)$, equating the linear coefficients gives
    \begin{align*}
        X^*_s(x) = \frac{1}{\eta^*_s (1 + \eta^*_s)} \left( \frac{t}{(1 - t)^2} x + \frac{s}{(1 - s)^2} \bar{X}_s \right). \tag*{\qedhere}
    \end{align*}
\end{proof}

\section{Additional results} \label{app:additional-results}

\subsection{Analytic guidance for the Gaussian target} \label{app:analytic-guidance}

In this appendix, we compute the exact Doob $h$-transform guidance for a Gaussian target distribution under the memoryless noise schedule, providing a benchmark for the biases that we study in \Cref{sec:reward-hacking}.

\begin{proposition}[Analytic guidance for Gaussian target] \label{prop:analytic-guidance-gaussian}
    Suppose $I_1 \sim \Normal(\mu, \Sigma)$ and $\sigma_t = \sqrt{2 (1 - t) / t}\, I_d$ is the memoryless noise schedule. The guidance term of the Doob $h$-transform is
    \begin{align*}
        \grad \log h_t(x) = -2 \lambda t\, \Sigma \Sigma_t^{-1} (I_d + 2 \lambda \Sigma_{1 \vert t})^{-1} (\mu_{1 \vert t}(x) - a),
    \end{align*}
    where $\mu_{1 \vert t}(x) = \mu + t \Sigma \Sigma_t^{-1} (x - t \mu)$ and $\Sigma_{1 \vert t} = (1 - t)^2 \Sigma \Sigma_t^{-1}$ are the mean and covariance of $(X_1 \given X_t = x)$. The terminal distribution of the guided ODE is the reward-tilted measure
    \begin{align*}
        \tilde{x}_1 \sim \Normal\!\left(\mu - 2 \lambda \Sigma (I_d + 2 \lambda \Sigma)^{-1} (\mu - a),\; \Sigma (I_d + 2 \lambda \Sigma)^{-1}\right).
    \end{align*}
\end{proposition}
\begin{proof}
    The last part of the corollary follows immediately from \Cref{thm:ivf-bias-gaussian} by noting that $\Psi = 0$ for the memoryless noise schedule. Note that for the memoryless noise schedule $\eta_t = \sqrt{2 (1 - t) / t}$,
    \begin{align*}
        \frac{d}{dv}\, \log(v^2\, \Sigma \Sigma_v^{-1})
         & = \frac{2}{v} I_d - \Sigma_v^{-1} \dot{\Sigma}_v                                                 \\
         & = \Sigma_v^{-1} \left( 2v \Sigma + \frac{2 (1 - v)^2}{v} I_d - 2v \Sigma + 2 (1 - v) I_d \right) \\
         & = \frac{2 (1 - v)}{v} \Sigma_v^{-1}                                                              \\
         & = \eta_v^2\, \Sigma_v^{-1}.
    \end{align*}
    Therefore, we can compute the integral defining $\Psi_t$ in closed form:
    \begin{align*}
        \Psi_t = \exp\!\left( -\int_t^1 \eta_v^2\, \Sigma_v^{-1}\, dv \right)
        = \exp(\log(t^2\, \Sigma \Sigma_t^{-1}))
        = t^2\, \Sigma \Sigma_t^{-1}.
    \end{align*}
    Plugging this back into the covariance formula, we obtain $\Cov(X_t, X_1) = t \Sigma$. The Gaussian conditioning formula states that $(X_1 \given X_t = x) \sim \Normal(\mu_{1 \vert t}(x), \Sigma_{1 \vert t})$ where $\mu_{1 \vert t}(x) = \mu + t \Sigma \Sigma_t^{-1} (x - t \mu)$ and $\Sigma_{1 \vert t} = (1 - t)^2 \Sigma \Sigma_t^{-1}$. Using the fact that $\Sigma$ and $\Sigma_t$ are simultaneously diagonalizable, the guidance term is given by
    \begin{align*}
        \grad \log h_t(x)
        = -2 \lambda t \Sigma \Sigma_t^{-1} (I_d + 2 \lambda \Sigma_{1 \vert t})^{-1} (\mu_{1 \vert t}(x) - a). \tag*{\qedhere}
    \end{align*}
\end{proof}

\subsection{Initial value function bias} \label{sec:ivf-bias}

The Doob $h$-transform of \Cref{thm:doob-h-transform} requires a memoryless noise schedule, but the memoryless schedule of \Cref{thm:memoryless-schedule} has $\sigma_t = \sqrt{2 (1 - t) / t} \to \infty$ as $t \to 0$, so it is numerically unstable near $t = 0$; in practice, practitioners use a non-memoryless alternative. Here we characterize the resulting bias---referred to as the \emph{initial value function bias} \citep[Section 4.2]{domingo2025adjoint}---in closed form for Gaussian and Gaussian mixture targets, using the notation introduced in \Cref{sec:reward-hacking}. As discussed in the main text, this bias acts in the opposite direction to the plug-in bias: it under-corrects the mean and inflates the variance relative to the analytic tilt, so it partially mitigates plug-in reward hacking.

However, non-memoryless schedules cannot recover the analytic tilt for any choice of schedule. In particular, any non-memoryless schedule with $\Psi \succ 0$ strictly (which holds for all schedules used in practice) has both the mean pull (strictly decreasing in $\psi$) and the terminal covariance (strictly increasing in $\psi$) strictly missing their analytic-tilt values $2 \lambda \sigma / (1 + 2 \lambda \sigma)$ and $\sigma / (1 + 2 \lambda \sigma)$ along every eigendirection; no such schedule can recover the analytic tilt as the terminal distribution.

\begin{theorem}[Initial value function bias for Gaussian target] \label{thm:ivf-bias-gaussian}
    The guidance term in the Doob $h$-transform (\Cref{thm:doob-h-transform}) is
    \begin{align*}
        \grad \log h_t(x) = -2 \lambda \Sigma^{1/2} \Sigma_t^{-1/2} \Psi_t^{1/2} (I_d + 2 \lambda \Sigma_{1 \vert t})^{-1} (\mu_{1 \vert t}(x) - a).
    \end{align*}
    Defining $\Psi \coloneq \Psi_0$, the terminal distribution of the guided ODE is $\tilde{x}_1 \sim \Normal(\tilde\mu, \tilde\Sigma)$, where
    \begin{align*}
        \tilde\mu \coloneq \mu - T_{\mathrm{pull}}(\mu - a), \qquad
        T_{\mathrm{pull}} \coloneq 2 \lambda \Sigma (I_d - \Psi)(I_d + 2 \lambda \Sigma (I_d - \Psi))^{-1},
    \end{align*}
    and
    \begin{align*}
        \tilde\Sigma \coloneq \Sigma(I_d + 2 \lambda \Sigma (I_d - \Psi)^2)(I_d + 2 \lambda \Sigma (I_d - \Psi))^{-2}.
    \end{align*}
\end{theorem}

The proof is below; we express the SDE for $X_t$ as a linear system, derive a closed-form ODE for the cross-covariance $\Cov(X_s, X_t)$ via Fubini's theorem, and evaluate $h_t(x)$ as a Gaussian integral. To provide some intuition for this formula, note that for any non-memoryless schedule, $0 \prec \Psi \prec I_d$ and all operators are simultaneously diagonalizable (since $\Psi$ is a matrix function of $\Sigma$). Along a joint eigendirection $v$ with $\Psi$ having eigenvalue $\psi \in (0, 1)$ and $\Sigma$ having eigenvalue $\sigma > 0$, the eigenvalue of $T_{\mathrm{pull}}$ is
\begin{align*}
    \frac{2 \lambda \sigma (1 - \psi)}{1 + 2 \lambda \sigma (1 - \psi)},
\end{align*}
a decreasing function of $\psi$ on $[0, 1]$, equal to the memoryless tilt's pull $2 \lambda \sigma / (1 + 2 \lambda \sigma)$ at $\psi = 0$ and vanishing as $\psi \to 1$, while the eigenvalue of $\tilde\Sigma$ is
\begin{align*}
    \frac{\sigma (1 + 2 \lambda \sigma (1 - \psi)^2)}{(1 + 2 \lambda \sigma (1 - \psi))^2},
\end{align*}
an increasing function of $\psi$ on $[0, 1]$, equal to the memoryless tilt's variance $\sigma / (1 + 2 \lambda \sigma)$ at $\psi = 0$ and approaching the unguided variance $\sigma$ as $\psi \to 1$. Thus, non-memoryless schedules under-correct the mean and inflate the covariance relative to the memoryless tilt. In particular, since $0 \prec \Psi \prec I_d$ strictly for any non-memoryless schedule, both the mean pull (strictly decreasing in $\psi$) and the terminal covariance (strictly increasing in $\psi$) miss their analytic-tilt values $2 \lambda \sigma / (1 + 2 \lambda \sigma)$ and $\sigma / (1 + 2 \lambda \sigma)$ along every eigendirection; no choice of non-memoryless noise schedule can recover the analytic tilt as the terminal distribution. The corresponding analysis for a Gaussian mixture target appears as \Cref{thm:analytic-guidance-gmm} and is used in \Cref{sec:plugin-mode-selection} to characterize the long-range mode-selection behavior of the analytic $h$-transform.

\begin{proof}[Proof of \Cref{thm:ivf-bias-gaussian}] \label{app:ivf-bias-gaussian-proof}
    The unguided stochastic interpolant $I_t$ has distribution $\Normal(\mu_t, \Sigma_t)$ where $\mu_t = t \mu$ and $\Sigma_t = (1 - t)^2 I_d + t^2 \Sigma$. It is clear that $(I_t, \dot{I}_t)$ is a linear function of $(I_0, I_1)$ and hence jointly Gaussian. We then notice that
    \begin{align*}
        \Cov(\dot{I}_t, I_t) = \Cov(I_1 - I_0,\, (1 - t) I_0 + t I_1) = t \Sigma - (1 - t) I_d
    \end{align*}
    and hence the Gaussian conditional mean is given by
    \begin{align*}
        \E[\dot{I}_t \given I_t = x]
        = \mu + (t \Sigma - (1 - t) I_d) \Sigma_t^{-1} (x - \mu_t)
        = \mu + \frac{1}{2} \dot{\Sigma}_t \Sigma_t^{-1} (x - \mu_t).
    \end{align*}
    Because the score function is $\grad \log \rho_t(x) = -\Sigma_t^{-1} (x - \mu_t)$, the forward SDE associated with the stochastic interpolant is
    \begin{align*}
        dX_t = \left( \mu + \frac{1}{2} (\dot{\Sigma}_t - \eta_t^2\, I_d)\, \Sigma_t^{-1} (X_t - \mu_t) \right) dt + \eta_t\, dB_t.
    \end{align*}
    Here, we can use Fubini's theorem to compute that for $0 \leq s < t \leq 1$,
    \begin{align*}
        \Cov(X_s, X_t)
         & = \E[(X_s - \mu_s) (X_t - \mu_t)^\top]                                                                                                                                                  \\
         & = \E\left[ (X_s - \mu_s) \left( (X_s - \mu_s) + \int_s^t \frac{1}{2} (\dot{\Sigma}_v - \eta_v^2\, I_d)\, \Sigma_v^{-1} (X_v - \mu_v)\, dv + \int_s^t \eta_v\, dB_v \right)^\top \right] \\
         & = \Sigma_s + \int_s^t \frac{1}{2} \Cov(X_s, X_v) ((\dot{\Sigma}_v - \eta_v^2\, I_d)\, \Sigma_v^{-1})^\top\, dv.
    \end{align*}
    Differentiating both sides with respect to $t$ yields the linear ODE
    \begin{align*}
        \partial_t \Cov(X_s, X_t) = \frac{1}{2} \Cov(X_s, X_t) ((\dot{\Sigma}_t - \eta_t^2\, I_d)\, \Sigma_t^{-1})^\top.
    \end{align*}
    Since all relevant terms are simultaneously diagonalizable, we can change into the eigenbasis of $\Sigma$ to obtain a scalar ODE for each eigenvalue. As a result, we obtain the closed-form solution
    \begin{equation} \label{eq:cov-ode}
        \Cov(X_s, X_t)
        = \Sigma_s^{1/2} \Sigma_t^{1/2} \exp\!\left( -\frac{1}{2} \int_s^t \eta_v^2\, \Sigma_v^{-1}\, dv \right)
    \end{equation}
    using the boundary condition $\Cov(X_s, X_s) = \Sigma_s$. Instantiating this formula yields the joint distribution of the Gaussian vector $(X_0, X_1)$ as $\Cov(X_0, X_1) = \Sigma^{1/2} \Psi^{1/2}$ where
    \begin{align*}
        \Psi \coloneq \exp\!\left( -\int_0^1 \eta_v^2\, \Sigma_v^{-1}\, dv \right).
    \end{align*}
    Therefore, we have
    \begin{align*}
        (X_1 \given X_0) \sim \Normal\left( \mu + \Sigma^{1/2} \Psi^{1/2} X_0,\, \Sigma (I_d - \Psi) \right)
    \end{align*}
    and
    \begin{align*}
        (X_0 \given X_1) \sim \Normal\left( \Sigma^{-1/2} \Psi^{1/2} (X_1 - \mu),\, I_d - \Psi \right).
    \end{align*}
    Using these expressions, we compute $h_0$ as a Gaussian integral:
    \begin{align*}
        h_0(x)
         & =\E[e^{-\lambda \norm{X_1 - a}_2^2} \given X_0 = x]                                                                                                                                          \\
         & \propto \exp\!\left( -\lambda \left( \mu + \Sigma^{1/2} \Psi^{1/2} x - a \right)^\top (I_d + 2 \lambda \Sigma (I_d - \Psi))^{-1} \left( \mu + \Sigma^{1/2} \Psi^{1/2} x - a \right) \right).
    \end{align*}
    Similarly, we can compute the initial value function bias by a Gaussian integral:
    \begin{align*}
         & \E\left[ \frac{1}{h_0(X_0)} \givenlarge X_1 = x \right]                                                                                                                                                                                   \\
         & \quad \propto \E\left[ \exp\!\left( \lambda \left( \mu + \Sigma^{1/2} \Psi^{1/2} X_0 - a \right)^\top (I_d + 2 \lambda \Sigma (I_d - \Psi))^{-1} \left( \mu + \Sigma^{1/2} \Psi^{1/2} X_0 - a \right) \right) \givenlarge X_1 = x \right] \\
         & \quad \propto \exp\!\left( \lambda (\mu - a + \Psi (x - \mu))^\top (I_d + 2 \lambda \Sigma (I_d - \Psi)^2)^{-1} (\mu - a + \Psi (x - \mu)) \right).
    \end{align*}
    At last, the distribution of $\tilde{x}_1$ under the guided ODE is
    \begin{align*}
        \tilde{\rho}_1(x)
        \propto \exp\!\left( -\frac{1}{2} (x - \mu)^\top \Sigma^{-1} (x - \mu) - \lambda \norm{x - a}_2^2 \right) \E\left[ \frac{1}{h_0(X_0)} \givenlarge X_1 = x \right],
    \end{align*}
    which is a Gaussian measure
    \begin{align*}
        \Normal\biggl( & \mu - 2 \lambda \Sigma (I_d - \Psi) (I_d + 2 \lambda \Sigma (I_d - \Psi))^{-1} (\mu - a),         \\
                       & \Sigma(I_d + 2 \lambda \Sigma (I_d - \Psi)^2) (I_d + 2 \lambda \Sigma (I_d - \Psi))^{-2} \biggr).
    \end{align*}
    Now, we know that $(X_t, X_1)$ is Gaussian (since it is a linear transformation of the Gaussian vector $(X_0, X_1)$), so we can compute $h_t$ in closed form. In particular, we have from~\eqref{eq:cov-ode} that $\Cov(X_t, X_1) = \Sigma_t^{1/2} \Sigma^{1/2} \Psi_t^{1/2}$ and thus the Gaussian conditioning formula yields
    \begin{align*}
        (X_1 \given X_t)
        \sim \Normal(\mu_{1 \vert t}(x), \Sigma_{1 \vert t})
        \coloneq \Normal\!\left( \mu + \Sigma^{1/2} \Sigma_t^{-1/2} \Psi_t^{1/2} (X_t - \mu_t),\, \Sigma (I_d - \Psi_t) \right).
    \end{align*}
    Now, we can compute $h_t$ in closed form as a Gaussian integral:
    \begin{align*}
        h_t(x)
        = \E[e^{-\lambda \norm{X_1 - a}_2^2} \given X_t = x]
        \propto \exp\!\left( -\lambda (\mu_{1 \vert t}(x) - a)^\top (I_d + 2 \lambda \Sigma_{1 \vert t})^{-1} (\mu_{1 \vert t}(x) - a) \right).
    \end{align*}
    This means that the guidance term is given by
    \begin{align*}
        \grad \log h_t(x) = -2 \lambda \Sigma^{1/2} \Sigma_t^{-1/2} \Psi_t^{1/2} (I_d + 2 \lambda \Sigma_{1 \vert t})^{-1} (\mu_{1 \vert t}(x) - a). \tag*{\qedhere}
    \end{align*}
\end{proof}

\section{Omitted proofs} \label{app:proofs}

\subsection{Proof of \texorpdfstring{\Cref{thm:doob-h-transform}}{Theorem 1}} \label{app:doob-proof}

\begin{restatable}[Doob $h$-transform]{theorem}{thmDoob} \label{thm:doob-h-transform}
    Suppose $X_t$ satisfies \eqref{eq:unguided-diffusion} and define the Doob $h$-function $h_t(x) \coloneq \E[e^{\lambda r(X_1)} \given X_t = x]$. Then, consider the guided ODE \eqref{eq:guided-ode}. If $X_0 \indep X_1$, the time-marginals of $\tilde{x}_t$ will be $\tilde{\rho}_t(x) \propto \rho_t(x)\, h_t(x)$; in particular, $\tilde{\rho}_1(x) \propto \rho_1(x)\, e^{\lambda r(x)}$ will be the reward-tilted measure.
\end{restatable}

\begin{proof}
    We first define the Doob $h$-function
    \begin{align*}
        h_t(x) = \E[e^{\lambda r(X_1)} \given X_t = x]
    \end{align*}
    and note that $h_t(X_t)$ is a Doob martingale. Letting $\mc{L}_t$ denote the generator of the diffusion~\eqref{eq:unguided-diffusion}:
    \begin{equation} \label{eq:unguided-generator}
        \mc{L}_t f = \left( b_t + \frac{1}{2} (\sigma_t \sigma_t^\top) \grad \log \rho_t \right)^\top \grad f + \frac{1}{2} \tr\!\left( (\sigma_t \sigma_t^\top) \grad^2 f \right),
    \end{equation}
    we know that $h_t(X_t)$ must satisfy the Kolmogorov backward equation
    \begin{equation} \label{eq:doob-kbe}
        \partial_t h_t(X_t) + \mc{L}_t h_t(X_t) = 0.
    \end{equation}
    Next, let $(\mc{F}_t)_{t \geq 0}$ denote the natural filtration generated by the Brownian motion $(B_t)_{t \geq 0}$. Letting $P$ denote the law of the path $(X_t)_{t \in [0, 1]}$ and letting $P_t$ denote the restriction of $P$ to $\mc{F}_t$, we define the \emph{Doob $h$-transform} of $P$ by its Radon--Nikodym derivative:
    \begin{align*}
        \frac{dQ_t}{dP_t} = g_t \coloneq \frac{h_t(X_t)}{h_0(X_0)}.
    \end{align*}
    By It\^o's formula and~\eqref{eq:doob-kbe}, we have
    \begin{align*}
        dg_t
         & = \frac{1}{h_0(X_0)} \left( (\partial_t h_t(X_t) + \mc{L}_t h_t(X_t))\, dt + \grad h_t(X_t)^\top \sigma_t\, dB_t \right) \\
         & = \frac{1}{h_0(X_0)} \grad h_t(X_t)^\top \sigma_t\, dB_t                                                                 \\
         & = g_t \grad \log h_t(X_t)^\top \sigma_t\, dB_t.
    \end{align*}
    At this point, applying It\^o's formula shows that
    \begin{align*}
        d(\log g_t) = \grad \log h_t(X_t)^\top \sigma_t\, dB_t - \frac{1}{2} \norm{\grad \log h_t(X_t)^\top \sigma_t}_2^2\, dt
    \end{align*}
    and integrating both sides shows that $g_t$ is the Dol\'eans--Dade exponential (local) martingale
    \begin{align*}
        g_t = \exp\!\left( \int_0^t \grad \log h_s(X_s)^\top \sigma_s\, dB_s - \frac{1}{2} \int_0^t \norm{\grad \log h_s(X_s)^\top \sigma_s}_2^2\, ds \right).
    \end{align*}
    Under the standard Novikov's condition (for instance):
    \begin{align*}
        \E\left[ \exp\!\left( \frac{1}{2} \int_0^1 \norm{\grad \log h_s(X_s)^\top \sigma_s}_2^2\, ds \right) \right] < \infty,
    \end{align*}
    $g_t$ is a true martingale. At last, Girsanov's theorem implies that under $Q \coloneq Q_1$, the process
    \begin{align*}
        \tilde{B}_t = B_t - \int_0^t \sigma_s^\top \grad \log h_s(X_s)\, ds
    \end{align*}
    is a Brownian motion, and substituting back into the SDE for $X_t$ gives
    \begin{align*}
        dX_t = \left( b_t(X_t) + \frac{1}{2} (\sigma_t \sigma_t^\top) \grad \log \rho_t(X_t) + (\sigma_t \sigma_t^\top) \grad \log h_t(X_t) \right) dt + \sigma_t\, d\tilde{B}_t.
    \end{align*}
    We can then find the density $\tilde{\rho}_t$ of $X_t$ under $Q$ by integrating a test function $f \in C_c^\infty(\R^d)$:
    \begin{align*}
        \E_Q[f(X_t)]
        = \E_P[g_t\, f(X_t)]
        = \E_P\left[ f(X_t)\, \frac{h_t(X_t)}{h_0(X_0)} \right].
    \end{align*}
    In the specific case that $X_0 \indep X_1$ (which happens using the memoryless noise schedule from \Cref{thm:memoryless-schedule}), it is clear that $h_0(X_0) = \E[e^{\lambda r(X_1)}]$ is a normalizing constant, and thus $\tilde{\rho}_t(x) \propto \rho_t(x)\, h_t(x)$. At the endpoint, this means that $\tilde{\rho}_1(x) \propto e^{\lambda r(x)} \rho_1(x)$ as desired, and the Doob $h$-transform of the original diffusion gives a principled way to sample from the reward-tilted measure. We can then convert the guided SDE back to a guided ODE by considering the associated probability flow:
    \begin{align*}
        dX_t
         & = \left( b_t(X_t) + \frac{1}{2} (\sigma_t \sigma_t^\top) \grad \log \rho_t(X_t) + (\sigma_t \sigma_t^\top) \grad \log h_t(X_t) - \frac{1}{2} (\sigma_t \sigma_t^\top) \grad \log \tilde{\rho}_t(X_t) \right) dt \\
         & = \left( b_t(X_t) + \frac{1}{2} (\sigma_t \sigma_t^\top) \grad \log h_t(X_t) \right) dt.
    \end{align*}
    Hence, the guided probability flow ODE simply includes an additional score term $\frac{1}{2} (\sigma_t \sigma_t^\top) \grad \log h_t(X_t)$ that steers the dynamics toward high-reward regions of the output space.
\end{proof}

\subsection{Proof of \texorpdfstring{\Cref{thm:memoryless-schedule}}{Theorem 2}} \label{app:memoryless-proof}

\begin{restatable}[Memoryless noise schedule]{theorem}{thmMemoryless} \label{thm:memoryless-schedule}
    Consider the forward SDE \eqref{eq:unguided-diffusion} for the linear interpolant \eqref{eq:linear-interpolant}. Then, choosing the memoryless noise schedule $\sigma_t = \sqrt{2 (1 - t) / t}\, I_d$ ensures that $(X_t \given X_1) \overset{d}{=} (I_t \given I_1)$.
\end{restatable}

\begin{proof}
    Note that $(I_t \given I_1) \sim \Normal(t I_1,\, (1 - t)^2 I_d)$, so by Tweedie's formula, we have
    \begin{align*}
        \E[I_1 \given I_t = x] = \frac{x}{t} + \frac{(1 - t)^2}{t} \grad \log \rho_t(x)
    \end{align*}
    and
    \begin{align*}
        \E[X_0 \given I_t = x]
        = \E\left[ \frac{I_t - t I_1}{1 - t} \givenlarge I_t = x \right]
        = -(1 - t) \grad \log \rho_t(x).
    \end{align*}
    As a result, we deduce that
    \begin{align*}
        b_t(x) = \E[\dot{I}_t \given I_t = x]
        = \E[I_1 - I_0 \given I_t = x]
        = \frac{x}{t} + \frac{1 - t}{t} \grad \log \rho_t(x).
    \end{align*}
    Parameterizing $\sigma_t = \eta_t I_d$, the forward SDE for the stochastic interpolant with the memoryless noise schedule is
    \begin{align*}
        dX_t = \left( \frac{X_t}{t} + \frac{1 - t}{t} \grad \log \rho_t(X_t) + \frac{\eta_t^2}{2} \grad \log \rho_t(X_t) \right) dt + \eta_t\, dB_t.
    \end{align*}
    The Anderson time-reversal formula then yields the backward SDE
    \begin{align*}
        dX_t^\leftarrow = \left( -\frac{X_t^\leftarrow}{t} - \left( \frac{1 - t}{t} - \frac{\eta_t^2}{2} \right) \grad \log \rho_t(X_t^\leftarrow) \right) dt + \eta_t\, dB_t^\leftarrow,
    \end{align*}
    and setting $\eta_t = \sqrt{2 (1 - t) / t}$ cancels the score term in the drift of the backward SDE so that
    \begin{align*}
        dX_t^\leftarrow = -\frac{X_t^\leftarrow}{t}\, dt + \sqrt{\frac{2 (1 - t)}{t}}\, dB_t^\leftarrow.
    \end{align*}
    This shows that $(X_t \given X_1)$ is Gaussian (because the diffusion coefficient is deterministic) with mean $t X_1$ and covariance $(1 - t)^2 I_d$ (for example, by the It\^o isometry), matching the distribution of $(I_t \given I_1)$.
\end{proof}

\subsection{Convergence to plug-in flow} \label{app:plugin-flow-convergence}

We state and prove the formal version of the convergence result discussed in \Cref{sec:plugin-bias}.

\begin{proposition}[Convergence to plug-in flow] \label{prop:plugin-flow}
    Let $\tilde{x}_t$ be the solution to the exact continuous-time ODE $d\tilde{x}_t = b_t^{(k)}(\tilde{x}_t)\, dt$ for $t \in [0, 1]$ with initial condition $\tilde{x}_0 = x_0$. Let $(\bar{x}_t)_{t \in [0, 1]}$ denote the continuous-time Euler approximation with step size $h = 1/N$, defined by $\bar{x}_0 = x_0$ and
    \begin{equation*}
        d\bar{x}_t = \left( b_{\eta(t)}(\bar{x}_{\eta(t)}) + \frac{1}{2} (\sigma_{\eta(t)} \sigma_{\eta(t)}^\top) \grad \log \hat{h}_{\eta(t)}^{(k)}(\bar{x}_{\eta(t)}) \right) dt,
    \end{equation*}
    where $\eta(t) \coloneq \floor{t/h} h$ is the time of the most recent grid point. We make the following assumptions:
    \begin{enumerate}
        \item[(i)] $\norm{b_s^{(k)}(x) - b_t^{(k)}(y)}_2 \leq L (\norm{x - y}_2 + \abs{s - t})$ for some $L > 0$ and for all $s, t \in [0, 1]$ and $x, y \in \R^d$.
        \item[(ii)] Defining the guidance residual $\xi_t(x) \coloneq \grad \log \hat{h}^{(k)}_t(x) - u_t^{(k)}(x)$, we have $\E\left[ \norm*{\frac{1}{2} (\sigma_t \sigma_t^\top)\, \xi_t(x)}_2^2 \right] \leq V_k$ for some $V_k > 0$ and for all $x \in \R^d$.
        \item[(iii)] $\norm{b_t^{(k)}(\bar{x}_t)}_2 \leq B$ is uniformly bounded by $B > 0$.
    \end{enumerate}
    Then, the Euler trajectory converges uniformly over $[0, 1]$ in probability to the ODE dynamics as $N \to \infty$:
    \begin{equation*}
        \sup_{t \in [0, 1]}\; \norm{\bar{x}_t - \tilde{x}_t}_2 \overset{p}{\to} 0.
    \end{equation*}
\end{proposition}

The proof combines Gr\"onwall's inequality with Doob's $L^2$ maximal inequality for the martingale residual.

\begin{proof}
    First, substituting $b_t^{(k)}(x) = b_t(x) + \frac{1}{2}(\sigma_t \sigma_t^\top) u_t^{(k)}(x)$ and $\grad \log \hat{h}_t^{(k)}(x) = u_t^{(k)}(x) + \xi_t(x)$ shows that
    \begin{align*}
        \bar{x}_t
         & = x_0 + \int_0^t \left( b_{\eta(s)}(\bar{x}_{\eta(s)}) + \frac{1}{2} (\sigma_{\eta(s)} \sigma_{\eta(s)}^\top) \left( u_{\eta(s)}^{(k)}(\bar{x}_{\eta(s)}) + \xi_{\eta(s)}(\bar{x}_{\eta(s)}) \right) \right) ds \\
         & = x_0 + \int_0^t \left( b_{\eta(s)}^{(k)}(\bar{x}_{\eta(s)}) + \frac{1}{2} (\sigma_{\eta(s)} \sigma_{\eta(s)}^\top) \xi_{\eta(s)}(\bar{x}_{\eta(s)}) \right) ds.
    \end{align*}
    Thus, we have
    \begin{align*}
        \norm{\bar{x}_t - \tilde{x}_t}_2
         & = \norm*{\int_0^t \left( b_{\eta(s)}^{(k)}(\bar{x}_{\eta(s)}) - b_s^{(k)}(\tilde{X}_s) \right) ds + M_t}_2                                                                                                \\
         & \leq \norm*{\int_0^t \left( b_s^{(k)}(\bar{x}_s) - b_s^{(k)}(\tilde{X}_s) \right) ds}_2 + \norm*{\int_0^t \left( b_{\eta(s)}^{(k)}(\bar{x}_{\eta(s)}) - b_s^{(k)}(\bar{x}_s) \right) ds}_2 + \norm{M_t}_2 \\
         & \leq L \int_0^t \norm{\bar{x}_s - \tilde{X}_s}_2\, ds + \norm*{\int_0^t \left( b_{\eta(s)}^{(k)}(\bar{x}_{\eta(s)}) - b_s^{(k)}(\bar{x}_s) \right) ds}_2 + \norm{M_t}_2,
    \end{align*}
    where $M_t \coloneq \int_0^t \frac{1}{2} (\sigma_{\eta(s)} \sigma_{\eta(s)}^\top)\, \xi_{\eta(s)}(\bar{x}_{\eta(s)})\, ds$. By Gr\"onwall's inequality, this implies the bound
    \begin{align*}
        \sup_{t \in [0, 1]}\; \norm{\bar{x}_t - \tilde{x}_t}_2 \leq \left( \sup_{t \in [0, 1]}\, \norm*{\int_0^t \left( b_{\eta(s)}^{(k)}(\bar{x}_{\eta(s)}) - b_s^{(k)}(\bar{x}_s) \right) ds}_2 + \sup_{t \in [0, 1]} \norm{M_t}_2 \right) e^L.
    \end{align*}
    At this point, we have
    \begin{align*}
         & \sup_{t \in [0, 1]}\, \norm*{\int_0^t \left( b_{\eta(s)}^{(k)}(\bar{x}_{\eta(s)}) - b_s^{(k)}(\bar{x}_s) \right) ds}_2                                                                                  \\
         & \quad \leq L \left( \int_0^1 \norm{\bar{x}_{\eta(s)} - \bar{x}_s}_2\, ds + \int_0^1 \abs{\eta(s) - s}\, ds \right)                                                                                      \\
         & \quad \leq L \int_0^1 (s - \eta(s)) \norm*{ b_{\eta(s)}^{(k)}(\bar{x}_{\eta(s)}) + \frac{1}{2} (\sigma_{\eta(s)} \sigma_{\eta(s)}^\top)\, \xi_{\eta(s)}(\bar{x}_{\eta(s)}) }_2\, ds + L \int_0^1 h\, ds \\
         & \quad \leq Lh \int_0^1 \left( B + \norm*{ \frac{1}{2} (\sigma_{\eta(s)} \sigma_{\eta(s)}^\top)\, \xi_{\eta(s)}(\bar{x}_{\eta(s)}) }_2 \right) ds + Lh.
    \end{align*}
    Taking expectations on both sides and using Jensen's inequality, we obtain
    \begin{align*}
         & \E\left[ \sup_{t \in [0, 1]}\, \norm*{\int_0^t \left( b_{\eta(s)}^{(k)}(\bar{x}_{\eta(s)}) - b_s^{(k)}(\bar{x}_s) \right) ds}_2 \right]                                                                       \\
         & \quad \leq Lh \int_0^1 \left( B + \sqrt{\E\left[ \norm*{ \frac{1}{2} (\sigma_{\eta(s)} \sigma_{\eta(s)}^\top)\, \xi_{\eta(s)}(\bar{x}_{\eta(s)}) }_2^2 \givenlarge \bar{x}_{\eta(s)} \right]} \right) ds + Lh \\
         & \quad \leq Lh (B + \sqrt{V_k} + 1).
    \end{align*}
    Markov's inequality now implies that as $h \downarrow 0$,
    \begin{align*}
        \sup_{t \in [0, 1]}\, \norm*{\int_0^t \left( b_{\eta(s)}^{(k)}(\bar{x}_{\eta(s)}) - b_s^{(k)}(\bar{x}_s) \right) ds}_2 \overset{p}{\to} 0.
    \end{align*}
    Lastly, it remains to show that $\sup_{t \in [0, 1]}\, \norm{M_t}_2 \overset{p}{\to} 0$ as $h \downarrow 0$. Defining $t_n = nh$ for $0 \leq n \leq N$, we know that the supremum of $M_t$ is achieved at some $t_n$ since $M_t$ is piecewise linear. Furthermore, $(M_{t_n})_{n=0}^N$ is a discrete-time martingale with respect to the filtration $(\mc{F}_{t_n})_{n=0}^N$ generated by the independent samples drawn in the plug-in estimator, so by Doob's $L^2$ maximal inequality for discrete-time martingales:
    \begin{align*}
        \E\left[ \max_{0 \leq n \leq N}\, \norm{M_{t_n}}_2^2 \right] \leq 4\, \E[\norm{M_1}_2^2].
    \end{align*}
    Since the martingale increments are orthogonal in $L^2(\P)$, we have
    \begin{align*}
        \E[\norm{M_1}_2^2]
        = \sum_{n=0}^{N-1} h^2\, \E\left[ \norm*{\frac{1}{2} (\sigma_{t_n} \sigma_{t_n}^\top)\, \xi_{t_n}(\bar{x}_{t_n})}_2^2 \right]
        \leq V_k \sum_{n=0}^{N-1} h^2
        = h V_k.
    \end{align*}
    Thus, we have $\max_{0 \leq n \leq N}\, \norm{M_{t_n}}_2 \to 0$ in $L^2(\P)$ and therefore in probability by Chebyshev's inequality, completing the proof.
\end{proof}

\subsection{Plug-in flow for Gaussian target (non-memoryless)} \label{app:plugin-flow-gaussian-nonmemoryless}

The general non-memoryless analogue of \Cref{thm:plugin-flow-gaussian} is the following.

\begin{theorem}[Plug-in flow for Gaussian target (non-memoryless)] \label{thm:plugin-flow-gaussian-nonmemoryless}
    Consider the plug-in flow \eqref{eq:plugin-flow} with $k = 1$, and let $\Psi \coloneq \Psi_0$ as in \Cref{thm:ivf-bias-gaussian}. Then, the guidance term is given by
    \begin{align*}
        u_t^{(1)}(x) = -2 \lambda \Sigma^{1/2} \Sigma_t^{-1/2} \Psi_t^{1/2} (\mu_{1 \vert t}(x) - a)
    \end{align*}
    and the terminal distribution of the plug-in flow is $\tilde{x}_1 \sim \Normal(\mu^{(1)},\, \Sigma^{(1)})$, where
    \begin{align*}
        \mu^{(1)} = \mu - \sqrt{\pi}\, (\lambda \Sigma)^{1/2} \exp(-\lambda \Sigma) \left( \erfi\!\left((\lambda \Sigma)^{1/2}\right) - \erfi\!\left((\lambda \Sigma)^{1/2} \Psi^{1/2}\right) \right) (\mu - a)
    \end{align*}
    and
    \begin{align*}
        \Sigma^{(1)} = \Sigma \exp(-2 \lambda \Sigma (I_d - \Psi)).
    \end{align*}
\end{theorem}

\begin{proof}
    We begin by computing the guidance term $u_t^{(1)}(x)$ using the reparameterization trick
    \begin{align*}
        X_1^{(1)}(x) = \mu_{1 \vert t}(x) + \Sigma_{1 \vert t}^{1/2} \epsilon
    \end{align*}
    where $\epsilon \sim \Normal(0, I_d)$ is independent of everything else. In particular, since $\mu_{1 \vert t}(x)$ is affine in $x$, we have $\nabla X_1^{(1)}(x) = \Sigma^{1/2} \Sigma_t^{-1/2} \Psi_t^{1/2}$. Thus, we apply the chain rule:
    \begin{align*}
        \grad \log \hat{h}^{(1)}_t(x)
        = \grad \left( -\lambda \norm{X_1^{(1)}(x) - a}_2^2 \right)
        = -2 \lambda \Sigma^{1/2} \Sigma_t^{-1/2} \Psi_t^{1/2} (X_1^{(1)}(x) - a),
    \end{align*}
    and taking conditional expectations on both sides yields the guidance term
    \begin{align*}
        u_t^{(1)}(x)
        = \E[\grad \log \hat{h}^{(1)}_t(X_t) \given X_t = x]
        = -2 \lambda \Sigma^{1/2} \Sigma_t^{-1/2} \Psi_t^{1/2} (\mu_{1 \vert t}(x) - a).
    \end{align*}
    Substituting this back into the plug-in flow ODE, the dynamics are given by
    \begin{align*}
        d\tilde{x}_t = \left( \mu + \frac{1}{2} \dot{\Sigma}_t \Sigma_t^{-1} (\tilde{x}_t - \mu_t) - \lambda \eta_t^2\, \Sigma^{1/2} \Sigma_t^{-1/2} \Psi_t^{1/2} (\mu_{1 \vert t}(\tilde{x}_t) - a) \right) dt.
    \end{align*}
    Because the initial condition $\tilde{x}_0 \sim \Normal(0, I_d)$ is Gaussian and the drift is an affine function of $\tilde{x}_t$, the process $\tilde{x}_t$ remains exactly Gaussian for all $t \in [0, 1]$, and we track its mean $\mu^{(1)}_t = \E[\tilde{x}_t]$ and covariance $\Sigma^{(1)}_t = \Cov(\tilde{x}_t)$. Taking the expectation of the drift yields an ODE for the mean:
    \begin{align*}
        d\mu^{(1)}_t
        = \mu + \left( \frac{1}{2} \dot{\Sigma}_t \Sigma_t^{-1} - \lambda \eta_t^2\, \Sigma \Sigma_t^{-1} \Psi_t \right) (\mu^{(1)}_t - \mu_t) - \lambda \eta_t^2\, \Sigma^{1/2} \Sigma_t^{-1/2} \Psi_t^{1/2} (\mu - a).
    \end{align*}
    We then multiply by the integrating factor $\Sigma_t^{-1/2} \exp(\lambda \Sigma \Psi_t)$ to get
    \begin{align*}
        \frac{d}{dt} \left( \Sigma_t^{-1/2} \exp(\lambda \Sigma \Psi_t) (\mu^{(1)}_t - \mu_t) \right)
        = -2 \lambda \Sigma^{1/2} \exp(\lambda \Sigma \Psi_t)\, \frac{d}{dt} (\Psi_t^{1/2}) (\mu - a).
    \end{align*}
    Integrating both sides from 0 to 1 then gives us
    \begin{align*}
        \Sigma^{-1/2} \exp(\lambda \Sigma) (\mu^{(1)} - \mu)
        = -2 \lambda \Sigma^{1/2} \int_0^1 \exp(\lambda \Sigma \Psi_t)\, \frac{d}{dt} (\Psi_t^{1/2})\, dt\, (\mu - a).
    \end{align*}
    Making the substitution $M_t = (\lambda \Sigma)^{1/2} \Psi_t^{1/2}$ and using the fundamental theorem of calculus gives us a formula for the terminal mean:
    \begin{align*}
        \mu^{(1)} = \mu - \sqrt{\pi}\, (\lambda \Sigma)^{1/2} \exp(-\lambda \Sigma) \left( \erfi\!\left((\lambda \Sigma)^{1/2}\right) - \erfi\!\left((\lambda \Sigma)^{1/2} \Psi^{1/2}\right) \right) (\mu - a).
    \end{align*}
    Next, we have the linear ODE
    \begin{align*}
        d(\tilde{x}_t - \mu^{(1)}_t)
        = \left( \frac{1}{2} \dot{\Sigma}_t \Sigma_t^{-1} - \lambda \eta_t^2 \Sigma \Sigma_t^{-1} \Psi_t \right) (\tilde{x}_t - \mu^{(1)}_t)\, dt,
    \end{align*}
    so by the chain rule, the covariance evolves according to the equation
    \begin{align*}
        d\Sigma^{(1)}_t
        = (\dot{\Sigma}_t \Sigma_t^{-1} - 2 \lambda \eta_t^2 \Sigma \Sigma_t^{-1} \Psi_t)\, \Sigma^{(1)}_t\, dt
        = (\dot{\Sigma}_t \Sigma_t^{-1} - 2 \lambda \Sigma \dot{\Psi}_t)\, \Sigma^{(1)}_t\, dt.
    \end{align*}
    Because all relevant matrices commute, we can solve this ODE in closed form at $t = 1$:
    \begin{align*}
        \Sigma^{(1)}
        = \exp\!\left( \int_0^1 (\dot{\Sigma}_v \Sigma_v^{-1} - 2 \lambda \Sigma \dot{\Psi}_v)\, dv \right)\, \Sigma^{(1)}_0
        = \Sigma \exp(-2 \lambda \Sigma (I_d - \Psi)). \tag*{\qedhere}
    \end{align*}
\end{proof}

\subsection{Proof of \texorpdfstring{\Cref{thm:plugin-flow-gaussian}}{Theorem 4 (memoryless plug-in flow)}} \label{app:plugin-flow-gaussian-memoryless-proof}

\thmPluginGaussian*

\begin{proof}
    Under the memoryless schedule, $\Psi = 0$ and $\Sigma^{1/2} \Sigma_t^{-1/2} \Psi_t^{1/2} = t \Sigma \Sigma_t^{-1}$, so substituting into the formulas of \Cref{thm:plugin-flow-gaussian-nonmemoryless} and using $\erfi(0) = 0$ yields the stated guidance term and terminal mean and covariance.
\end{proof}

\subsection{Plug-in guidance is aggressive in the tails} \label{app:plugin-flow-gaussian-multi}

The following uniform bound on the $k$-particle plug-in drift underlies the $\infty$-Wasserstein bound of \Cref{thm:wasserstein-bound}.

\begin{lemma}[Plug-in guidance is aggressive in the tails] \label{lem:plugin-flow-gaussian-multi}
    In the setting of \Cref{thm:plugin-flow-gaussian}, consider the plug-in flow \eqref{eq:plugin-flow} with $k$ particles, where $u_t^{(k)}(x) = \E[\grad \log \hat{h}^{(k)}_t(X_t) \given X_t = x]$, $\hat{h}^{(k)}_t(x) = \frac{1}{k} \sum_{i=1}^k e^{\lambda r(X_1^{(i)}(x))}$, and $X_1^{(1)}(x), \ldots, X_1^{(k)}(x) \sim \mathrm{law}(X_1 \given X_t = x)$ are drawn independently and independent from everything else. Then, we have the uniform bounds
    \begin{align*}
        \sup_{x \in \R^d}\; \norm{u_t^{(k)}(x) - u_t^{(1)}(x)}_2
        \leq M_t(k)
        \coloneq 2 \lambda\, \norm*{\Sigma^{1/2} \Sigma_t^{-1/2} \Psi_t^{1/2} \Sigma_{1 \vert t}^{1/2}}_2 \sqrt{2 \log k}
    \end{align*}
    and
    \begin{align*}
         & 4 \lambda^2\, \lambda_{\min}\!\left( \Sigma^{1/2} \Sigma_t^{-1/2} \Psi_t^{1/2} \Sigma_{1 \vert t} (I_d + 2 \lambda \Sigma_{1 \vert t})^{-1} \right) \norm{\mu_{1 \vert t}(x) - a}_2 - M_t(k)            \\
         & \quad \leq \norm{u_t^{(k)}(x) - \grad \log h_t(x)}_2                                                                                                                                                    \\
         & \quad \leq 4 \lambda^2\, \lambda_{\max}\!\left( \Sigma^{1/2} \Sigma_t^{-1/2} \Psi_t^{1/2} \Sigma_{1 \vert t} (I_d + 2 \lambda \Sigma_{1 \vert t})^{-1} \right) \norm{\mu_{1 \vert t}(x) - a}_2 + M_t(k)
    \end{align*}
    for all $x \in \R^d$.
\end{lemma}

The first inequality shows that the $k$-particle plug-in flow is well-approximated by the $k = 1$ dynamics in the tails. The second inequality shows that the error of the plug-in drift grows linearly in $\norm{x}_2$, confirming that the plug-in estimator is overly aggressive far from the reward center.

\begin{proof}
    We use the reparameterization trick to write
    \begin{align*}
        X_1^{(i)}(x) = \mu_{1 \vert t}(x) + \Sigma_{1 \vert t}^{1/2} \epsilon^{(i)}
    \end{align*}
    for i.i.d.\ $\epsilon^{(1)}, \ldots, \epsilon^{(k)} \sim \Normal(0, I_d)$ independent of everything else. Recall from the proof of \Cref{thm:plugin-flow-gaussian-nonmemoryless} that $\nabla X_1^{(i)}(x) = \Sigma^{1/2} \Sigma_t^{-1/2} \Psi_t^{1/2}$, so by the chain rule, we have
    \begin{align*}
        \grad \log \hat{h}^{(k)}_t(x)
         & = \grad \log \left( \frac{1}{k} \sum_{i=1}^k e^{-\lambda \norm{X_1^{(i)}(x) - a}_2^2} \right)                                                                                                                                  \\
         & = -2 \lambda\, \Sigma^{1/2} \Sigma_t^{-1/2} \Psi_t^{1/2} \frac{\sum_{i=1}^k e^{-\lambda \norm{X_1^{(i)}(x) - a}_2^2} (X_1^{(i)}(x) - a)}{\sum_{i=1}^k e^{-\lambda \norm{X_1^{(i)}(x) - a}_2^2}}                                \\
         & = -2 \lambda\, \Sigma^{1/2} \Sigma_t^{-1/2} \Psi_t^{1/2} (\mu_{1 \vert t}(x) - a)                                                                                                                                              \\
         & \quad - 2 \lambda\, \Sigma^{1/2} \Sigma_t^{-1/2} \Psi_t^{1/2} \Sigma_{1 \vert t}^{1/2}\, \frac{\sum_{i=1}^k e^{-\lambda \norm{X_1^{(i)}(x) - a}_2^2}\, \epsilon^{(i)}}{\sum_{i=1}^k e^{-\lambda \norm{X_1^{(i)}(x) - a}_2^2}}.
    \end{align*}
    Taking a conditional expectation and subtracting $u_t^{(1)}(x)$ from both sides, we have
    \begin{equation} \label{eq:plugin-multi-bias}
        u_t^{(k)}(x) - u_t^{(1)}(x)
        = -2 \lambda\, \Sigma^{1/2} \Sigma_t^{-1/2} \Psi_t^{1/2} \Sigma_{1 \vert t}^{1/2}\, \E\!\left[\frac{\sum_{i=1}^k e^{-\lambda \norm{X_1^{(i)}(x) - a}_2^2}\, \epsilon^{(i)}}{\sum_{i=1}^k e^{-\lambda \norm{X_1^{(i)}(x) - a}_2^2}}\right].
    \end{equation}
    For any unit vector $v \in \R^d$, the projections $\langle v, \epsilon^{(i)}\rangle$ are i.i.d.\ $\Normal(0, 1)$, and since the softmax weights are nonnegative and sum to one,
    \begin{align*}
        \norm*{\E\!\left[\frac{\sum_{i=1}^k e^{-\lambda \norm{X_1^{(i)}(x) - a}_2^2}\, \epsilon^{(i)}}{\sum_{i=1}^k e^{-\lambda \norm{X_1^{(i)}(x) - a}_2^2}}\right]}_2
         & = \sup_{\norm{v}_2 = 1}\; \E\!\left[\left\langle v,\, \frac{\sum_{i=1}^k e^{-\lambda \norm{X_1^{(i)}(x) - a}_2^2}\, \epsilon^{(i)}}{\sum_{i=1}^k e^{-\lambda \norm{X_1^{(i)}(x) - a}_2^2}}\right\rangle\right] \\
         & \leq \sup_{\norm{v}_2 = 1}\; \E\!\left[\max_{1 \leq i \leq k}\, \langle v, \epsilon^{(i)}\rangle\right]                                                                                                        \\
         & \leq \sqrt{2 \log k}
    \end{align*}
    by the standard sub-Gaussian maximal inequality (e.g., \cite[Theorem 2.5]{boucheron2013concentration}). Applying a spectral-norm bound to the matrix factor gives the first bound. For the second inequality, note that
    \begin{equation} \label{eq:plugin-multi-triangle}
        \begin{split}
             & u_t^{(k)}(x) - \grad \log h_t(x)                                                                                                                                                                       \\
             & \quad = (u_t^{(1)}(x) - \grad \log h_t(x)) + (u_t^{(k)}(x) - u_t^{(1)}(x))                                                                                                                             \\
             & \quad = \left( -4 \lambda^2\, \Sigma^{1/2} \Sigma_t^{-1/2} \Psi_t^{1/2} \Sigma_{1 \vert t} (I_d + 2 \lambda \Sigma_{1 \vert t})^{-1} (\mu_{1 \vert t}(x) - a) \right) + (u_t^{(k)}(x) - u_t^{(1)}(x)),
        \end{split}
    \end{equation}
    so by the reverse triangle inequality and the first part of the theorem, we have
    \begin{align*}
         & \norm{u_t^{(k)}(x) - \grad \log h_t(x)}_2                                                                                                                                                                \\
         & \quad \geq \norm*{4 \lambda^2\, \Sigma^{1/2} \Sigma_t^{-1/2} \Psi_t^{1/2} \Sigma_{1 \vert t} (I_d + 2 \lambda \Sigma_{1 \vert t})^{-1} (\mu_{1 \vert t}(x) - a)}_2                                       \\
         & \quad \quad - \norm{u_t^{(k)}(x) - u_t^{(1)}(x)}_2                                                                                                                                                       \\
         & \quad \geq 4 \lambda^2\, \lambda_{\min}\!\left( \Sigma^{1/2} \Sigma_t^{-1/2} \Psi_t^{1/2} \Sigma_{1 \vert t} (I_d + 2 \lambda \Sigma_{1 \vert t})^{-1} \right) \norm{\mu_{1 \vert t}(x) - a}_2 - M_t(k).
    \end{align*}
    Similarly, the upper bound follows by applying the triangle inequality to~\eqref{eq:plugin-multi-triangle} and using the first part of the theorem again.
\end{proof}

\subsection{Proof of \texorpdfstring{\Cref{thm:wasserstein-bound}}{the infinity-Wasserstein bound}} \label{app:wasserstein-proof}

\thmWasserstein*

\begin{proof}
    Let $X_t^{(k)}$ denote the solution to the plug-in flow ODE with $k$ particles and note that
    \begin{align*}
        X_t^{(k)} - X_t^{(1)}
        = \int_0^t \left( b_s^{(k)}(X_s^{(k)}) - b_s^{(1)}(X_s^{(1)}) \right) ds,
    \end{align*}
    which implies that
    \begin{align*}
        \norm{X_t^{(k)} - X_t^{(1)}}_2
         & \leq \int_0^t \norm*{b_s^{(k)}(X_s^{(k)}) - b_s^{(1)}(X_s^{(1)})}_2\, ds                                                                       \\
         & \leq \int_0^t \left( \norm*{b_s^{(k)}(X_s^{(k)}) - b_s^{(1)}(X_s^{(k)})}_2 + \norm*{b_s^{(1)}(X_s^{(k)}) - b_s^{(1)}(X_s^{(1)})}_2 \right) ds.
    \end{align*}
    The first term in the integrand is bounded above by $\frac{1}{2} \eta_s^2\, M_s(k)$ by \Cref{lem:plugin-flow-gaussian-multi}. For the second term, note that $b_s^{(1)}$ is affine, so it is Lipschitz with constant
    \begin{align*}
        \grad b_s^{(1)}(x) = \frac{1}{2} \dot{\Sigma}_s \Sigma_s^{-1} - \lambda \eta_s^2\, \Sigma \Sigma_s^{-1} \Psi_s.
    \end{align*}
    Combining these observations gives us
    \begin{align*}
        \norm{X_t^{(k)} - X_t^{(1)}}_2
        \leq \int_0^1 \left( \frac{1}{2} \eta_s^2\, M_s(k) + \norm*{\frac{1}{2} \dot{\Sigma}_s \Sigma_s^{-1} - \lambda \eta_s^2\, \Sigma \Sigma_s^{-1} \Psi_s}_2\, \norm{X_s^{(k)} - X_s^{(1)}}_2 \right) ds.
    \end{align*}
    Applying Gr\"onwall's inequality then gives us
    \begin{align*}
         & \norm{X_1^{(k)} - X_1^{(1)}}_2                                                                                                                                                                                                                                              \\
         & \quad \leq \int_0^1 \frac{1}{2} \eta_s^2\, M_s(k) \exp\!\left( \int_s^1 \norm*{\frac{1}{2} \dot{\Sigma}_v \Sigma_v^{-1} - \lambda \eta_v^2\, \Sigma \Sigma_v^{-1} \Psi_v}_2\, dv \right) ds                                                                                 \\
         & \quad = \lambda \sqrt{2 \log k} \int_0^1 \eta_s^2\, \norm*{\Sigma^{1/2} \Sigma_s^{-1/2} \Psi_s^{1/2} \Sigma_{1 \vert s}^{1/2}}_2 \exp\!\left( \int_s^1 \norm*{\frac{1}{2} \dot{\Sigma}_v \Sigma_v^{-1} - \lambda \eta_v^2\, \Sigma \Sigma_v^{-1} \Psi_v}_2\, dv \right) ds.
    \end{align*}
    Lastly, we can choose the coupling $\pi = \mathrm{law}(X_1^{(k)}, X_1^{(1)})$ to conclude the desired bound.
\end{proof}

\subsection{Proof of \texorpdfstring{\Cref{thm:analytic-guidance-gmm}}{the analytic GMM guidance theorem}} \label{app:ivf-bias-gmm-proof}

\begin{restatable}[Analytic guidance for Gaussian mixture target]{theorem}{thmAnalyticGMM} \label{thm:analytic-guidance-gmm}
    The overall guidance term is given by
    \begin{align*}
        \grad \log h_t(x) = \sum_{i=1}^m w_{it}(x)\, (\grad \log \rho_{it}(x) + \grad \log h_{it}(x)) - \grad \log \rho_t(x)
    \end{align*}
    where
    \begin{align*}
        w_{it}(x) \coloneq \frac{\pi_i\, \rho_{it}(x)\, h_{it}(x)}{\sum_{j=1}^m \pi_j\, \rho_{jt}(x)\, h_{jt}(x)}.
    \end{align*}
\end{restatable}

\begin{proof}
    Introducing a latent mixture index $Z \in [m]$ and applying the law of total expectation:
    \begin{align*}
        h_t(x)
        = \sum_{i=1}^m \P(Z = i \given X_t = x)\, \E[e^{\lambda r(X_1)} \given X_t = x, Z = i]
        = \frac{1}{\rho_t(x)} \sum_{i=1}^m \pi_i\, \rho_{it}(x)\, h_{it}(x).
    \end{align*}
    Taking $\grad \log$ and applying the product rule yields the result.
\end{proof}

\subsection{Proof of \texorpdfstring{\Cref{thm:plugin-flow-gmm}}{Theorem (plug-in flow for Gaussian mixture)}} \label{app:plugin-flow-gmm-proof}

\thmPluginGMM*

\begin{proof}
    The guidance formula~\eqref{eq:plugin-gmm-guidance} follows directly from the reparameterization trick and the chain rule. To compute the Jacobian $\grad \bar{X}_1(x)$, we differentiate the GLASS ODE with respect to $x$. Using
    \begin{align*}
        \grad X^*_s(x) = \frac{1}{\eta^*_s (1 + \eta^*_s)} \left( \frac{t}{(1 - t)^2} I_d + \frac{s}{(1 - s)^2} \grad \bar{X}_s(x) \right),
    \end{align*}
    differentiating under the integral sign yields the variational equation
    \begin{align*}
        d(\grad \bar{X}_s(x))
         & = \frac{1}{1 - s} \left( \nabla D_{\tau_s^*}(X^*_s(x)) \left( \frac{1}{\eta^*_s (1 + \eta^*_s)} \left( \frac{t}{(1 - t)^2} I_d + \frac{s}{(1 - s)^2} \grad \bar{X}_s(x) \right) \right) - \grad \bar{X}_s(x) \right) ds
    \end{align*}
    with initial condition $\grad \bar{X}_0(x) = 0$. By Tweedie's formula, the optimal denoiser is given by
    \begin{align*}
        D_t(x) = \frac{x}{t} + \frac{(1 - t)^2}{t} \grad \log \rho_t(x),
    \end{align*}
    so taking the gradient of both sides gives us
    \begin{align*}
        \grad D_t(x) = \frac{1}{t} I_d + \frac{(1 - t)^2}{t} \grad^2 \log \rho_t(x).
    \end{align*}
    Since $\rho_t = \sum_{i=1}^m \pi_i\, \rho_{it}$, applying the chain rule to compute the Hessian of $\log \rho_t$ yields
    \begin{align*}
         & \grad D_t(x)                                                                                                                                                                                                            \\
         & \quad = \frac{1}{t} I_d + \frac{(1 - t)^2}{t} \sum_{i=1}^m w_{it}(x) \left( (\grad \log \rho_{it}(x) - \grad \log \rho_t(x)) (\grad \log \rho_{it}(x) - \grad \log \rho_t(x))^\top + \grad^2 \log \rho_{it}(x) \right),
    \end{align*}
    where $w_{it}(x) \coloneq \pi_i\, \rho_{it}(x) / \sum_{j=1}^m \pi_j\, \rho_{jt}(x)$.
\end{proof}

\subsection{Proof of \texorpdfstring{\Cref{thm:mode-selection}}{Theorem (mode selection)}} \label{app:mode-selection-proof}

\thmModeSelection*

\begin{proof}
    For (i), we compute that
    \begin{align*}
        Z_-
        \coloneq \int_{-\infty}^0 e^{\lambda r(x)} \rho_1(x)\, dx
        = e^{-\lambda R} \int_{-\infty}^0 \rho_1(x)\, dx
        = \frac{e^{-\lambda R}}{2}.
    \end{align*}
    On the other hand, we have
    \begin{align*}
        Z_+
        \coloneq \int_0^\infty e^{\lambda r(x)} \rho_1(x)\, dx
        = \int_0^\infty \rho_1(x)\, dx
        = \frac{1}{2}.
    \end{align*}
    Therefore, the correct mode probability under the tilted measure is
    \begin{align*}
        \tilde{p}
        = \P(X \geq 0)
        = \frac{Z_+}{Z_+ + Z_-}
        = (1 + e^{-\lambda R})^{-1}.
    \end{align*}
    For (ii), note that the gradient of the reward is zero Lebesgue-almost everywhere, so by the chain rule the $k$-particle plug-in gradient
    \begin{align*}
        \grad \log \hat{h}_t^{(k)}(x)
        = \sum_{i = 1}^k w_i\, \lambda\, \grad X_1^{(i)}(x)^\top \grad r(X_1^{(i)}(x)),
        \qquad w_i \coloneq \frac{e^{\lambda r(X_1^{(i)}(x))}}{\sum_{j = 1}^k e^{\lambda r(X_1^{(j)}(x))}},
    \end{align*}
    is zero almost everywhere for any finite $k \geq 1$, since the law of $(X_1 \given X_t = x)$ is absolutely continuous with respect to Lebesgue measure. Therefore, defining $\sigma_t^2 \coloneq (1 - t)^2 + t^2 \sigma^2$, the $k$-particle plug-in flow coincides with the unguided flow:
    \begin{align*}
        b_t(x)
        = \left( \frac{t \sigma^2 - (1 - t)}{\sigma_t^2} \right) x + \frac{(1 - t) \mu}{\sigma_t^2} \tanh\!\left( \frac{t \mu x}{\sigma_t^2} \right).
    \end{align*}
    Since the drift is an odd function, the distribution of the final sample is symmetric about zero, so the correct mode probability is $\tilde{p}_1^{(k)} = \frac{1}{2}$. For (iii), note that the $n$ independent finite-$k$ plug-in trajectories are i.i.d., so the probability that all $n$ trajectories end up negative is $(1/2)^n$.
\end{proof}

\section{Experimental details and ablations} \label{app:experiments}

\subsection{Damping for Gaussian target} \label{sec:exp-damping-gaussian}

We use a single Gaussian target $\rho_1 = \Normal(0, 0.5 I_2)$ with quadratic reward $r(x) = -\norm{x - a}_2^2$ centered at $a = (0, 2.5)^\top$ and $\lambda = 3.0$. By \Cref{prop:analytic-guidance-gaussian} and \Cref{thm:plugin-flow-gaussian}, the analytic guidance and $k = 1$ plug-in guidance both produce Gaussian terminal distributions with analytically known means and covariances. We use a Heun integrator for the GLASS flow and draw 200 samples.

\subsection{Gaussian-mixture ablations: within-mode reward hacking} \label{app:gmm-additional}

We provide additional experiments verifying that the plug-in estimator hacks the reward and that reward scale damping mitigates this in a wider variety of settings than the symmetric isotropic 2-component mixture of \Cref{fig:gmm-quadratic-main}. The protocol matches that of the main text: 200 samples drawn with the GLASS-flow Heun integrator at $\lambda = 3.0$, background colormap shows the analytic tilt $\tilde{\rho}_1$ together with green level sets, and the red $\times$ marks the reward target $a$.

\paragraph{Single Gaussian.}
\Cref{fig:gmm-gaussian} verifies the within-mode reward hacking and damping correction predicted by \Cref{thm:plugin-flow-gaussian} on the simplest setting, a single isotropic Gaussian target $\rho_1 = \Normal((0,0)^\top,\, 0.5 I_2)$ with quadratic reward $r(x) = -\norm{x - a}_2^2$ at $a = (0, 2.5)^\top$. The $k = 1$ plug-in over-concentrates exactly as predicted; increasing $k$ does not help; reward damping recovers the analytic tilt.

\image{0.95}{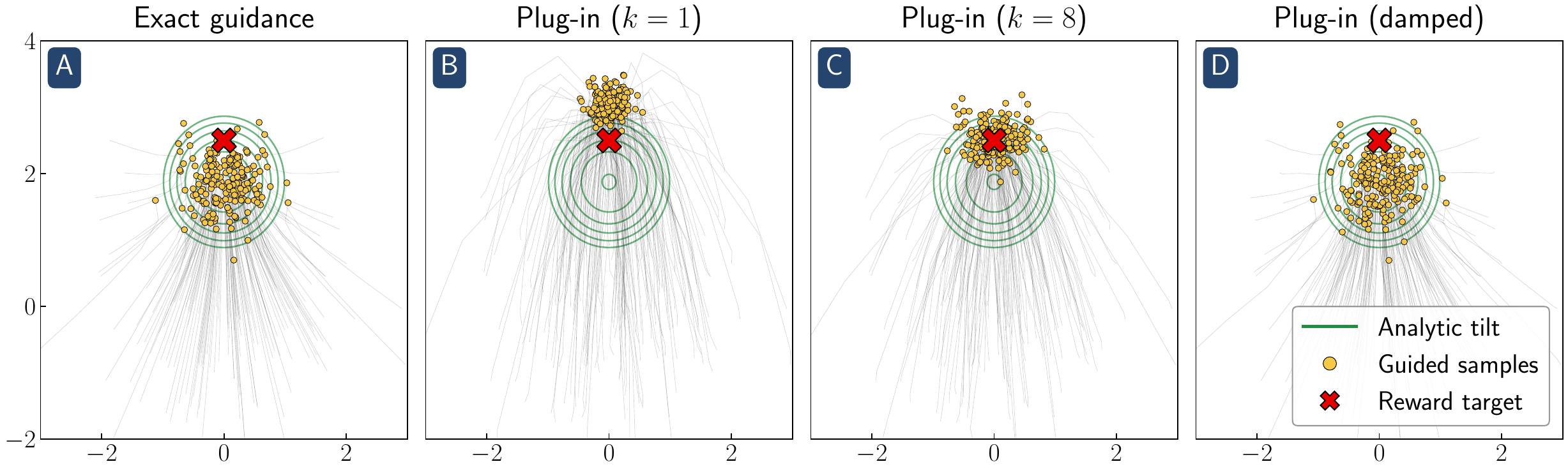}{\textbf{Single Gaussian target.} The $k=1$ plug-in over-concentrates exactly as predicted by \Cref{thm:plugin-flow-gaussian}; increasing $k$ does not help; reward damping recovers the true tilted distribution.}{fig:gmm-gaussian}

\paragraph{Non-quadratic reward.}
To verify that the plug-in bias is not specific to the quadratic reward setting of our theory, \Cref{fig:gmm-doublewell} repeats the experiment with a double-well reward $r(x) = -0.125(x_1^4 - 12.5 x_1^2) - (x_2 - 0.4 x_1)^2$, which has local maxima at $(-2.5, -1.0)$ and $(2.5, 1.0)$. Although the double-well reward has no closed-form Doob $h$-transform, we approximate the exact reward-tilt using a $k = 1000$ plug-in that draws particles directly from this closed-form posterior. The qualitative pattern persists: samples spread across both reward peaks, the $k = 1$ plug-in estimator over-concentrates near the maxima, and reward damping partially corrects this.

\image{0.95}{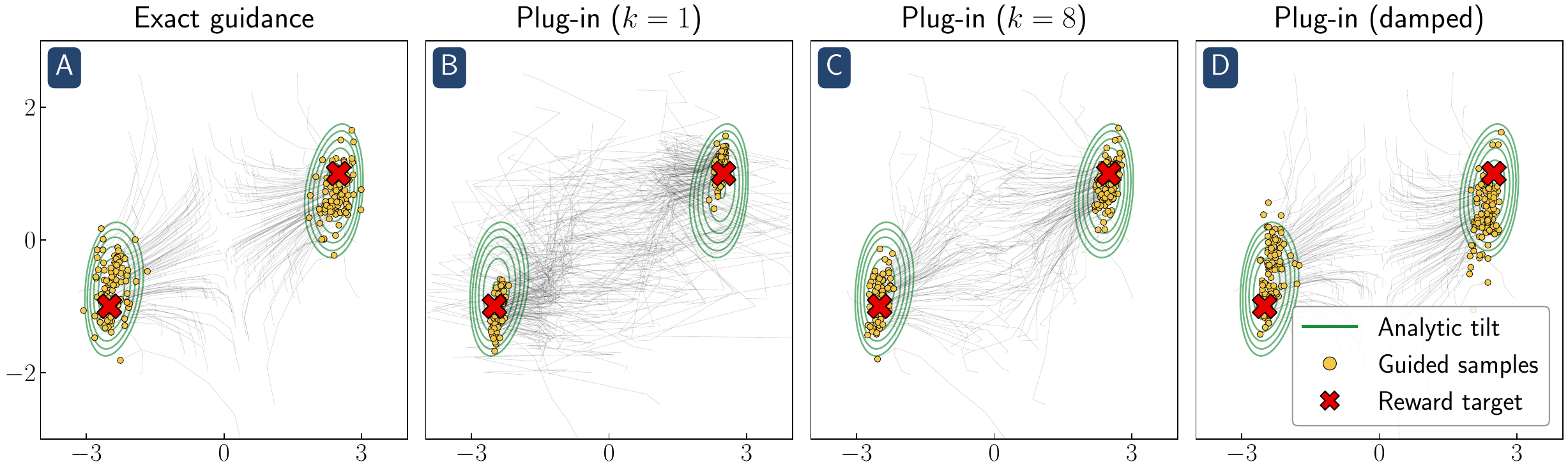}{\textbf{Double-well reward.} The plug-in over-concentrates at the reward maxima (marked by red $\times$); reward damping partially corrects this.}{fig:gmm-doublewell}

\paragraph{GMM ablations.}
\Cref{fig:gmm-noniso,fig:gmm-unequal,fig:gmm-uncentered} verify that the plug-in bias and the effectiveness of reward damping persist across three further variations: (i) \emph{non-isotropic covariances}, where the component covariances are $\bigl[\begin{smallmatrix}0.5 & \pm 0.25 \\ \pm 0.25 & 0.5\end{smallmatrix}\bigr]$ (tilted toward $a$); (ii) \emph{unequal weights} $(\pi_1, \pi_2) = (0.2, 0.8)$; and (iii) \emph{uncentered components} at $(-4.0, 0)$ and $(1.0, 0)$. In each case the plug-in estimator collapses samples near the reward target while reward damping restores spread consistent with the tilted density.

\image{0.95}{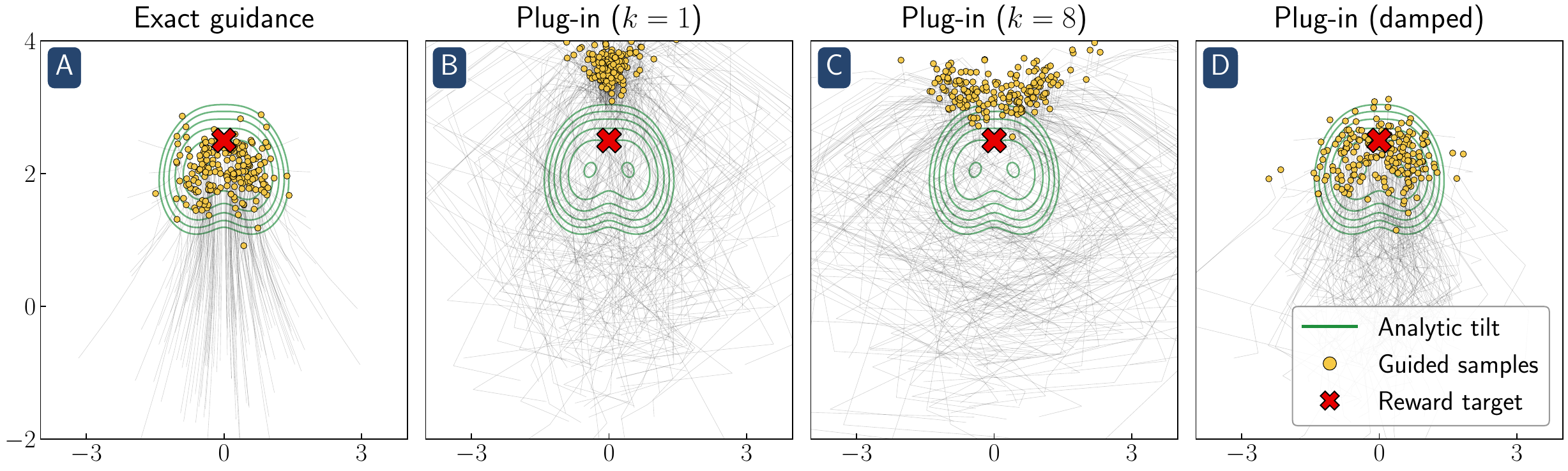}{\textbf{Non-isotropic covariances.} Both components are tilted toward the target $a = (0, 2.5)^\top$ via off-diagonal entries $\pm 0.25$.}{fig:gmm-noniso}

\image{0.95}{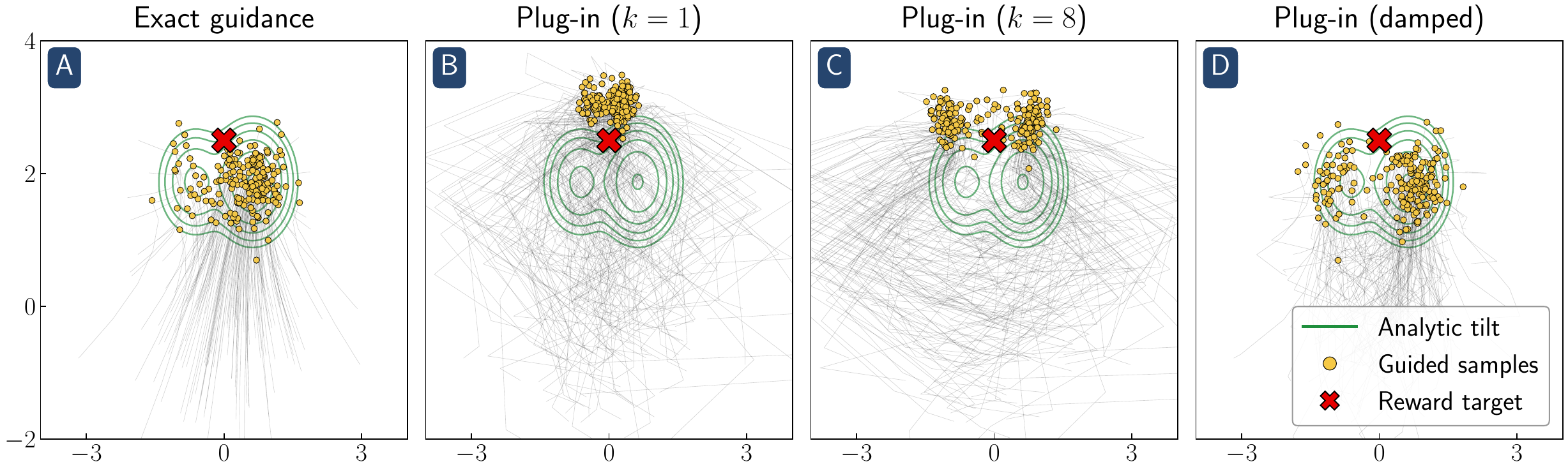}{\textbf{Unequal component weights} $(\pi_1, \pi_2) = (0.2, 0.8)$.}{fig:gmm-unequal}

\image{0.95}{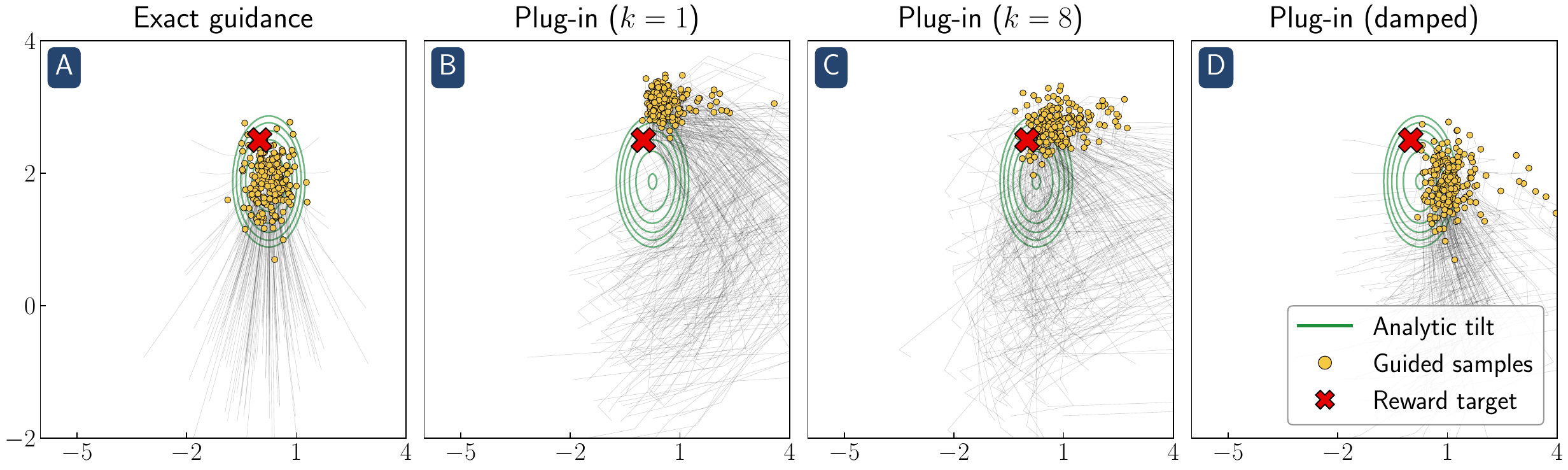}{\textbf{Uncentered components} at $(-4.0, 0)$ and $(1.0, 0)$.}{fig:gmm-uncentered}

\subsection{Checkerboard mode selection: experimental details} \label{sec:exp-checkerboard-details}

In this section, we provide further details on the checkerboard mode-selection experiment described in \Cref{sec:experiments-mode-selection-checkerboard}.

\paragraph{Base model.} The base velocity field $b_t$ is a 4-layer MLP with 256 hidden units and GELU activations, taking the concatenated input $(x / s, t) \in \R^3$ where $s$ is the empirical standard deviation of the checkerboard data. We train it via flow matching against the linear interpolant $I_t = (1 - t)\, \epsilon + t\, X_1$ with $\epsilon \sim \Normal(0, I_2)$ and $X_1$ drawn uniformly from the 18 filled checkerboard squares. The model is optimized with Adam (learning rate $10^{-3}$, cosine annealing to $0$) for $5 \times 10^5$ steps with batch size $4096$, using an exponential moving average of the weights with decay $0.9999$ for inference. Unguided samples are obtained by integrating the learned velocity ODE with a Heun integrator at $200$ steps, starting from $\epsilon \sim \Normal(0, s^2 I_2)$.

\paragraph{Guided sampling.} We use $\lambda = 10$, $k = 1$ particle by default, and the memoryless schedule $\sigma_t = \sqrt{2 (1 - t) / t}$. Guided trajectories are integrated with a Heun integrator at $200$ outer steps, and the GLASS lookahead $(X_1 \given X_t = x)$ is computed with $50$ inner Euler steps. The reward is the Gaussian bump $r(x) = \exp(-\norm{x - c}_2^2 / (2 \sigma_r^2))$ with $c = (0.5, 0.5)^\top$ and $\sigma_r = 1.5$. We draw $5000$ samples per condition.

\paragraph{Analytic samples.} We obtain exact samples from $\tilde{\rho}_1 \propto e^{\lambda r} \rho_1$ via rejection sampling over the 18 filled squares: sample a square uniformly (since $\rho_1$ is uniform on each square), sample a point uniformly within that square, and accept with probability $e^{\lambda(r(x) - 1)}$. This is valid because $r(x) \leq 1$ everywhere (the Gaussian bump peaks at $r = 1$ at its center), so $e^{\lambda}$ is a global upper bound on $e^{\lambda r}$. The worst-case acceptance rate is $e^{-\lambda}$, attained where $r(x) \to 0$ far from $c$; at $c$ itself, where $r(c) = 1$, the acceptance probability is 1.

\subsection{FLUX.1-dev experimental details} \label{sec:exp-flux-details}

\paragraph{Licenses.} FLUX.1-dev is released under the FLUX.1 [dev] Non-Commercial License; Qwen2.5-VL-3B-Instruct is released under the Qwen Research License Agreement; ImageReward is released under the Apache-2.0 License. All three checkpoints are used within the terms of their respective licenses for non-commercial academic research.

\paragraph{Cover photo.} \Cref{fig:cover} compares unguided FLUX.1 samples to guided and damped runs with ImageReward across six prompts. We use the following parameters:
\begin{itemize}
    \item Guided baselines: $\lambda = 50$ for all prompts.
    \item Damped runs: $\lambda = 100$ with $\sigma_\mathrm{damp} = 0.1$ for all prompts except the skyscraper prompt, which uses $\sigma_\mathrm{damp} = 0.15$.
\end{itemize}

\paragraph{Main experiments.} The within-mode and mode-selection experiments shown in \Cref{sec:experiments} use the following parameters.
\begin{itemize}
    \item \Cref{fig:flux-brightness} (welder, masked-brightness): naive $\lambda = 100$, lower-$\lambda$ baseline $\lambda = 50$, $k = 8$ baseline at $\lambda = 50$, damped $\lambda = 100$ with $\sigma_\mathrm{damp} = 0.1$.
    \item \Cref{fig:flux-color} (rococo, blueness): naive $\lambda = 50$, lower-$\lambda$ baseline $\lambda = 30$, $k = 8$ baseline at $\lambda = 50$, damped $\lambda = 100$ with $\sigma_\mathrm{damp} = 0.1$.
    \item \Cref{fig:flux-fox} (fox, blueness): naive $\lambda = 100$, lower-$\lambda$ baseline $\lambda = 50$, $k = 8$ baseline at $\lambda = 50$, damped $\lambda = 100$ with $\sigma_\mathrm{damp} = 0.1$.
    \item \Cref{fig:flux-archaeo} (archaeologist, ImageReward): naive $\lambda = 100$, lower-$\lambda$ baseline $\lambda = 50$, $k = 8$ baseline at $\lambda = 50$, damped $\lambda = 100$ with $\sigma_\mathrm{damp} = 0.15$.
    \item \Cref{fig:flux-miner} (miner, ImageReward): naive $\lambda = 100$, lower-$\lambda$ baseline $\lambda = 50$, $k = 8$ baseline at $\lambda = 50$, damped $\lambda = 100$ with $\sigma_\mathrm{damp} = 0.1$.
    \item \Cref{fig:flux-market} (market, ImageReward): naive $\lambda = 50$, lower-$\lambda$ baseline $\lambda = 30$, $k = 8$ baseline at $\lambda = 50$, damped $\lambda = 100$ with $\sigma_\mathrm{damp} = 0.05$.
    \item \Cref{fig:ms-flux-diner} (ECLIPSE DINER, VLM): naive $\lambda = 100$, $k = 8$ baseline at $\lambda = 50$, damped $\lambda = 100$ with $\sigma_\mathrm{damp} = 0.1$. Best-of-$n$ selections taken from 20 i.i.d.\ generations per row.
    \item \Cref{fig:ms-flux-subway} (NEXT TRAIN MARS, VLM): same parameters as \Cref{fig:ms-flux-diner}.
\end{itemize}

\paragraph{Guidance.} The plug-in guidance gradient $\grad \log \hat{h}_t^{(k)}(x)$ is rescaled to unit norm before applying guidance; this is a standard approximation made for numerical stability \citep{holderrieth2026diamond, huang2026guide}. We see empirically that this approximation corresponds to a slight damping effect, decreasing the effective guidance scale $\lambda$ throughout the trajectory. Following \citet{holderrieth2026diamond}, we apply guidance for $5$ guidance steps starting from outer step $1$, skipping steps with $\sigma_t > 0.9$ to ensure numerical stability. Reward damping (\Cref{sec:reward-damping}) replaces $\lambda$ with the time-dependent schedule \eqref{eq:reward-damping}, parameterized by $\sigma_\mathrm{damp}$. Gradient checkpointing on the FLUX transformer is used to fit backpropagation in memory. We generate $4$ images per condition for the within-mode reward hacking grids and $20$ images per condition for the mode-selection best-of-$n$ experiments.

\subsection{Compute} \label{sec:exp-compute}

All FLUX text-to-image experiments run on a single NVIDIA RTX A6000 (48 GB) or L40S (48 GB) GPU, with 8 CPUs and 64 GB RAM. At 512 x 512 resolution with 28 inference steps, an unguided sample takes approximately 7-15 seconds; a $k = 1$ guided sample with 5 Diamond-map guidance steps takes 11-23 seconds; and a $k = 8$ guided sample takes 37-76 seconds. A typical 20-image condition therefore completes in 4-8 minutes for $k = 1$ and roughly 25 minutes for $k = 8$. Including hyperparameter sweeps over $\lambda$, $\sigma_\mathrm{damp}$, gradient-norm scale, prompt variants, and earlier lookahead schemes, the total FLUX compute is approximately 50-55 GPU-hours, of which roughly 10 GPU-hours back the figures and tables in the paper. Checkerboard base-model training ($5 \times 10^5$ steps of a 4-layer MLP) takes about 60 minutes on a single GPU, and the reported guided-sampling grid adds another 1-2 GPU-hours. The Gaussian-mixture and 1D mode-selection experiments are closed-form analytics with small numerical integrators and complete in CPU-minutes.

\subsection{FLUX.1 mode-selection: additional experiments} \label{app:flux-mode-selection-additional}

\Cref{tab:ms-flux} reports the mean VLM probability of a legible ``ECLIPSE'' sign for the diner experiment of \Cref{sec:experiments-mode-selection-flux} across best-of-$n$ values, complementing the qualitative results of \Cref{fig:ms-flux-diner,fig:ms-flux-subway}. We report VLM scores as probabilities $\mathrm{sigmoid}(\log p(\text{Yes}) - \log p(\text{No}))$.

\begin{table}
    \centering
    \caption{\textbf{Mode selection on FLUX.1 (ECLIPSE DINER).} Mean VLM probability of a clear ``ECLIPSE'' sign with best-of-$n$ sampling (higher is better). Uncertainties are $\pm 2$ bootstrap standard errors over the 20 generations per condition.}
    \label{tab:ms-flux}
    \small
    \setlength{\tabcolsep}{4pt}
    \begin{tabular}{lccccc}
        \toprule
        Method          & $n=1$                 & $n=2$                 & $n=4$                 & $n=8$                 & $n=16$                \\
        \midrule
        Unguided        & $0.781_{\,\pm 0.074}$ & $0.857_{\,\pm 0.032}$ & $0.885_{\,\pm 0.021}$ & $0.899_{\,\pm 0.014}$ & $0.904_{\,\pm 0.008}$ \\
        Guided          & $0.626_{\,\pm 0.104}$ & $0.761_{\,\pm 0.081}$ & $0.840_{\,\pm 0.054}$ & $0.879_{\,\pm 0.041}$ & $0.901_{\,\pm 0.037}$ \\
        Guided ($k=8$)  & $0.776_{\,\pm 0.070}$ & $0.851_{\,\pm 0.033}$ & $0.881_{\,\pm 0.022}$ & $0.897_{\,\pm 0.020}$ & $0.910_{\,\pm 0.024}$ \\
        Guided (damped) & $0.826_{\,\pm 0.029}$ & $0.863_{\,\pm 0.024}$ & $0.887_{\,\pm 0.022}$ & $0.904_{\,\pm 0.020}$ & $0.914_{\,\pm 0.018}$ \\
        \bottomrule
    \end{tabular}
\end{table}

\subsection{Gaussian-mixture mode-selection experiments} \label{app:mode-selection-gmm}

We provide additional empirical results for the 1D symmetric Gaussian mixture $\rho_1 = \tfrac{1}{2}\Normal(\mu, \sigma^2) + \tfrac{1}{2}\Normal(-\mu, \sigma^2)$ with $\mu = 5$, $\sigma = 1$, and $\lambda = 5.0$, complementing the empirical illustration in \Cref{fig:ms-step-overview}.

\paragraph{Step function reward.} \label{app:mode-selection-step}
With the step reward $r(x) = \indicator_{x \geq 0}$, the gradient is zero almost everywhere, so the plug-in estimator (run with $k = 1$) receives no signal and achieves correct-mode probability $\tilde{p}_1^{(1)} = 0.502$, essentially random. Best-of-$n$ sampling exactly recovers the prediction $\tilde{p}_n^{(1)} = 1 - (1/2)^n$ from \Cref{thm:mode-selection} (\Cref{tab:ms-step}). \Cref{fig:ms-step-traj-combined} shows the underlying particle trajectories: unguided and plug-in trajectories split roughly evenly between the two modes, while best-of-$n$ selection concentrates on positive-half trajectories.

\begin{table}
    \centering
    \caption{\textbf{Step-reward best-of-$n$.} Empirical $\tilde{p}_n^{(1)}$ for best-of-$n$ guided sampling with $k = 1$. Theory: $1 - (1/2)^n$. Uncertainties are $\pm 2$ binomial standard errors $2 \sqrt{\hat{p}(1-\hat{p})/M}$ over $M = 1{,}000$ independent groups of $32$ trajectories.}
    \label{tab:ms-step}
    \small
    \setlength{\tabcolsep}{3pt}
    \begin{tabular}{lcccccc}
        \toprule
                                      & $n=1$                 & $n=2$                 & $n=4$                 & $n=8$                 & $n=16$                & $n=32$                \\
        \midrule
        Empirical $\tilde{p}_n^{(1)}$ & $0.506_{\,\pm 0.032}$ & $0.740_{\,\pm 0.028}$ & $0.931_{\,\pm 0.016}$ & $0.996_{\,\pm 0.004}$ & $1.000_{\,\pm 0.000}$ & $1.000_{\,\pm 0.000}$ \\
        Theory $1 - (1/2)^n$          & $0.500$               & $0.750$               & $0.938$               & $0.996$               & $1.000$               & $1.000$               \\
        \bottomrule
    \end{tabular}
\end{table}

\image{0.95}{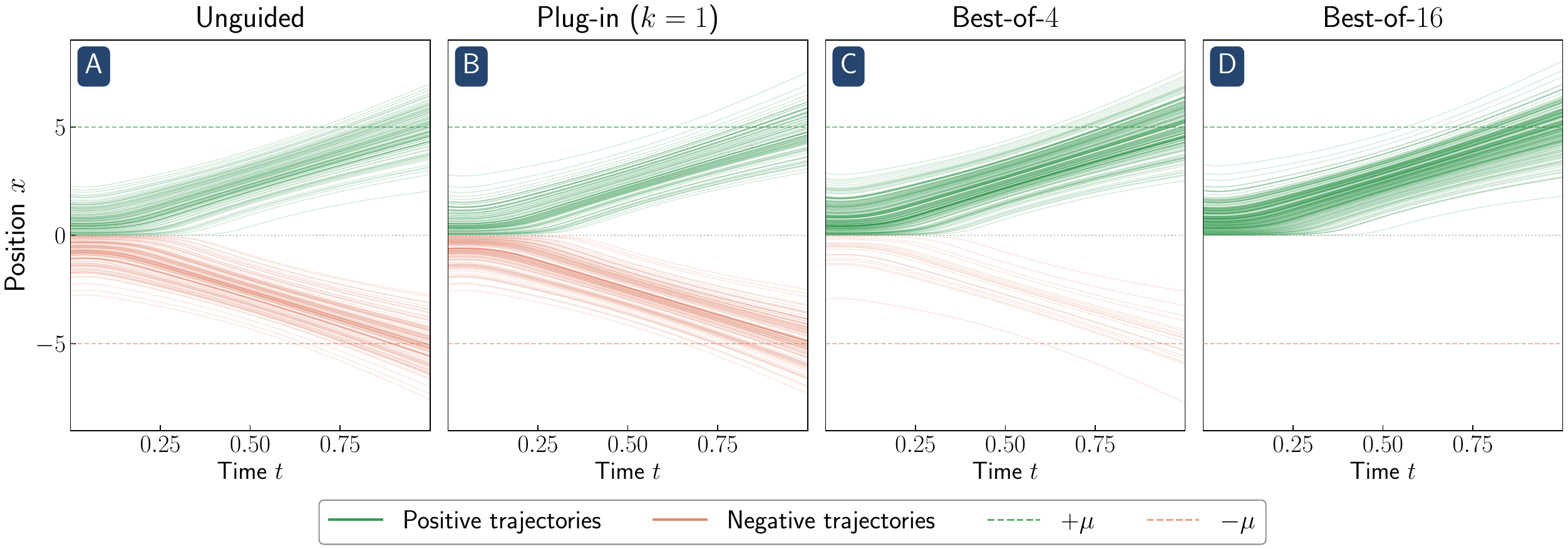}{\textbf{Step reward trajectories.} Colored by terminal sign. Unguided and plug-in trajectories split evenly between positive (green) and negative (coral) modes; best-of-$n$ selection produces a strong majority of positive-mode trajectories, increasingly so as $n$ grows.}{fig:ms-step-traj-combined}

\paragraph{Gaussian reward.} \label{app:mode-selection-gaussian}
We repeat the same 1D mixture experiment with a Gaussian reward $r(x) = \exp(-(x - 3)^2 / 2)$ centered at $x = 3$ in the positive half. \Cref{thm:mode-selection} does not apply here---the reward is smooth and the plug-in estimator receives a nonzero gradient signal---so this case goes beyond what our theory covers. Nevertheless, the qualitative behavior persists: plug-in guidance skews trajectories toward the positive mode (\Cref{fig:ms-gaussian-traj}, left), and best-of-$n$ further concentrates them (\Cref{fig:ms-gaussian-traj}, center and right).

\image{0.85}{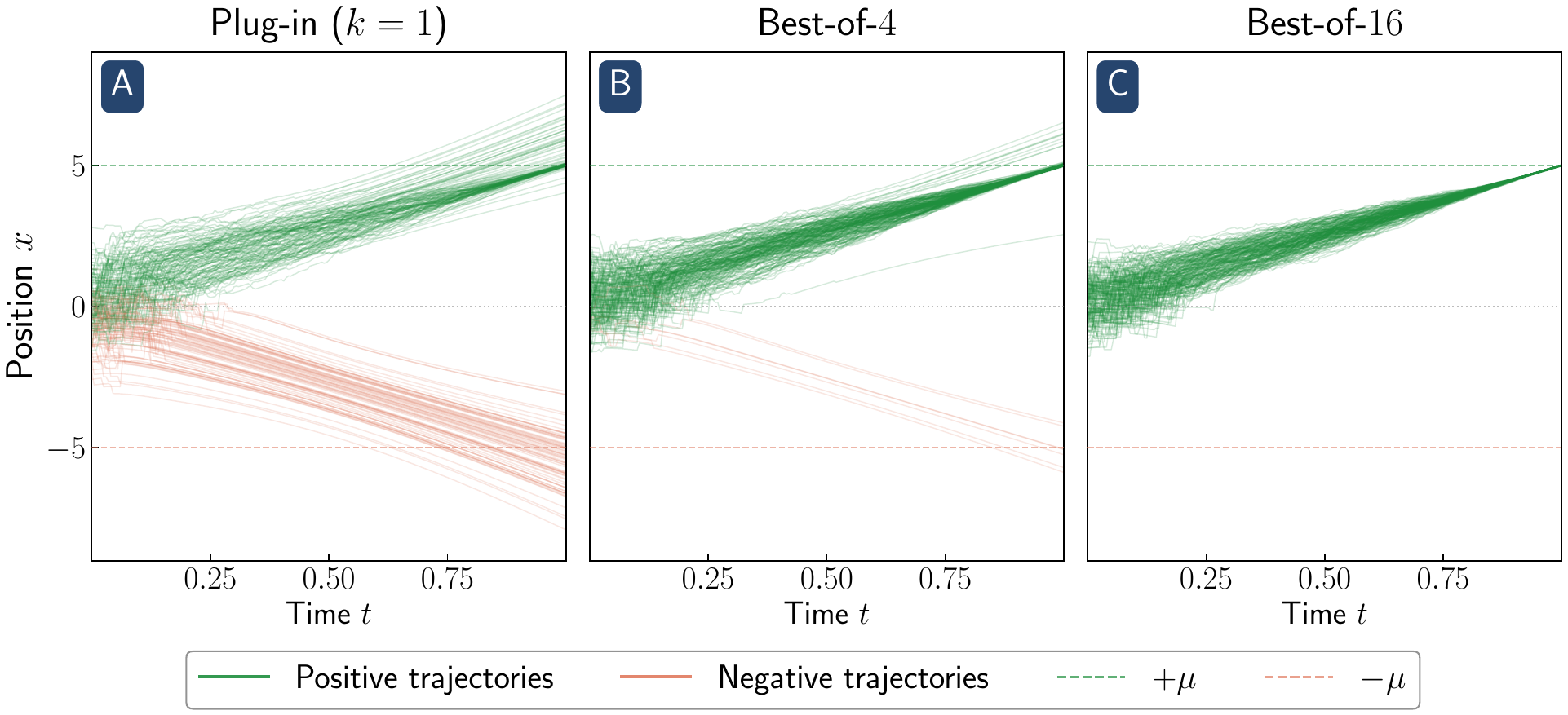}{\textbf{Gaussian reward trajectories.} Plug-in already skews toward the positive mode; best-of-$n$ further concentrates it.}{fig:ms-gaussian-traj}

\section{Comparison with flow map reward guidance (FMRG)} \label{app:fmrg-comparison}

\citet{huang2026guide} propose a deterministic guidance scheme derived directly from the probability flow ODE, without passing through a diffusion. Given the unguided probability flow $dX_s = b_s(X_s)\, ds$ and flow map $X_{t, s}(x)$ denoting its solution at time $s$ given $X_t = x$, the \emph{flow map reward guidance} for the reward $r$ is the feedback control
\begin{equation*}
    u_t^{\mathrm{FMRG}}(x) \coloneq \lambda\, \grad X_{t, 1}(x)^\top \grad r(X_{t, 1}(x)),
\end{equation*}
and the guided trajectory solves
\begin{equation} \label{eq:fmrg-ode}
    d\tilde{x}_t = (b_t(\tilde{x}_t) + u_t^{\mathrm{FMRG}}(\tilde{x}_t))\, dt
\end{equation}
with $\tilde{x}_0 \sim \Normal(0, I_d)$. For the quadratic reward $r(x) = -\norm{x - a}_2^2$, this simplifies to
\begin{equation} \label{eq:fmrg-quadratic}
    u_t^{\mathrm{FMRG}}(x) = -2 \lambda\, \grad X_{t, 1}(x)^\top (X_{t, 1}(x) - a).
\end{equation}
This is structurally similar to the $k = 1$ plug-in guidance of \Cref{thm:plugin-flow-gaussian}: both schemes are of the form $-2 \lambda\, (\text{Jacobian})^\top (\text{endpoint} - a)$, and they differ only in which endpoint is used. The plug-in scheme evaluates at a random conditional sample $X_1^{(1)}(x) \sim (X_1 \given X_t = x)$ and averages over the initial noise in the GLASS flow, while the flow map reward scheme evaluates at the deterministic probability flow endpoint $X_{t, 1}(x)$. Neither scheme targets the exact Doob $h$-function $h_t(x) = \E[e^{\lambda r(X_1)} \given X_t = x]$, so both are single-endpoint surrogates that do not capture the long-range mode attraction of the exact $h$-transform. In this appendix, we carry out the analysis of \citet{huang2026guide} in $\R^d$ and extend it to Gaussian mixture targets, allowing us to compare the flow map reward guidance directly with the plug-in formulae of \Cref{thm:plugin-flow-gaussian,thm:plugin-flow-gmm}.

\subsection{Gaussian target}

Throughout this subsection, $\rho_1 = \Normal(\mu, \Sigma)$ with $\Sigma \succ 0$, and $\Sigma_t \coloneq (1 - t)^2 I_d + t^2 \Sigma$ denotes the marginal covariance of $I_t$ as in \Cref{thm:ivf-bias-gaussian}. We begin by computing the flow map in closed form.

\begin{lemma}[Flow map for Gaussian target] \label{lem:gauss-flow map}
    Suppose $\rho_1 = \Normal(\mu, \Sigma)$. Then the probability flow ODE $dX_s = b_s(X_s)\, ds$ with $b_s(x) = \mu + \frac{1}{2} \dot{\Sigma}_s \Sigma_s^{-1} (x - s \mu)$ has flow map
    \begin{align*}
        X_{t, 1}(x) = \mu + \Sigma^{1/2} \Sigma_t^{-1/2} (x - t \mu)
    \end{align*}
    and state-independent Jacobian $\grad X_{t, 1}(x) = \Sigma^{1/2} \Sigma_t^{-1/2}$.
\end{lemma}
\begin{proof}
    Along a trajectory, define $Y_s \coloneq X_s - s \mu$ so that $dY_s = \frac{1}{2} \dot{\Sigma}_s \Sigma_s^{-1} Y_s\, ds$. Because $\Sigma_s$ is a matrix polynomial in $\Sigma$, the family $\{\dot{\Sigma}_s \Sigma_s^{-1}\}_s$ commutes with itself and with $\Sigma$, so integrating from $s = t$ to $s = 1$ yields
    \begin{align*}
        Y_1
        = \exp\!\left( \frac{1}{2} \int_t^1 \dot{\Sigma}_s \Sigma_s^{-1}\, ds \right) Y_t
        = \exp\!\left( \frac{1}{2} \log(\Sigma \Sigma_t^{-1}) \right) Y_t
        = \Sigma^{1/2} \Sigma_t^{-1/2} Y_t.
    \end{align*}
    Adding $\mu$ back gives the flow map, and the Jacobian follows by differentiating in $x$.
\end{proof}

\begin{theorem}[Flow map reward guidance for Gaussian target] \label{thm:fmrg-gaussian}
    Suppose $\rho_1 = \Normal(\mu, \Sigma)$ and consider the flow map reward guided ODE~\eqref{eq:fmrg-ode} with $u_t^{\mathrm{FMRG}}$ given by~\eqref{eq:fmrg-quadratic}. Then, the guidance term is
    \begin{align*}
        u_t^{\mathrm{FMRG}}(x)
        = -2 \lambda\, \Sigma^{1/2} \Sigma_t^{-1/2} (X_{t, 1}(x) - a)
        = -2 \lambda\, \Sigma \Sigma_t^{-1} (x - t \mu) - 2 \lambda\, \Sigma^{1/2} \Sigma_t^{-1/2} (\mu - a),
    \end{align*}
    and the terminal distribution of the guided ODE is $\tilde{x}_1 \sim \Normal(\tilde{\mu}^{\mathrm{FMRG}},\, \tilde{\Sigma}^{\mathrm{FMRG}})$, where
    \begin{align*}
        \tilde{\mu}^{\mathrm{FMRG}} \coloneq \mu - T_{\mathrm{pull}}^{\mathrm{FMRG}} (\mu - a), \qquad
        T_{\mathrm{pull}}^{\mathrm{FMRG}} \coloneq I_d - \exp(-\pi \lambda \Sigma^{1/2}),
    \end{align*}
    and
    \begin{align*}
        \tilde{\Sigma}^{\mathrm{FMRG}} \coloneq \Sigma \exp(-2 \pi \lambda \Sigma^{1/2}).
    \end{align*}
\end{theorem}
\begin{proof}
    We begin by computing the guidance term $u_t^{\mathrm{FMRG}}(x)$. From \Cref{lem:gauss-flow map}, the Jacobian $\grad X_{t, 1}(x) = \Sigma^{1/2} \Sigma_t^{-1/2}$ is a matrix function of $\Sigma$ and hence symmetric. Substituting into~\eqref{eq:fmrg-quadratic} and expanding $X_{t, 1}(x)$ using the flow map formula gives
    \begin{align*}
        u_t^{\mathrm{FMRG}}(x)
        = -2 \lambda\, \Sigma^{1/2} \Sigma_t^{-1/2} (X_{t, 1}(x) - a)
        = -2 \lambda\, \Sigma \Sigma_t^{-1} (x - t \mu) - 2 \lambda\, \Sigma^{1/2} \Sigma_t^{-1/2} (\mu - a).
    \end{align*}
    Substituting back into the FMRG ODE~\eqref{eq:fmrg-ode}, the dynamics are given by
    \begin{align*}
        d\tilde{x}_t = \left( \mu + \frac{1}{2} \dot{\Sigma}_t \Sigma_t^{-1} (\tilde{x}_t - \mu_t) - 2 \lambda\, \Sigma \Sigma_t^{-1} (\tilde{x}_t - \mu_t) - 2 \lambda\, \Sigma^{1/2} \Sigma_t^{-1/2} (\mu - a) \right) dt.
    \end{align*}
    Because the initial condition $\tilde{x}_0 \sim \Normal(0, I_d)$ is Gaussian and the drift is an affine function of $\tilde{x}_t$, the process $\tilde{x}_t$ remains exactly Gaussian for all $t \in [0, 1]$, and we track its mean $\mu_t^{\mathrm{FMRG}} = \E[\tilde{x}_t]$ and covariance $\Sigma_t^{\mathrm{FMRG}} = \Cov(\tilde{x}_t)$. Taking the expectation of the drift yields an ODE for the mean:
    \begin{align*}
        d\mu_t^{\mathrm{FMRG}}
        = \mu + \left( \frac{1}{2} \dot{\Sigma}_t \Sigma_t^{-1} - 2 \lambda\, \Sigma \Sigma_t^{-1} \right) (\mu_t^{\mathrm{FMRG}} - \mu_t) - 2 \lambda\, \Sigma^{1/2} \Sigma_t^{-1/2} (\mu - a).
    \end{align*}
    Defining $\Phi_t \coloneq \int_0^t \Sigma \Sigma_s^{-1}\, ds$, we multiply by the integrating factor $\Sigma_t^{-1/2} \exp(2 \lambda \Phi_t)$ to get
    \begin{align*}
        \frac{d}{dt} \left( \Sigma_t^{-1/2} \exp(2 \lambda \Phi_t) (\mu_t^{\mathrm{FMRG}} - \mu_t) \right)
        = -\Sigma^{-1/2}\, \frac{d}{dt} \exp(2 \lambda \Phi_t)\, (\mu - a),
    \end{align*}
    where we used that $\dot{\Phi}_t = \Sigma \Sigma_t^{-1}$. Integrating both sides from 0 to 1 then gives us
    \begin{align*}
        \Sigma^{-1/2} \exp(2 \lambda \Phi_1) (\tilde{\mu}^{\mathrm{FMRG}} - \mu)
        = -\Sigma^{-1/2} (\exp(2 \lambda \Phi_1) - I_d) (\mu - a),
    \end{align*}
    so $\tilde{\mu}^{\mathrm{FMRG}} = \mu - (I_d - \exp(-2 \lambda \Phi_1)) (\mu - a)$. To evaluate $\Phi_1$, we note that all relevant matrices are matrix functions of $\Sigma$ and hence simultaneously diagonalize. Along an eigendirection of $\Sigma$ with eigenvalue $\sigma > 0$ and substituting $u = s / (1 - s)$,
    \begin{align*}
        \int_0^1 \frac{\sigma}{(1 - s)^2 + s^2 \sigma}\, ds
        = \int_0^\infty \frac{\sigma}{1 + u^2 \sigma}\, du
        = \frac{\pi \sqrt{\sigma}}{2},
    \end{align*}
    so $\Phi_1 = \frac{\pi}{2}\, \Sigma^{1/2}$ and $\tilde{\mu}^{\mathrm{FMRG}} = \mu - (I_d - \exp(-\pi \lambda \Sigma^{1/2})) (\mu - a)$, which is the claimed form. Next, we have the linear ODE
    \begin{align*}
        d(\tilde{x}_t - \mu_t^{\mathrm{FMRG}})
        = \left( \frac{1}{2} \dot{\Sigma}_t \Sigma_t^{-1} - 2 \lambda\, \Sigma \Sigma_t^{-1} \right) (\tilde{x}_t - \mu_t^{\mathrm{FMRG}})\, dt,
    \end{align*}
    so by the chain rule, the covariance evolves according to the equation
    \begin{align*}
        d\Sigma_t^{\mathrm{FMRG}}
        = \left( \dot{\Sigma}_t \Sigma_t^{-1} - 4 \lambda\, \Sigma \Sigma_t^{-1} \right) \Sigma_t^{\mathrm{FMRG}}\, dt.
    \end{align*}
    Because all relevant matrices commute, we can solve this ODE in closed form at $t = 1$:
    \begin{align*}
        \tilde{\Sigma}^{\mathrm{FMRG}}
        = \exp\!\left( \int_0^1 \left( \dot{\Sigma}_s \Sigma_s^{-1} - 4 \lambda\, \Sigma \Sigma_s^{-1} \right) ds \right) \Sigma_0^{\mathrm{FMRG}}
        = \Sigma \exp(-2 \pi \lambda \Sigma^{1/2}),
    \end{align*}
    using $\int_0^1 \dot{\Sigma}_s \Sigma_s^{-1}\, ds = \log \Sigma$ and $\int_0^1 \Sigma \Sigma_s^{-1}\, ds = \frac{\pi}{2}\, \Sigma^{1/2}$.
\end{proof}

The proof follows by solving linear ODEs for the mean and covariance of the FMRG flow in closed form. Since $T_{\mathrm{pull}}^{\mathrm{FMRG}}$ and $\tilde{\Sigma}^{\mathrm{FMRG}}$ are matrix functions of $\Sigma$, all operators simultaneously diagonalize. Along an eigendirection $v$ with $\Sigma$ having eigenvalue $\sigma > 0$, the eigenvalues of $T_\mathrm{pull}^{\mathrm{FMRG}}$ and $T_\mathrm{pull}^{(1)}$ from \Cref{thm:plugin-flow-gaussian} along $v$ are
\begin{align*}
    \lambda_v(T_\mathrm{pull}^{\mathrm{FMRG}}) = 1 - e^{-\pi \lambda \sqrt{\sigma}}, \qquad
    \lambda_v(T_\mathrm{pull}^{(1)}) = \sqrt{\pi \lambda \sigma}\, e^{-\lambda \sigma}\, \erfi\!\left(\sqrt{\lambda \sigma}\right).
\end{align*}
The FMRG eigenvalue is monotonically increasing in $\lambda$ and bounded above by 1, so the FMRG mean contracts toward $a$ without ever overshooting; in contrast, the plug-in eigenvalue crosses 1 around $\lambda \sigma \approx 0.854$ and exhibits mean overshoot. Similarly, the eigenvalues of $\tilde{\Sigma}^{\mathrm{FMRG}}$ and $\Sigma^{(1)}$ along $v$ are
\begin{align*}
    \lambda_v(\tilde{\Sigma}^{\mathrm{FMRG}}) = \sigma\, e^{-2 \pi \lambda \sqrt{\sigma}}, \qquad
    \lambda_v(\Sigma^{(1)}) = \sigma\, e^{-2 \lambda \sigma}.
\end{align*}
Both covariances shrink exponentially in $\lambda$, but with different scalings: the plug-in exponent grows like $\lambda \sigma$ while the FMRG exponent grows like $\lambda \sqrt{\sigma}$, with a crossover at $\sigma = \pi^2$. Although flow map reward guidance avoids mean overshoot, it still exponentially contracts the covariance, so flow map reward guidance also exhibits reward hacking on a Gaussian target.

\subsection{Gaussian mixture target}

We now extend the analysis to $\rho_1 = \sum_{i=1}^m \pi_i\, \Normal(\cdot \given \mu_i, \Sigma_i)$. Let $\Sigma_{it} \coloneq (1 - t)^2 I_d + t^2 \Sigma_i$, let $\rho_{it}$ denote the unguided marginal density for the $i$th component, and let $w_{it}(x) \coloneq \pi_i \rho_{it}(x) / \rho_t(x)$ with $\rho_t = \sum_{i=1}^m \pi_i \rho_{it}$. The probability flow drift is the weighted combination
\begin{equation} \label{eq:mix-drift}
    b_t(x) = \sum_{i = 1}^m w_{it}(x)\, b_{it}(x),
    \qquad b_{it}(x) \coloneq \mu_i + \frac{1}{2} \dot{\Sigma}_{it} \Sigma_{it}^{-1} (x - t \mu_i),
\end{equation}
which is nonlinear in $x$ through the mixture weights $w_{it}(x)$. Unlike the single-Gaussian case, the flow map $X_{t, 1}$ has no closed form and must be obtained by integrating the ODE numerically.

\begin{theorem}[Flow map reward guidance for Gaussian mixture target] \label{thm:fmrg-gmm}
    In the above setting, fix $t \in [0, 1)$ and let $X_{t, s}(x)$ denote the solution of the probability flow ODE $dX_s = b_s(X_s)\, ds$ with drift~\eqref{eq:mix-drift} and initial condition $X_t = x$. Then, the flow map reward guidance is
    \begin{align*}
        u_t^{\mathrm{FMRG}}(x) = -2 \lambda\, \grad X_{t, 1}(x)^\top (X_{t, 1}(x) - a),
    \end{align*}
    where the Jacobian $\grad X_{t, s}(x)$ can be computed by integrating the variational equation
    \begin{equation} \label{eq:fmrg-mix-variational}
        \frac{d}{ds} \grad X_{t, s}(x) = \grad b_s(X_{t, s}(x))\, \grad X_{t, s}(x),
        \qquad \grad X_{t, t}(x) = I_d.
    \end{equation}
    In this formula, the spatial Jacobian of the mixture drift is
    \begin{align*}
        \grad b_s(y)
        = \sum_{i = 1}^m w_{is}(y) \left[
            \frac{1}{2} \dot{\Sigma}_{is} \Sigma_{is}^{-1}
            + (b_{is}(y) - b_s(y)) (\grad \log \rho_{is}(y) - \grad \log \rho_s(y))^\top
            \right].
    \end{align*}
\end{theorem}
\begin{proof}
    The guidance formula follows from~\eqref{eq:fmrg-quadratic}, and the variational equation~\eqref{eq:fmrg-mix-variational} follows from differentiating $dX_s = b_s(X_s)\, ds$ with respect to the initial condition and applying the chain rule. For the expression for $\grad b_s$, the product rule applied to~\eqref{eq:mix-drift} gives
    \begin{align*}
        \grad b_s(y)
        = \sum_{i = 1}^m \left( w_{is}(y)\, \grad b_{is}(y) + b_{is}(y)\, \grad w_{is}(y)^\top \right).
    \end{align*}
    Since $b_{is}$ is affine in $y$, we have $\grad b_{is}(y) = \frac{1}{2} \dot{\Sigma}_{is} \Sigma_{is}^{-1}$, and the standard softmax gradient gives $\grad w_{is}(y) = w_{is}(y) (\grad \log \rho_{is}(y) - \grad \log \rho_s(y))$. Because the weights sum to 1, we have $\sum_i \grad w_{is}(y) = 0$, and we can use this to recenter $b_{is}(y) - b_s(y)$ inside the second sum, yielding the claimed formula.
\end{proof}

Comparing \Cref{thm:fmrg-gmm} with \Cref{thm:plugin-flow-gmm}, the FMRG guidance has the same algebraic form as the $k = 1$ plug-in guidance: the random GLASS endpoint $\bar{X}_1(x)$ is replaced by the deterministic probability flow endpoint $X_{t, 1}(x)$, and the GLASS Jacobian by the unguided flow map Jacobian $\grad X_{t, 1}(x)$. In both cases, the reward enters only through the displacement ($\mathrm{endpoint} - a$), and the Jacobian is determined entirely by the unguided dynamics.

Since the probability flow starting from $x$ typically transports mass to its own mode, a far-away mode contributes negligibly to $\grad X_{t, 1}(x)^\top (X_{t, 1}(x) - a)$ regardless of how large its reward is. The mode selection failure of \Cref{thm:mode-selection} therefore applies equally to flow map reward guidance, and best-of-$n$ sampling compensates for it verbatim with $u_t^{(1)}$ replaced by $u_t^{\mathrm{FMRG}}$. Reward damping (\Cref{sec:reward-damping}) does not lift to FMRG, however: from \Cref{thm:fmrg-gaussian}, $u_t^{\mathrm{FMRG}}(x)$ has mismatched coefficients on $(x - t \mu)$ and $(\mu - a)$ that no scalar schedule $\lambda_t$ can simultaneously match, so the FMRG bias is not simply multiplicative, even for an isotropic Gaussian target.

\subsection{Empirical comparison} \label{app:fmrg-empirical}

We empirically verify two claims from the analysis above. First, we show that flow map reward guidance and the $k = 1$ plug-in flow exhibit qualitatively similar reward hacking on FLUX.1-dev with a blueness reward; the two schemes drive the model in the same direction and produce visually comparable reward-hacked outputs. Second, on the single-Gaussian target of \Cref{thm:fmrg-gaussian}, we verify the $\sigma$-regime crossover predicted by the analysis: the eigenvalues of $\tilde{\Sigma}^{\mathrm{FMRG}}$ are smaller than those of $\Sigma^{(1)}$ when $\sigma < \pi^2$ (FMRG hacks more aggressively than the plug-in flow), and larger when $\sigma > \pi^2$ (FMRG is less aggressive).

\paragraph{FLUX.1-dev.} For the FLUX run in \Cref{fig:fmrg-flux-blueness}, we use the implementation of flow map reward guidance defined by \eqref{eq:fmrg-ode} and \eqref{eq:fmrg-quadratic}. The flow map reward guidance and the $k = 1$ plug-in are evaluated at matched gradient-norm scale on a blue-minus-red-green reward, and produce visually similar reward-hacked outputs. This is consistent with the structural analysis of \Cref{thm:fmrg-gaussian,thm:fmrg-gmm}: both schemes are single-endpoint surrogates and exhibit the same characteristic reward hacking failure mode.

\image{0.85}{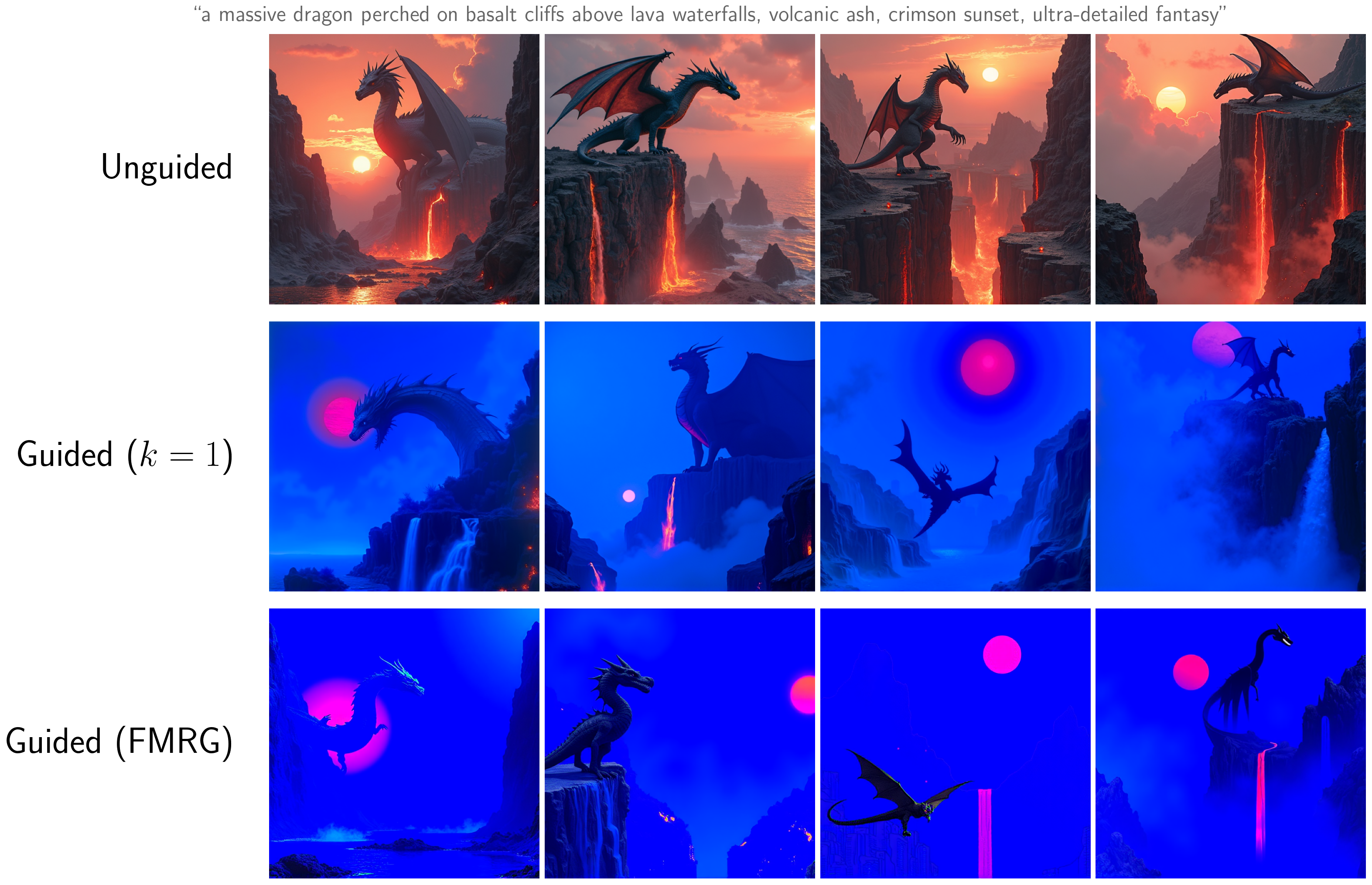}{\textbf{FMRG vs.\ plug-in on FLUX.1.} Top: unguided samples. Middle: the $k = 1$ plug-in flow. Bottom: flow map reward guidance. Both guided schemes drive the entire image onto a blue tint, sharing the same reward hacking failure mode.}{fig:fmrg-flux-blueness}

\paragraph{Crossover at $\sigma = \pi^2$.} Recall from \Cref{thm:fmrg-gaussian,thm:plugin-flow-gaussian} that for an isotropic Gaussian target $\rho_1 = \Normal(\mu, \sigma I_d)$ under the memoryless schedule, the eigenvalues of the terminal covariances are $\tilde{\Sigma}^{\mathrm{FMRG}} = \sigma e^{-2\pi \lambda \sqrt{\sigma}}$ and $\Sigma^{(1)} = \sigma e^{-2 \lambda \sigma}$. The plug-in covariance shrinks faster than the FMRG covariance precisely when $2 \lambda \sigma > 2 \pi \lambda \sqrt{\sigma} \iff \sigma > \pi^2$. \Cref{fig:fmrg-gaussian-crossover} verifies this regime crossover empirically: with the narrow target ($\sigma = 0.5 < \pi^2$) at $\lambda = 1$, FMRG produces a much sharper terminal distribution than the $k = 1$ plug-in flow, while with the wide target ($\sigma = 16 > \pi^2$) at $\lambda = 0.1$, the same comparison reverses and FMRG samples are visibly more spread out than the plug-in samples. Both regimes exhibit aggressive reward hacking, but the relative aggressiveness flips with $\sigma$ exactly as predicted.

\image{0.95}{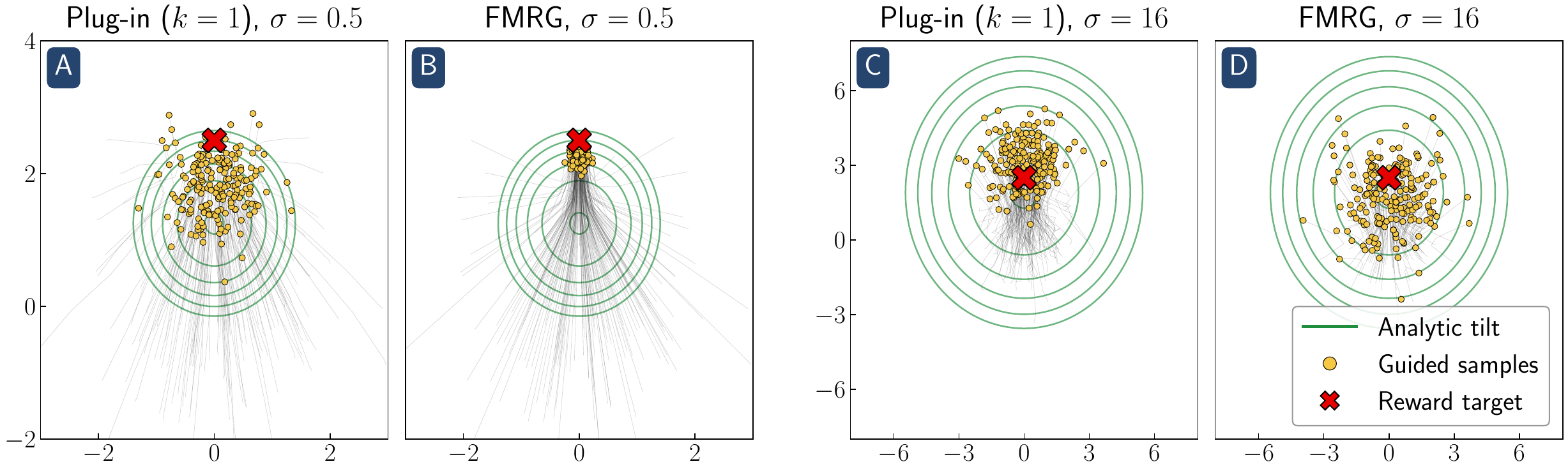}{\textbf{$\sigma$-regime crossover.} Empirical comparison of FMRG vs.\ the $k = 1$ plug-in flow on an isotropic Gaussian target. Panels A,B (narrow, $\sigma = 0.5 < \pi^2$, $\lambda = 1$): FMRG concentrates much more sharply at the reward target than the plug-in. Panels C,D (wide, $\sigma = 16 > \pi^2$, $\lambda = 0.1$): the relative aggressiveness reverses, with FMRG samples visibly more spread out, exactly as predicted by \Cref{thm:fmrg-gaussian}.}{fig:fmrg-gaussian-crossover}

\end{document}